\documentclass[11pt]{article}
\pdfoutput=1

\usepackage{mathrsfs}
\usepackage{amsmath,amssymb}
\usepackage{bm}
\usepackage{natbib}
\usepackage[usenames]{color}
\usepackage{amsthm}
\usepackage{wrapfig}
\usepackage{multirow} 
\usepackage{enumitem}
\usepackage{booktabs}
\usepackage{caption}

\usepackage[colorlinks,
linkcolor=red,
anchorcolor=blue,
citecolor=blue
]{hyperref}

\allowdisplaybreaks

\usepackage{mylatexstyle}

\usepackage{setspace}
\usepackage[left=1in, right=1in, top=1in, bottom=1in]{geometry}

\usepackage[textsize=tiny]{todonotes}
\title{\huge Towards Understanding Mixture of Experts in Deep Learning}

\author
{
Zixiang Chen\thanks{Department of Computer Science, University of California, Los Angeles, CA 90095, USA; e-mail: {\tt chenzx19@cs.ucla.edu}}
~and~
Yihe Deng\thanks{Department of Computer Science, University of California, Los Angeles, CA 90095, USA; e-mail: {\tt yihedeng@cs.ucla.edu}}
~and~
Yue Wu\thanks{Department of Computer Science, University of California, Los Angeles, CA 90095, USA; e-mail: {\tt :
ywu@cs.ucla.edu}}
~and~
Quanquan Gu\thanks{Department of Computer Science, University of California, Los Angeles, CA 90095, USA; e-mail: {\tt qgu@cs.ucla.edu}}
~and~
Yuanzhi Li\thanks{Machine Learning Department, Carnegie Mellon University, Pittsburgh, PA, USA; email: {\tt yuanzhil@andrew.cmu.edu}}
}

\date{}

\def\poly{\mathrm{poly}}

\def\cb{\mathrm{\mathbf{c}}}
\def\polylog{\mathrm{polylog}}

\def\cN{\mathcal{N}}

\newcommand{\la}{\langle}
\newcommand{\ra}{\rangle}

\begin{document}

\maketitle

\begin{abstract}%
The Mixture-of-Experts (MoE) layer, a sparsely-activated model controlled by a router, has achieved great success in deep learning. However, the understanding of such architecture remains elusive. In this paper, we formally study how the MoE layer improves the performance of neural network learning and why the mixture model will not collapse into a single model. Our empirical results suggest that the cluster structure of the underlying problem and the non-linearity of the expert are pivotal to the success of MoE. To further understand this, we consider a challenging classification problem with intrinsic cluster structures, which is hard to learn using a single expert. Yet with the MoE layer, by choosing the experts as two-layer nonlinear convolutional neural networks (CNNs), we show that the problem can be learned successfully. Furthermore, our theory shows that the router can learn the cluster-center features, which helps divide the input complex problem into simpler linear classification sub-problems that individual experts can conquer. To our knowledge, this is the first result towards formally understanding the mechanism of the MoE layer for deep learning.

\end{abstract}

\section{Introduction}
The Mixture-of-Expert (MoE) structure \citep{jacobs1991adaptive,jordan1994hierarchical} is a classic design that substantially scales up the model capacity and only introduces small computation overhead. In recent years, the MoE layer \citep{eigen2013learning,shazeer2017outrageously}, which is an extension of the MoE model to deep neural networks, has achieved remarkable success in deep learning. 
Generally speaking, an MoE layer contains many experts that share the same network architecture and are trained by the same algorithm, with a gating (or routing) function that routes individual inputs to a few experts among all the candidates. 
Through the sparse gating function, the router in the MoE layer can route each input to the top-$K (K \geq 2)$ best experts \citep{shazeer2017outrageously}, or the single ($K=1$) best expert \citep{fedus2021switch}. This routing scheme only costs the computation of $K$ experts for a new input, which enjoys fast inference time.

Despite the great empirical success of the MoE layer, the theoretical understanding of such architecture is still elusive. 
In practice, all experts have the same structure, initialized from the same weight distribution \citep{fedus2021switch} and are trained with the same optimization configuration. The router is also initialized to dispatch the data uniformly. It is unclear why the experts can diverge to different functions that are specialized to make predictions for different inputs, and why the router can automatically learn to dispatch data, especially when they are all trained using simple \emph{local search algorithms} such as gradient descent.
Therefore, we aim to answer the following questions:
\begin{center}
\emph{
Why do the experts in MoE diversify instead of collapsing into a single model? And how can the router learn to dispatch the data to the right expert?
}    
\end{center}




In this paper, in order to answer the above question, we consider the natural ``mixture of classification'' data distribution with cluster structure and theoretically study the behavior and benefit of the MoE layer. We focus on the simplest setting of the mixture of linear classification, where the data distribution has multiple clusters, and each cluster uses separate (linear) feature vectors to represent the labels. In detail, we consider the data generated as a combination of feature patches, cluster patches, and noise patches (See Definition~\ref{def:data_distribution} for more details). We study training an MoE layer based on the data generated from the ``mixture of classification'' distribution using gradient descent, where each expert is chosen to be a two-layer CNN.
The main contributions of this paper are summarized as follows:
\begin{itemize}[leftmargin=*,nosep]
  \item We first prove a negative result (Theorem~\ref{thm:neg}) that any single expert, such as two-layer CNNs with arbitrary activation function, cannot achieve a test accuracy of more than $87.5\%$ on our data distribution. 
    \item Empirically, we found that the mixture of linear experts performs better than the single expert but is still significantly worse than the mixture of non-linear experts. Figure~\ref{fig:Demo_2dl} provides such a result in a special case of our data distribution with four clusters. \emph{Although a mixture of linear models can represent the labeling function of this data distribution with $100\%$ accuracy, it fails to learn so after training}. We can see that the underlying cluster structure cannot be recovered by the mixture of linear experts, and neither the router nor the experts are diversified enough after training. In contrast, the mixture of non-linear experts can correctly recover the cluster structure and diversify. 
    \item Motivated by the negative result and the experiment on the toy data, we study a sparsely-gated MoE model with two-layer CNNs trained by gradient descent. We prove that this MoE model can achieve nearly $100 \%$ test accuracy \emph{efficiently} (Theorem~\ref{thm: MoE}).
   \item Along with the result on the test accuracy, we formally prove that each expert of the sparsely-gated MoE model will be specialized to a specific portion of the data (i.e., at least one cluster), which is determined by the initialization of the weights. In the meantime, the router can learn the cluster-center features and route the input data to the right experts.
    \item Finally, we also conduct extensive experiments on both synthetic and real datasets to corroborate our theory.
\end{itemize}

\noindent\textbf{Notation.}
We use lower case letters, lower case bold face letters, and upper case bold face
letters to denote scalars, vectors, and matrices respectively. We denote a union of disjoint sets $(A_{i}:i \in I)$ by $\sqcup_{i \in I}A_{i}$. For a vector $\xb$, we use $\|\xb\|_2$ to denote its Euclidean norm. For a matrix $\Wb$, we use $\|\Wb\|_F$ to denote its Frobenius norm. Given two sequences $\{x_n\}$ and $\{y_n\}$, we denote $x_n = \cO(y_n)$ if $|x_n|\le C_1 |y_n|$ for some absolute positive constant $C_1$, $x_n = \Omega(y_n)$ if $|x_n|\ge C_2 |y_n|$ for some absolute positive constant $C_2$, and $x_n = \Theta(y_n)$ if $C_3|y_n|\le|x_n|\le C_4 |y_n|$ for some absolute constants $C_3,C_4>0$. We also use $\tilde \cO(\cdot)$ to hide logarithmic factors of $d$  in $\cO(\cdot)$. 
Additionally, we denote $x_n=\poly(y_n)$ if $x_n=\cO( y_n^{D})$ for some positive constant $D$, and $x_n = \polylog(y_n)$ if $x_n= \poly( \log (y_n))$. We also denote by $x_{n}=o(y_{n})$ if $\lim_{n\rightarrow \infty}x_{n}/y_{n} = 0$. Finally we use $[N]$ to denote the index set $\{1, \dots, N\}$. 



\begin{figure}[!t]
    \centering
    \includegraphics[scale=0.5]{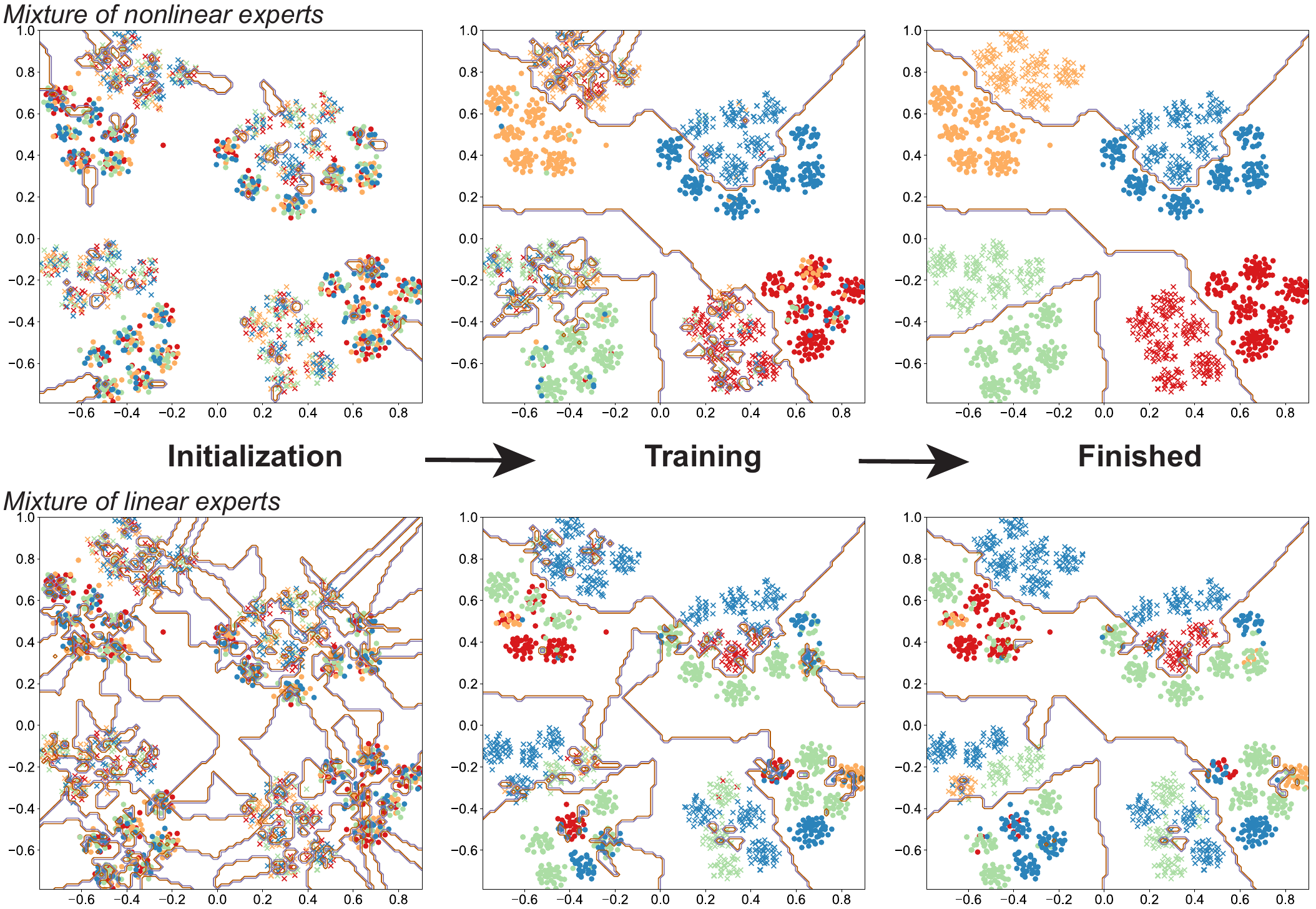}
    \caption{\textbf{Visualization of the training of MoE with nonlinear expert and linear expert}. Different colors denote router's dispatch to different experts. The lines denote the decision boundary of the MoE model. The data points are visualized on 2d space via t-SNE \citep{van2008visualizing}. The MoE architecture follows section \ref{section:problemsetting} where nonlinear experts use activation function $\sigma(z)=z^3$. For this visualization, we let the expert number $M=4$ and cluster number $K=4$. We generate $n=1,600$ data points from the distribution illustrated in Section \ref{section:problemsetting} with $\alpha \in (0.5,2)$, $\beta \in (1,2)$, $\gamma \in (1,2)$, and $\sigma_p = 1$. More details of the visualization are discussed in Appendix~\ref{appendix:experiment}.}
    \label{fig:Demo_2dl}
\end{figure}

\section{Related Work}
\noindent \textbf{Mixture of Experts Model.} The mixture of experts model \citep{jacobs1991adaptive,jordan1994hierarchical} has long been studied in the machine learning community. These MoE models are based on various base expert models such as support vector machine \citep{collobert2002parallel}
, Gaussian processes \citep{tresp2001mixtures}, or hidden Markov models \citep{jordan1997hidden}. In order to increase the model capacity to deal with the complex vision and speech data, \citet{eigen2013learning} extended the MoE structure to the deep neural networks, and proposed a deep MoE model composed of multiple layers of routers and experts.  \citet{shazeer2017outrageously} simplified the MoE layer by making the output of the gating function sparse for each example, which greatly improves the 
training stability and reduces the computational cost. 
Since then, the MoE layer with different base neural network structures \citep{shazeer2017outrageously, dauphin2017language, vaswani2017attention} has been proposed and achieved tremendous successes in a variety of language tasks. 
Very recently, \citet{fedus2021switch} improved the performance of the MoE layer by routing one example to only a single expert instead of $K$ experts, which further reduces the routing computation while preserving the model quality.

\noindent\textbf{Mixture of Linear Regressions/Classifications.} In this paper, we  consider a ``mixture of classification'' model. This type of models can be dated back to \citep{de1989mixtures, jordan1994hierarchical,  faria2010fitting} and has been applied to many tasks including object recognition \mbox{ \citep{quattoni2004conditional} }human action recognition \citep{wang2009max}, and machine translation \citep{liang2006end}.  In order to learn the unknown parameters for mixture of linear regressions/classification model, \citep{anandkumar2012method, hsu2012identifiability, chaganty2013spectral, anandkumar2014tensor, li2018learning} studies the method of moments and tensor factorization.
Another line of work studies specific algorithms such as Expectation-Maximization (EM) algorithm \citep{khalili2007variable, yi2014alternating, balakrishnan2017statistical,wang2015high}.

\noindent\textbf{Theoretical Understanding of Deep Learning.} In recent years, great efforts have been made to establish the theoretical foundation of deep learning. A series of studies have proved the convergence \citep{jacot2018neural, li2018learning, du2018gradient, allen2018convergence, zou2018stochastic} and generalization \citep{allen2018learning,arora2019fine,arora2019exact,cao2019generalizationsgd} guarantees in the so-called ``neural tangent kernel'' (NTK) regime, where the parameters stay close to the initialization, and 
the neural network function is approximately linear in its parameters. A recent line of works \citep{allen2019can,bai2019beyond, allen2020backward, allen2020feature, allen2020towards,li2020learning,cao2022benign, zou2021understanding, wen2021toward} studied the learning dynamic of neural networks
beyond the NTK regime. It is worthwhile to mention that our analysis of the MoE model is also beyond the NTK regime.

\section{Problem Setting and Preliminaries}\label{section:problemsetting}

We consider an MoE layer with each expert being a two-layer CNN trained by gradient descent (GD) over $n$ independent training examples $\{(\xb_{i}, y_{i})\}_{i =1}^{n}$ generated from a data distribution $\cD$. In this section, we will first introduce our data model $\cD$, and then explain our neural network model and the details of the training algorithm. 

\subsection{Data distribution}
We consider a binary classification problem over $P$-patch inputs, where each patch has $d$ dimensions. In particular, each labeled data is represented by $(\xb, y)$, where input $\xb = (\xb^{(1)}, \xb^{(2)}, \ldots, \xb^{(P)})\in (\RR^{d})^P$ is a collection of $P$ patches and $y \in \{\pm 1\}$ is the data label. We consider data generated from $K$ clusters. Each cluster $k \in [K]$ has a label signal vector $\vb_{k}$ and a cluster-center signal vector $\cb_{k}$ with $\|\vb_{k}\|_{2} = \|\cb_{k}\|_{2}=1$. For simplicity, we assume that all the signals $\{\vb_{k}\}_{k \in [K]}\cup \{\cb_{k}\}_{k \in [K]}$ are orthogonal with each other.  
\begin{definition}\label{def:data_distribution}
A data pair $(\xb,y)\in (\RR^{d})^{P} \times \{\pm 1\}$ is generated from the distribution $\cD$ as follows. 
\begin{itemize}[leftmargin=*,nosep]
    \item  Uniformly draw a pair $(k,k')$ with $k \not= k'$ from $\{1,\ldots,K\}$. 
    \item Generate the label $y\in\{\pm 1\}$ uniformly, generate a Rademacher random variable $\epsilon \in \{\pm 1\}$.
    \item Independently generate random variables $\alpha, \beta, \gamma$ from distribution $\cD_{\alpha}, \cD_{\beta}, \cD_{\gamma}$. In this paper, we assume there exists absolute constants $C_1, C_2$ such that almost surely $0 < C_1 \leq \alpha, \beta, \gamma \leq C_2$.
    \item Generate $\xb$ as a collection of $P$ patches: $\xb = (\xb^{(1)},\xb^{(2)}, \ldots, \xb^{(P)})\in (\RR^{d})^{P}$, where
    \begin{itemize}[leftmargin=*,nosep]
        \item \textbf{Feature signal.} One and only one patch is given by $y\alpha \vb_k$.
        \item \textbf{Cluster-center signal.}  One and only one patch is given by  $\beta\cb_k$.
         \item \textbf{Feature noise.} 
        One and only one patch is given by $\epsilon\gamma \vb_{k'}$. 
        \item  \textbf{Random noise.} The rest of the $P-3$ patches are Gaussian noises that are independently drawn from $N(0, (\sigma_{p}^{2}/d)\cdot\Ib_{d})$ where $\sigma_{p}$ is an absolute constant.
    \end{itemize}
\end{itemize}
\end{definition}

\noindent\textbf{How to learn this type of data?} Since the positions of signals and noises are not specified in Definition~\ref{def:data_distribution}, it is natural to use the CNNs structure that applies the same function to each patch. 
We point out that the strength of the feature noises $\gamma$ could be as large as the strength of the feature signals $\alpha$. As we will see later in Theorem~\ref{thm:neg}, this classification problem is hard to learn with a single expert, such as any two-layer CNNs (any activation function with any number of neurons). However, such a classification problem has an intrinsic clustering structure that may be utilized to achieve better performance. Examples can be divided into $K$ clusters $\cup_{k \in [K]}\Omega_{k}$ based on the cluster-center signals: an example $(\xb, y) \in \Omega_{k}$ if and only if at least one patch of $\xb$ aligns with $\cb_{k}$. It is not difficult to show that the binary classification sub-problem over $\Omega_{k}$ can be easily solved by an individual expert. We expect the MoE can learn this data cluster structure from the cluster-center signals.

\noindent\textbf{Significance of our result.} Although this data can be learned by existing works on a mixture of linear classifiers with sophisticated algorithms \citep{anandkumar2012method, hsu2012identifiability, chaganty2013spectral}, the focus of our paper is training a mixture of nonlinear neural networks, a more practical model used in real applications. When an MoE is trained by variants of gradient descent, we show that the experts \emph{automatically learn to specialize on each cluster}, while the router \emph{automatically learns to dispatch the data to the experts according to their specialty}. Although from a representation point of view, it is not hard to see that the concept class can be represented by MoEs, our result is very significant as we prove that gradient descent from random initialization can find a good MoE with non-linear experts efficiently. To make our results even more compelling, we empirically show that MoE with linear experts, despite also being able to represent the concept class, \emph{cannot} be trained to find a good classifier efficiently. 








\subsection{Structure of the MoE layer}\label{sec:MoE111}

An MoE layer consists of a set of $M$ ``expert networks'' $f_{1}, \ldots, f_{M}$, and a gating network which is generally set to be linear \citep{shazeer2017outrageously, fedus2021switch}. Denote by $f_{m}(\xb; \Wb)$ the output of the $m$-th expert network with input $x$ and parameter $\Wb$. Define an $M$-dimensional vector $\hb(\xb;\bTheta) = \sum_{p \in [P]}\bTheta^{\top}\xb^{(p)}$ as the output of the gating network parameterized by $\bTheta = [\btheta_1,\ldots,\btheta_M] \in \RR^{d\times M}$. The output $F$ of the MoE layer can be written as follows: 
\begin{align*}
F(\xb; \bTheta, \Wb) = \textstyle{\sum_{m\in \mathcal{T}_{\xb}}}\pi_m(\xb; \bTheta)f_{m}(\xb;\Wb),    
\end{align*}
where $\mathcal{T}_{\xb}\subseteq [M]$ is a set of selected indices and $\pi_{m}(\xb; \bTheta)$'s are route gate values given by   
\begin{align*}
     \pi_m(\xb; \bTheta) = \frac{\exp(h_{m}(\xb;\bTheta))}{\sum^M_{m'=1} \exp(h_{m'}(\xb;\bTheta))}, \forall m \in [M]. 
\end{align*}
\noindent\textbf{Expert Model.}
In practice, one often uses nonlinear neural networks as experts in the MoE layer. In fact, we found that the non-linearity of the expert is essential for the success of the MoE layer (see Section~\ref{sec:exp}). For $m$-th expert, we consider a convolution neural network as follows:
\begin{align}
    f_{m}(\xb;\Wb) = \textstyle{\sum_{j \in [J]}\sum^{P}_{p=1}} \sigma \big(\langle \wb_{m,j}, \xb^{(p)} \rangle \big), \label{eq:single expert}
\end{align}
where $\wb_{m,j} \in \RR^d$ is the weight vector of the $j$-th filter (i.e., neuron) in the $m$-th expert, $J$ is the number of filters (i.e., neurons). We denote $\Wb_{m} = [\wb_{m,1}, \ldots, \wb_{m,J}] \in \RR^{d\times J}$ as the weight matrix of the $m$-th expert and further let $\Wb = \{\Wb_{m}\}_{m \in [M]}$ as the collection of expert weight matrices. For nonlinear CNN, we consider the cubic activation function $\sigma(z) = z^{3}$, which is one of the simplest nonlinear activation functions \citep{vecci1998learning}. We also include the experiment for other activation functions such as RELU in  Appendix Table~\ref{tab:verification}.


\noindent\textbf{Top-1 Routing Model.} 
A simple choice of the selection set $\mathcal{T}_{\xb}$ would be the whole experts set $\mathcal{T}_{\xb} = [M]$ \citep{jordan1994hierarchical}, which is the case for the so-called soft-routing model. However, it would be time consuming to use soft-routing in deep learning. In this paper, we consider ``switch routing'', which is introduced by \citet{fedus2021switch} to make the gating network sparse and save the computation time.  For each input $\xb$, instead of using all the experts, we only pick one expert from $[M]$, i.e., $|\mathcal{T}_{\xb}| = 1$. In particular, we choose $\mathcal{T}_{\xb} = \argmax_{m} \{h_{m}(\xb;\bTheta)\}$.

\begin{minipage}{0.95\textwidth}
    \begin{minipage}[b]{0.58\textwidth}
    \centering
    \includegraphics[scale=0.22]{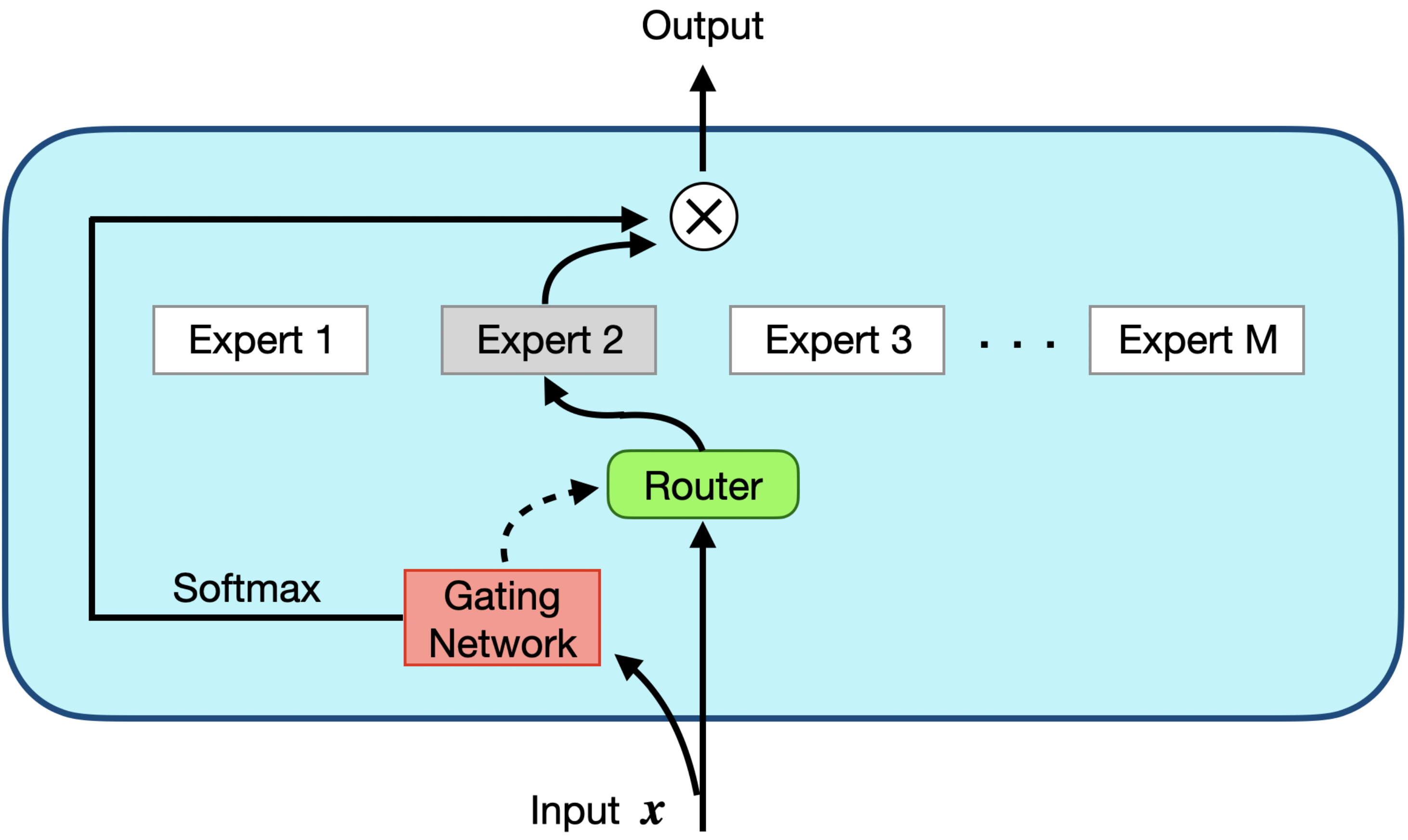}
    \captionof{figure}{\textbf{Illustration of an MoE layer.} For each input $\xb$, the router will only select one expert to perform computations. The choice is based on the output of the gating network (dotted line). The expert layer returns the output of the selected expert (gray box) multiplied by the route gate value (softmax of the gating function output).}
    \end{minipage}
    \hfill    
    \begin{minipage}[b]{0.4\textwidth}
    \centering
    \begin{algorithm}[H]\small 
    \caption{Gradient descent with random initialization}
    \begin{algorithmic}[1]\label{alg:GDrandominit}
    \REQUIRE Number of iterations $T$, expert learning rate $\eta$, router learning rate $\eta_{r}$, initialization scale $\sigma_{0}$, training set $S = \{(\xb_{i},y_{i})\}_{i=1}^{n}$.
    \STATE Generate each entry of $\Wb^{(0)}$ independently from $N(0, \sigma_{0}^{2})$.
    \STATE Initialize each entry of $\bTheta^{(0)}$ as zero.  
    \FOR{$t=0,2,\ldots, T-1$}
    \STATE Generate each entry of $\rb^{(t)}$ independently from Unif[0,1]. 
    \STATE Update $\Wb^{(t+1)}$ as in \eqref{eq:w-update}.
    \STATE Update $\bTheta^{(t+1)}$ as in \eqref{eq:theta-update}.
    \ENDFOR
    \RETURN $(\bTheta^{(T)}, \Wb^{(T)})$.
    \end{algorithmic}
    \end{algorithm}
    \end{minipage}
\end{minipage}

\subsection{Training Algorithm}

Given the training data $S = \{(\xb_i,y_i)\}_{i=1}^{n}$, we train $F$ with gradient descent to minimize the following empirical loss function: 
\begin{align}
   \cL(\bTheta, \Wb) = \frac{1}{n} \textstyle{\sum_{i=1}^n} \ell\big(y_{i}F(\xb_{i}; \bTheta, \Wb)\big),\label{eq:empirical loss}
\end{align}
where $\ell$ is the logistic loss defined as $\ell(z)= \log(1+\exp(-z))$. 
 We initialize $\bTheta^{(0)}$ to be zero and initialize each entry of $\Wb^{(0)}$ by i.i.d $\cN(0, \sigma_{0}^{2})$. Zero initialization of the gating network is widely used in MoE training. As discussed in \citet{shazeer2017outrageously}, it can help avoid out-of-memory errors and initialize the network in a state of approximately equal expert load (see \eqref{eq:Load} for the definition of expert load). 
 
Instead of directly using the gradient of empirical loss \eqref{eq:empirical loss} to update weights, we add perturbation to the router and use the gradient of the perturbed empirical loss to update the weights. In particular, the training example $\xb_{i}$  will be distributed to $\argmax_{m} \{h_{m}(\xb_{i};\bTheta^{(t)}) + r_{m,i}^{(t)}\}$ instead, where $\{r_{m,i}^{(t)}\}_{m \in [M],i\in [n]}$ are random noises. Adding noise term is a widely used training strategy for sparsely-gated MoE layer \citep{shazeer2017outrageously, fedus2021switch}, which can encourage exploration across the experts and stabilize the MoE training. In this paper, we draw $\{r_{m,i}^{(t)}\}_{m \in [M],i\in[n]}$ independently from the uniform distribution $\text{Unif}[0,1]$ and denotes its collection as $\rb^{(t)}$. Therefore, the perturbed empirical loss at iteration $t$ can be written as 
\begin{align}
\cL^{(t)}(\bTheta^{(t)},\Wb^{(t)})
    & = \frac{1}{n} \textstyle{\sum_{i=1}^{n}}\ell\big(y_{i}\pi_{m_{i,t}}(\xb_{i};\bTheta^{(t)})f_{m_{i,t}}(\xb_{i}; \Wb^{(t)})\big),\label{eq:perturbed empirical loss}
\end{align} 
where $m_{i,t} = \argmax_{m} \{h_{m}(\xb_{i};\bTheta^{(t)}) + r_{m,i}^{(t)}\}$. Starting from the initialization $\Wb^{(0)}$, the gradient descent update rule for the experts is 
 \begin{align}
\Wb^{(t+1)}_{m} &= \Wb^{(t)}_{m} - \eta\cdot \nabla_{\Wb_{m}}\cL^{(t)}(\bTheta^{(t)}, \Wb^{(t)})/\|\nabla_{\Wb_{m}}\cL^{(t)}(\bTheta^{(t)}, \Wb^{(t)})\|_{F}\label{eq:w-update}, \forall m \in [M],
 \end{align}
where $\eta>0$ is the expert learning rate. Starting from the initialization $\bTheta^{(0)}$, the gradient update rule for the gating network is 
 \begin{align}
 \btheta^{(t+1)}_{m} &= \btheta_{m}^{(t)} - \eta_{r}\cdot \nabla_{\btheta_{m}}\cL^{(t)}(\bTheta^{(t)}, \Wb^{(t)}), \forall m \in [M], \label{eq:theta-update}
 \end{align}
 where $\eta_{r}>0$ is the router learning rate. In practice, the experts are trained by Adam \citep{kingma2014adam} to make sure they have similar learning speeds. Here we use a normalized gradient which can be viewed as a simpler alternative to Adam \citep{jelassi2021adam}.

\section{Main Results}\label{sec: Thmmain}
In this section, we will present our main results. We first provide a negative result for learning with a single expert.   

\begin{theorem}[Single expert performs poorly]\label{thm:neg}
Suppose $\cD_{\alpha} = \cD_{\gamma}$ in Definition~\ref{def:data_distribution}, then any function with the form $F(\xb) = \sum_{p = 1}^{P}f(\xb^{(p)})$ will get large test error $\mathbb{P}_{(\xb,y)\sim \cD}\big(yF(\xb)\leq 0\big) \geq 1/8$.
\end{theorem}

Theorem~\ref{thm:neg} indicates that if the feature noise has the same strength as the feature signal i.e., $\cD_{\alpha} = \cD_{\gamma}$, any two-layer CNNs with the form $F(\xb) = \sum_{j \in [J]}a_{j}\sum_{p \in [P]}\sigma(\wb_{j}^{\top}\xb^{(p)} + b_{j})$ can't perform well on the classification problem defined in Definition~\ref{def:data_distribution} where $\sigma$ can be any activation function. Theorem 4.1 also shows that a simple ensemble of the experts may not improve the performance because the ensemble of the two-layer CNNs is still in the form of the function defined in Theorem~\ref{thm:neg}. 

As a comparison, the following theorem gives the learning guarantees for training an MoE layer that follows the structure defined in Section~\ref{sec:MoE111} with cubic activation function. 


\begin{theorem}[Nonlinear MoE performs well]\label{thm: MoE}
Suppose the training data size $n = \Omega(d)$. Choose experts number $M = \Theta(K\log K \log\log d)$, filter size $J = \Theta(\log M \log\log d )$, initialization scale $\sigma_{0} \in [d^{-1/3}, d^{-0.01}]$, learning rate $\eta = \tilde{O}(\sigma_{0}), \eta_{r} = \Theta(M^{2})\eta$. Then with probability at least $1 - 
o(1)$, Algorithm~\ref{alg:GDrandominit} is able to output $(\bTheta^{(T)}, \Wb^{(T)})$ within $T = \tilde{O}(\eta^{-1})$ iterations such that the non-linear MoE defined in Section~\ref{sec:MoE111} satisfies
\begin{itemize}[leftmargin=*,nosep]
    \item Training error is zero, i.e.,  $y_{i}F(\xb_{i};\bTheta^{(T)}, \Wb^{(T)}) > 0, \forall i \in [n]$.
    \item Test error is nearly zero, i.e., $ \mathbb{P}_{(\xb,y)\sim \cD}\big(y F(\xb;\bTheta^{(T)}, \Wb^{(T)}) \leq 0\big) = o(1)$.
\end{itemize}

More importantly, the experts can be divided into a disjoint union of $K$ non-empty sets $[M] = \sqcup_{k \in [K]}\cM_{k}$ and 
\begin{itemize}[leftmargin=*,nosep]
\item (Each expert is good on one cluster) Each expert $m \in \cM_{k}$ performs good on the cluster $\Omega_{k}$, $\mathbb{P}_{(\xb,y)\sim \cD}(yf_{m}(\xb;\Wb^{(T)}) \leq 0| (\xb, y)\in \Omega_{k}) = o(1) $.
\item (Router only distributes example to good expert) With probability at least $1 - o(1)$, an example $\xb \in \Omega_{k}$ will be routed to one of the experts in $\cM_{k}$.
\end{itemize}
\end{theorem}
Theorem~\ref{thm: MoE} shows that a non-linear MoE performs well on the classification problem in Definition~\ref{def:data_distribution}. In addition, the router will learn the cluster structure and divide the problem into $K$ simpler sub-problems, each of which is associated with one cluster. In particular, each cluster will be classified accurately by a subset of experts. On the other hand, each expert will perform well on at least one cluster. 

Furthermore, together with Theorem \ref{thm:neg}, Theorem~\ref{thm: MoE} suggests that there exist problem instances in Definition~\ref{def:data_distribution} (i.e., $\cD_{\alpha} = \cD_{\gamma}$) such that an MoE provably outperforms a single expert.


\section{Overview of Key Techniques}\label{sec: sketch} 
A successful MoE layer needs to ensure that the router can learn the cluster-center features and divide the
complex problem in Definition~\ref{def:data_distribution} into simpler linear classification sub-problems that individual experts can conquer. Finding such a gating network is difficult because this problem is highly non-convex. In the following, we will introduce the main difficulties in analyzing the MoE layer and the corresponding key techniques to overcome those barriers.


\noindent\textbf{Main Difficulty 1: Discontinuities in Routing.} Compared with the traditional soft-routing model, the sparse routing model saves computation and greatly reduces the inference time. However, this form of sparsity also causes discontinuities in routing \citep{shazeer2017outrageously}. In fact, even a small perturbation of the gating network outputs $\hb(\xb;\bTheta) + \boldsymbol{\delta}$ may change the router behavior drastically if the second largest gating network output is close to the largest gating network output.

\noindent\textbf{Key Technique 1: Stability by Smoothing.} We point out that the noise term added to the gating network output ensures a smooth transition between different routing behavior, which makes the router more stable. This is proved in the following lemma.

\begin{lemma}\label{lm:Msmoothly}
Let $\hb, \hat{\hb} \in \RR^{M}$ to be the output of the gating network and $\{r_{m}\}_{m=1}^{M}$ to be the noise independently drawn from Unif[0,1]. Denote $\pb, \hat{\pb} \in \RR^{M}$ to be the probability that experts get routed, i.e., $p_{m} = \mathbb{P}(\argmax_{m'\in[M]}\{h_{m'} + r_{m'}\} = m )$, $\hat{p}_{m} = \mathbb{P}(\argmax_{m'\in [M]}\{\hat{h}_{m'} + r_{m'}\} =m )$. Then we have that $\|\pb - \hat{\pb}\|_{\infty} \leq M^{2}\|\hb - \hat{\hb}\|_{\infty}$.
\end{lemma}

Lemma~\ref{lm:Msmoothly} implies that when the change of the gating network outputs at iteration $t$ and $t'$ is small, i.e., $\|\hb(\xb;\bTheta^{(t)})- \hb(\xb; \bTheta^{(t')})\|_{\infty}$, the router behavior will be similar. So adding noise provides a smooth transition from time $t$ to $t'$. It is also worth noting that $\bTheta$ is zero initialized. So $\hb(\xb; \bTheta^{(0)}) = 0$ and thus each expert gets routed with the same probability $p_{m} = 1/M$ by symmetric property. 
Therefore, at the early of the training when $\|\hb(\xb;\bTheta^{(t)})- \hb(\xb; \bTheta^{(0)})\|_{\infty}$ is small, router will almost uniformly pick one expert from $[M]$, which helps exploration across experts.

\noindent\textbf{Main Difficulty 2: No ``Real'' Expert.}
At the beginning of the training, the gating network is zero, and the experts are randomly initialized. Thus it is hard for the router to learn the right features because all the experts look the same: they share the same network architecture and are trained by the same algorithm. The only difference would be the initialization. Moreover, if the router makes a mistake at the beginning of the training, the experts may amplify the mistake because the experts will be trained based on mistakenly dispatched data.

\noindent\textbf{Key Technique 2: Experts from Exploration.} Motivated by the key technique 1, we introduce an exploration stage to the analysis of MoE layer during which the router almost uniformly picks one expert from $[M]$. This stage starts at $t=0$ and ends at $T_{1} = \lfloor\eta^{-1}\sigma_{0}^{0.5}\rfloor \ll T = \tilde{O}(\eta^{-1})$ and the gating network remains nearly unchanged $\|\hb(\xb;\bTheta^{(t)})- \hb(\xb; \bTheta^{(0)})\|_{\infty} = O(\sigma_{0}^{1.5})$. 
Because the experts are treated almost equally during exploration stage, we can show that the experts become specialized to some specific task only based on the initialization. In particular, the experts set $[M]$ can be divided into $K$ nonempty disjoint sets $[M] = \sqcup_{k} \cM_{k}$, where $\cM_{k} := \{m|\argmax_{k'\in [K], j\in [J]}\la \vb_{k'}, \wb_{m,j}^{(0)}\ra = k\}$. For nonlinear MoE with cubic activation function, the following lemma further shows that experts in different set $\cM_{k}$ will diverge at the end of the exploration stage. 



\begin{lemma}\label{lm:stage1}
Under the same condition as in Theorem~\ref{thm: MoE}, with probability at least $1 - o(1)$, the following equations hold for all expert $m \in \cM_{k}$, 
\begin{align*}
\mathbb{P}_{(\xb, y)\sim \cD}\big(yf_{m}(\xb; \Wb^{(T_{1})}\big) \leq 0 \big|(\xb, y) \in \Omega_{k}\big) &= o(1),\\
\mathbb{P}_{(\xb, y)\sim \cD}\big(yf_{m}(\xb; \Wb^{(T_{1})}) \leq 0\big|(\xb, y) \in \Omega_{k'}\big) &= \Omega\big(1/K\big), \forall k' \not = k.
\end{align*}
\end{lemma}

Lemma~\ref{lm:stage1} implies that, at the end of the exploration stage, the expert $m \in \cM_{k}$ can achieve nearly zero test error on the cluster $\Omega_{k}$ but high test error on the other clusters $\Omega_{k'}, k' \not= k$.

\noindent\textbf{Main Difficulty 3: Expert Load Imbalance.}  Given the training data set $S = \{(\xb_{i},y_{i})\}_{i=1}^{n}$, the load of expert $m$ at iterate $t$ is defined as 
\begin{align}
\text{Load}_{m}^{(t)} = \textstyle{\sum_{i \in [n]}} \mathbb{P}(m_{i,t} = m), \label{eq:Load}
\end{align}
where $\mathbb{P}(m_{i,t} = m)$ is probability that the input $\xb_{i}$ being routed to expert $m$ at iteration $t$.
\citet{eigen2013learning} first described the load imbalance issues in the training of the MoE layer. The gating network may converge to a state where it always produces
large $\text{Load}_{m}^{(t)}$ for the same few experts. This imbalance in expert load is self-reinforcing, as the favored experts
are trained more rapidly and thus are selected even more frequently by the router \citep{shazeer2017outrageously, fedus2021switch}. Expert load imbalance issue not only causes memory and performance
problems in practice, but also impedes the theoretical analysis of the expert training.


\noindent\textbf{Key Technique 3: Normalized Gradient Descent.} 
Lemma~\ref{lm:stage1} shows that the experts will diverge into $\sqcup_{k \in [K]}\cM_{k}$. Normalized gradient descent can help different experts in the same $\cM_{k}$ being trained at the same speed regardless the imbalance load caused by the router. Because the self-reinforcing circle no longer exists, we can prove that the router will treat different experts in the same $\cM_{k}$ almost equally and dispatch almost the same amount of data to them (See Section~\ref{subsection: routerlearning} in Appendix for detail). This Load imbalance issue can be further avoided by adding load balancing loss \citep{eigen2013learning, shazeer2017outrageously, fedus2021switch}, or advanced MoE layer structure such as BASE Layers \citep{lewis2021base, dua2021tricks} and Hash Layers \citep{roller2021hash}.

\noindent\textbf{Road Map:} Here we provide the road map of the proof of Theorem~\ref{thm: MoE} and the full proof is presented in Appendix \ref{appendix:main theory}.  The training process can be decomposed into several stages.  
The first stage is called \emph{Exploration stage}. 
During this stage, the experts will diverge into $K$ professional groups $ \sqcup_{k=1}^{K}\cM_{k} = [M]$. In particular, we will show that $\cM_{k}$ is not empty for all $k \in [K]$. Besides, for all $m \in \cM_{k}$, $f_{m}$ is a good classifier over $\Omega_{k}$. The second stage is called \emph{router learning stage}. During this stage, the router will learn to dispatch $\xb \in \Omega_{k}$ to one of the experts in $\cM_{k}$. Finally, we will give the generalization analysis for the MoEs from the previous two stages.

\section{Experiments}\label{sec:exp}

\begin{minipage}{\textwidth}
\begin{minipage}[b]{0.63\textwidth}
    \centering
    \begin{tabular}{c c c}
        \multicolumn{3}{c}{Setting 1:$\alpha\in (0.5,2)$, $\beta\in (1,2)$, $\gamma\in (0.5,3), \sigma_{p}=1$} \\
    \addlinespace[3pt]
    \toprule
         &  Test accuracy ($\%$) & Dispatch Entropy \\
    \midrule
         Single (linear) & $68.71$ & NA \\
         Single (nonlinear) & $79.48$ & NA\\
         MoE (linear) &  $92.99 \pm 2.11$ &  $1.300 \pm 0.044$\\
         MoE (nonlinear) & $\mathbf{99.46 \pm 0.55}$ & $\mathbf{0.098 \pm 0.087}$ \\
    \bottomrule
    \\
        \multicolumn{3}{c}{Setting 2: $\alpha\in (0.5,2)$, $\beta\in (1,2)$, $\gamma\in (0.5,3)$, $\sigma_{p}  = 2$} \\
    \addlinespace[3pt]
    \toprule
        &  Test accuracy ($\%$) & Dispatch Entropy \\
    \midrule
         Single (linear) & $60.59$ &  NA \\
         Single (nonlinear) & $72.29$ & NA \\
         MoE (linear) &  $88.48 \pm 1.96$ &  $1.294 \pm 0.036$\\
         MoE (nonlinear) & $\mathbf{98.09 \pm 1.27}$ & $\mathbf{0.171 \pm 0.103}$ \\ 
    \bottomrule
    \end{tabular}
    \captionof{table}{\textbf{Comparison between MoE (linear) and MoE (nonlinear)} in our setting. We report results of top-1 gating with noise for both linear and nonlinear models. Over ten random experiments, we report the average value $\pm$ standard deviation for both test accuracy and dispatch entropy.}
    \label{tab:synthetic_exp_results_1}
    \end{minipage}
      \hfill
  \begin{minipage}[b]{0.35\textwidth}
    \includegraphics[scale=0.32]{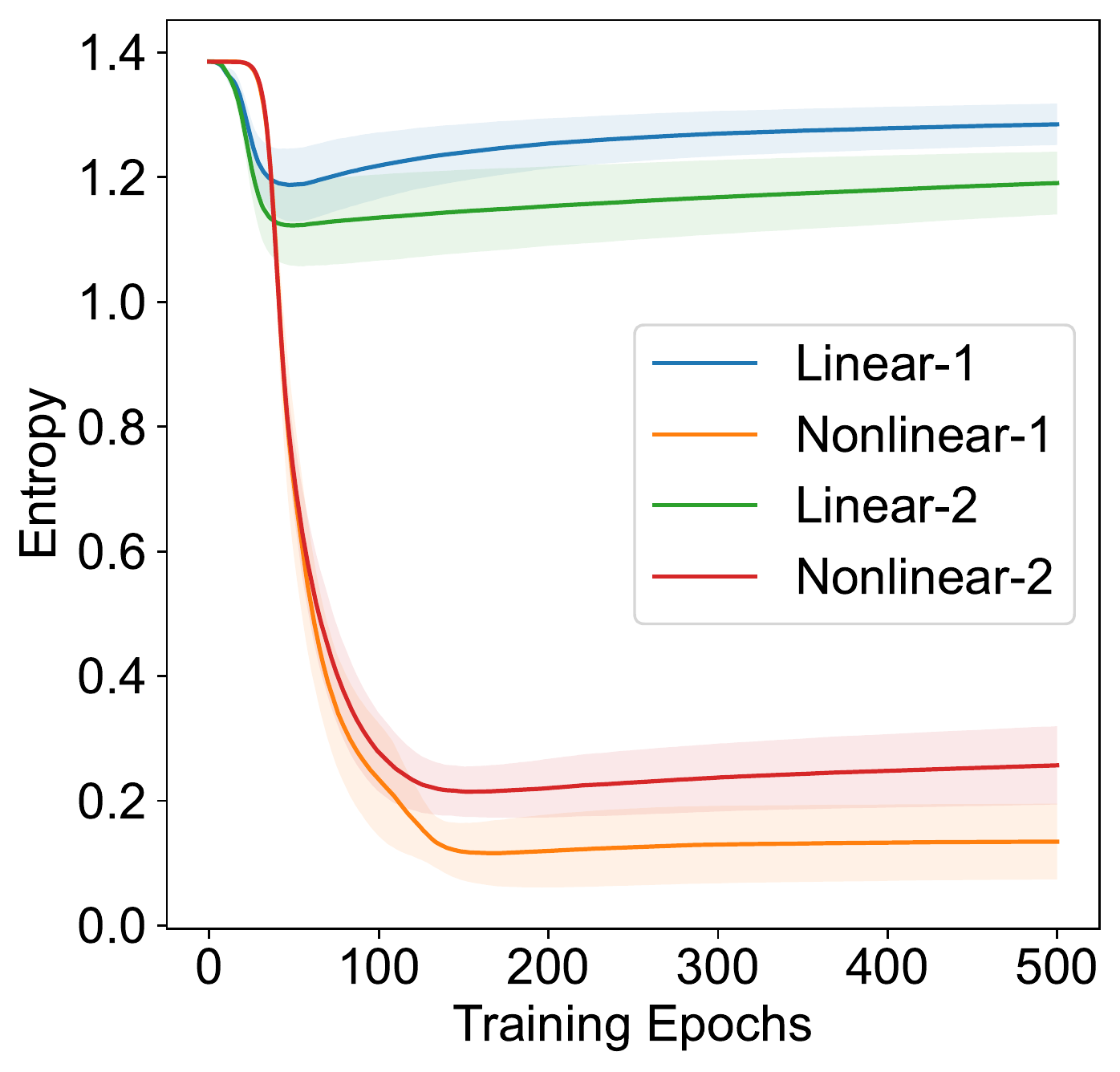}
    \captionof{figure}{\textbf{Illustration of router dispatch entropy.} We demonstrate the change of entropy of MoE during training on the synthetic data. MoE (linear)-1 and MoE (nonlinear)-1 refer to Setting 1 in Table~\ref{tab:synthetic_exp_results_1}. MoE (linear)-2 and MoE (nonlinear)-2 refer to Setting 2 in Table~\ref{tab:synthetic_exp_results_1}.}
    \label{fig:synthetic_exp_results}
  \end{minipage}
\end{minipage}


\subsection{Synthetic-data Experiments}
\noindent\textbf{Datasets.} We generate $16,000$ training examples and $16,000$ test examples from the data distribution defined in Definition~\ref{def:data_distribution} with cluster number $K=4$ , patch number $P = 4$ and dimension $d = 50$.  We randomly shuffle the order of the patches of $\xb$ after we generate data $(\xb, y)$. We consider two parameter settings: 1. $\alpha \sim \text{Uniform}(0.5,2)$, $\beta \sim \text{Uniform}(1,2)$, $\gamma \sim \text{Uniform}(0.5,3)$ and $\sigma_{p} = 1$; 2. $\alpha \sim \text{Uniform}(0.5,2)$, $\beta \sim \text{Uniform}(1,2)$, $\gamma \sim \text{Uniform}(0.5,3)$ and $\sigma_{p}  = 2$. Note that Theorem~\ref{thm:neg} shows that when $\alpha$ and $\gamma$ follow the same distribution, neither single linear expert or single nonlinear expert can give good performance. Here we consider a more general and difficult setting when $\alpha$ and $\gamma$ are from different distributions. 

\noindent\textbf{Models.} We consider the performances of single linear CNN, single nonlinear CNN, linear MoE, and nonlinear MoE. The single nonlinear CNN architecture follows \eqref{eq:single expert} with cubic activation function, while single linear CNN follows \eqref{eq:single expert} with identity activation function. For both linear and nonlinear MoEs, we consider a mixture of $8$ experts with each expert being a single linear CNN or a single nonlinear CNN. Finally, we train single models with gradient descent and train the MoEs with Algorithm~\ref{alg:GDrandominit}. We run $10$ random experiments and report the average accuracy with standard deviation.

\noindent\textbf{Evaluation.} To evaluate how well the router learned the underlying cluster structure of the data, we define the entropy of the router's dispatch as follows. Denote by $n_{k,m}$ the number of data in cluster $K$ that are dispatched to expert $m$. The total number of data dispatched to expert $m$ is $n_m = \sum_{k=1}^K n_{k,m}$ and the total number of data is $n = \sum_{k=1}^K \sum_{m=1}^M {n_{k,m}}$. The dispatch entropy is then defined as
\begin{align}
    \text{entropy} = - \textstyle{\sum_{m=1,n_m\ne 0}^M} \frac{n_m}{n} \sum_{k=1}^K\frac{n_{k,m}}{n_m} \cdot \log\big(\frac{n_{k,m}}{n_m}\big) .
\end{align}
When each expert receives the data from at most one cluster, the dispatch entropy will be zero. And a uniform dispatch will result in the maximum dispatch entropy.

As shown in Table~\ref{tab:synthetic_exp_results_1}, the linear MoE does not perform as well as the nonlinear MoE in Setting $1$, with around $6\%$ less test accuracy and much higher variance. With stronger random noise (Setting $2$), the difference between the nonlinear MoE and linear MoE becomes even more significant. We also observe that the final dispatch entropy of nonlinear MoE is nearly zero while that of the linear MoE is large. In Figure~\ref{fig:synthetic_exp_results}, we further demonstrate the change of dispatch entropy during the training process. The dispatch entropy of nonlinear MoE significantly decreases, while that of linear MoE remains large. Such a phenomenon indicates that the nonlinear MoE can successfully learn the underlying cluster structure of the data while the linear MoE fails to do so.

\subsection{Real-data Experiments}

\begin{table}
    \centering
    \begin{tabular}{c c c c}
    \toprule
        &  & CIFAR-10 ($\%$) & CIFAR-10-Rotate ($\%$) \\
    \midrule
         \multirow{2}{5em}{CNN} & Single & $80.68\pm 0.45$ & $76.78 \pm 1.79$  \\
         & MoE & $80.31 \pm0.62$ & $\mathbf{79.60\pm 1.25}$  \\
    \midrule
        \multirow{2}{5em}{MobileNetV2} & Single & $92.45 \pm 0.25$ & $85.76 \pm 2.91$  \\
         & MoE & $92.23\pm0.72$ & $\mathbf{89.85 \pm 2.54}$ \\
    \midrule
        \multirow{2}{5em}{ResNet18} & Single & $95.51\pm 0.31$ & $88.23 \pm 0.96$  \\
         & MoE & $95.32 \pm 0.68$ & $\mathbf{92.60 \pm 2.01}$ \\
    \bottomrule
    \addlinespace[3pt]
    \end{tabular}
    \caption{Comparison between MoE and single model on CIFAR-10 and CIFAR-10-Rotate datasets. We report the average test accuracy over $10$ random experiments $\pm$ the standard deviation.}
    \label{tab:cifar10rotate}
\end{table}

We further conduct experiments on real image datasets and demonstrate the importance of the clustering data structure to the MoE layer in deep neural networks. 

\noindent\textbf{Datasets.} We consider the \textbf{CIFAR-10} dataset \citep{Krizhevsky09learningmultiple} and the 10-class classification task. Furthermore, we create a \textbf{CIFAR-10-Rotate} dataset that has a strong underlying cluster structure that is independent of its labeling function. Specifically, we rotate the images by $30$ degrees and merge the rotated dataset with the original one. The task is to predict if the image is rotated, which is a binary classification problem. We deem that some of the classes in CIFAR-10 form underlying clusters in CIFAR-10-Rotate. 
In Appendix~\ref{appendix:experiment}, we explain in detail how we generate CIFAR-10-Rotate and present some specific examples. 

\noindent\textbf{Models.} For the MoE, we consider a mixture of $4$ experts with a linear gating network. For the expert/single model architectures, we consider a CNN with $2$ convolutional layers (architecture details are illustrated in Appendix \ref{appendix:experiment}.)
For a more thorough evaluation, we also consider expert/single models with architecture including \textbf{MobileNetV2} \citep{sandler2018mobilenetv2} and \textbf{ResNet18} \citep{he2016deep}. The training process of MoE also follows Algorithm \ref{alg:GDrandominit}.  

The experiment results are shown in Table~\ref{tab:cifar10rotate}, where we compare single and mixture models of different architectures over CIFAR-10 and CIFAR-10-Rotate datasets. We observe that the improvement of MoEs over single models differs largely on the different datasets. On CIFAR-10, the performance of MoEs is very close to the single models. However, on the CIFAR-10-Rotate dataset, we can observe a significant performance improvement from single models to MoEs. Such results indicate the advantage of MoE over single models depends on the task and the cluster structure of the data.

\section{Conclusion and Future Work}

In this work, we formally study the mechanism of the Mixture of Experts (MoE) layer for deep learning. To our knowledge, we provide the first theoretical result toward understanding how the MoE layer works in deep learning. Our empirical evidence reveals that the cluster structure of the data plays an important role in the success of the MoE layer. Motivated by these empirical observations, we study a data distribution with cluster structure and show that Mixture-of-Experts provably improves the test accuracy of a single expert of two-layer CNNs.


There are several important future directions. First, our current results are for CNNs. It is interesting to extend our results to other neural network architectures, such as transformers. Second, our data distribution is motivated by the classification problem of image data. We plan to extend our analysis to other types of data (e.g., natural language data). 

\appendix



\section{Experiment Details}
\label{appendix:experiment}

\subsection{Visualization}
In the visualization of Figure~\ref{fig:Demo_2dl}, MoE (linear) and MoE (nonlinear) are trained according to Algorithm~\ref{alg:GDrandominit} by normalized gradient descent with learning rate $0.001$ and gradient descent with learning rate $0.1$. According to Definition~\ref{def:data_distribution}, we set $K=4$, $P=4$ and $d=50$ and choose $\alpha\in(0.5,2)$, $\beta\in(1,2)$, $\gamma\in(1,2)$ and $\sigma_p = 1$, and generate $3,200$ data examples. We consider mixture of $M=4$ experts for both MoE (linear) and MoE (nonlinear). For each expert, we set the number of neurons/filters $J=16$. We train MoEs on $1,600$ data examples and visualize classification result and decision boundary on the remaining $1,600$ examples. The data examples are visualized via t-SNE \citep{van2008visualizing}. When visualizing the data points and decision boundary on the 2d space, we increase the magnitude of random noise patch by $3$ so that the positive/negative examples and decision boundaries can be better viewed. 


\subsection{Synthetic-data Experiments}

\begin{table}[]
    \centering
    \begin{tabular}{c c c c}
    \multicolumn{4}{c}{Setting 1:$\alpha\in (0.5,2)$, $\beta\in (1,2)$, $\gamma\in (0.5,3), \sigma_{p}=1$} \\
    \addlinespace[3pt]
    \toprule
         &  Test accuracy ($\%$) & Dispatch Entropy & Number of Filters \\
    \midrule
         Single (linear) & $68.71$ & NA & 128 \\
         Single (linear) & $67.63$ & NA & 512 \\
         Single (nonlinear) & $79.48$ & NA & 128 \\
         Single (nonlinear) & $78.18$ & NA & 512 \\
         MoE (linear) &  $92.99 \pm 2.11$ &  $1.300 \pm 0.044$ & 128 (16*8)\\
         MoE (nonlinear) & $\mathbf{99.46 \pm 0.55}$ & $\mathbf{0.098 \pm 0.087}$ & 128 (16*8)\\
    \bottomrule
    \\
        \multicolumn{4}{c}{Setting 2: $\alpha\in (0.5,2)$, $\beta\in (1,2)$, $\gamma\in (0.5,3)$, $\sigma_{p}  = 2$} \\
        \addlinespace[3pt]
    \toprule
        &  Test accuracy ($\%$) & Dispatch Entropy & Number of Filters \\
    \midrule
         Single (linear) & $60.59$ &  NA & 128 \\
         Single (linear) & $63.04$ &  NA & 512 \\
         Single (nonlinear) & $72.29$ & NA & 128 \\
         Single (nonlinear) & $52.09$ & NA & 512 \\
         MoE (linear) &  $88.48 \pm 1.96$ &  $1.294 \pm 0.036$ & 128 (16*8)\\
         MoE (nonlinear) & $\mathbf{98.09 \pm 1.27}$ & $\mathbf{0.171 \pm 0.103}$ & 128 (16*8) \\ 
    \bottomrule
    \\
        \multicolumn{4}{c}{Setting 3:$\alpha\in (0.5,2)$, $\beta\in (1,2)$, $\gamma\in (0.5,2), \sigma_{p}=1$} \\
        \addlinespace[3pt]
    \toprule
         &  Test accuracy ($\%$) & Dispatch Entropy & Number of Filters\\
    \midrule
         Single (linear) & $74.81$ & NA & 128 \\
         Single (linear) & $74.54$ & NA & 512 \\
         Single (nonlinear) & $72.69$ & NA & 128\\
         Single (nonlinear) & $67.78$ & NA & 512 \\
         MoE (linear) &  $95.93 \pm 1.34$ &  $1.160 \pm 0.100$ & 128 (16*8)\\
         MoE (nonlinear) & $\mathbf{99.99 \pm 0.02}$ & $\mathbf{0.008 \pm 0.011}$ & 128 (16*8)\\
    \bottomrule
    \\
        \multicolumn{4}{c}{Setting 4: $\alpha\in (0.5,2)$, $\beta\in (1,2)$, $\gamma\in (0.5,2)$, $\sigma_{p}  = 2$} \\
    \addlinespace[3pt]
    \toprule
        &  Test accuracy ($\%$) & Dispatch Entropy & Number of Filters \\
    \midrule
         Single (linear) & $74.63$ &  NA & 128 \\
         Single (linear) & $72.98$ & NA & 512 \\
         Single (nonlinear) & $68.60$ & NA & 128\\
         Single (nonlinear) & $61.65$ & NA & 512 \\
         MoE (linear) &  $93.30 \pm 1.48$ &  $1.160 \pm 0.155$ & 128 (16*8)\\
         MoE (nonlinear) & $\mathbf{98.92 \pm 1.18}$ & $\mathbf{0.089 \pm 0.120}$ & 128 (16*8)\\ 
    \bottomrule
    \addlinespace[3pt]
    \end{tabular}
    \caption{\textbf{Comparison between MoE (linear) and MoE (nonlinear)} in our setting. We report results of top-1 gating with noise for both linear and nonlinear models. Over ten random experiments, we report the average value $\pm$ standard deviation for both test accuracy and dispatch entropy.}
    \label{tab:synthetic_exp_results_s3}
\end{table}

\noindent\textbf{Synthetic-data experiment setup.} For the experiments on synthetic data, we generate the data according to Definition~\ref{def:data_distribution} with $K=4$, $P=4$ and $d=50$. We consider four parameter settings: 
\begin{itemize}
    \item $\alpha \sim \text{Uniform}(0.5,2)$, $\beta \sim \text{Uniform}(1,2)$, $\gamma \sim \text{Uniform}(0.5,3)$ and $\sigma_{p} = 1$;
    \item $\alpha \sim \text{Uniform}(0.5,2)$, $\beta \sim \text{Uniform}(1,2)$, $\gamma \sim \text{Uniform}(0.5,3)$ and $\sigma_{p}  = 2$;
    \item $\alpha \sim \text{Uniform}(0.5,2)$, $\beta \sim \text{Uniform}(1,2)$, $\gamma \sim \text{Uniform}(0.5,2)$ and $\sigma_p=1$;
    \item $\alpha \sim \text{Uniform}(0.5,2)$, $\beta \sim \text{Uniform}(1,2)$, $\gamma \sim \text{Uniform}(0.5,2)$ and $\sigma_p=2$.
\end{itemize}
We consider mixture of $M=8$ experts for all MoEs and $J=16$ neurons/filters for all experts. For single models, we consider $J=128$ neurons/filters. We train MoEs using Algorithm~\ref{alg:GDrandominit}. Specifically, we train the experts by normalized gradient descent with learning rate $0.001$ and the gating network by gradient descent with learning rate $0.1$. We train single linear/nonlinear models by Adam \citep{kingma2014adam} to achieve the best performance, with learning rate $0.01$ and weight decay 5e-4 for single nonlinear model and learning rate $0.003$ and weight decay $5e-4$ for single linear model.

\noindent\textbf{Synthetic-data experiment results.} In Table~\ref{tab:synthetic_exp_results_s3}, we present the empirical results of single linear CNN, single nonlinear CNN, linear MoE, and nonlinear MoE under settings $3$ and $4$, where $\alpha$ and $\gamma$ follow the same distribution as we assumed in theoretical analysis. Furthermore, we report the total number of filters for both single CNNs and a mixture of CNNs, where the filter size (equal to $50$) is the same for all single models and experts. For linear and nonlinear MoE, there are $16$ filters for each of the $8$ experts, and therefore $128$ filters in total. Note that in the synthetic-data experiment in the main paper, we let the number of filters of single models be the same as MoEs ($128$). Here, we additionally report the performances of single models with $512$ filters, and see if increasing the model size of single models can beat MoE. From Table \ref{tab:synthetic_exp_results_s3}, we observe that: 1. single models perform poorly in all settings; 2. linear MoEs do not perform as well as nonlinear MoEs. Specifically, the final dispatch entropy of nonlinear MoEs is nearly zero while the dispatch entropy of linear MoEs is consistently larger under settings $1$-$4$. This indicates that nonlinear MoEs successfully uncover the underlying cluster structure while linear MoEs fail to do so.
In addition, we can see that even larger single models cannot beat linear MoEs or nonlinear MoEs. This is consistent with Theorem~$\ref{thm:neg}$, where a single model fails under such data distribution regardless of its model size. Notably, by comparing the results in Table \ref{tab:synthetic_exp_results_1} and Table~\ref{tab:synthetic_exp_results_s3}, we can see that a single nonlinear model suffers from overfitting as we increase the number of filters.

\noindent\textbf{Router dispatch examples.} We demonstrate specific examples of router dispatch for MoE (nonlinear) and MoE (linear). The examples of initial and final router dispatch for MoE (nonlinear) are shown in Table~\ref{tab:dispatch_details_nonlinear_toy1} and Table~\ref{tab:dispatch_details_nonlinear_toy2}. Under the dispatch for nonlinear MoE, each expert is given either no data or data that comes from one cluster only. The entropy of such dispatch is thus $0$. The test accuracy of MoE trained under such a dispatch is either $100\%$ or very close to $100\%$, as the expert can be easily trained on the data from one cluster only. An example of the final dispatch for MoE (linear) is shown in 
Table~\ref{tab:dispatch_details_linear_toy}, where clusters are not well separated and an expert gets data from different clusters. The test accuracy under such dispatch is lower ($90.61\%$).

\begin{table}
    \centering
    \begin{tabular}{c c c c c c c c c}
    \toprule
      Expert number  & 1 & 2 & 3 & 4 & 5 & 6 & 7 & 8 \\
    \midrule
         Initial dispatch &  1921& 2032& 1963& 1969& 2075& 1980& 2027& 2033   \\
         Final dispatch   &  0& 3979& 4009&    0&    0& 3971&    0& 4041  \\
    \midrule
         Cluster 1 & 0&    0&    0&       0&    0&       3971&    0&    0  \\
         Cluster 2 & 0&    0&    4009&    0&    0&       0&       0&       0  \\
         Cluster 3 & 0&    0&    0&       0&    0&    0&    0&       4041  \\
         Cluster 4 & 0&    3979&    0&       0& 0&    0&       0&       0  \\
    \bottomrule
    \addlinespace[3pt]
    \end{tabular}
    \caption{Dispatch details of MoE (nonlinear) with test accuracy $100\%$.}
    \label{tab:dispatch_details_nonlinear_toy1}
\end{table}

\begin{table}
    \centering
    \begin{tabular}{c c c c c c c c c}
    \toprule
      Expert number  & 1 & 2 & 3 & 4 & 5 & 6 & 7 & 8 \\
    \midrule
         Initial dispatch & 1978  &  2028 & 2018 &   1968 & 2000 & 2046 & 2000 &   1962   \\
         Final dispatch   & 3987&    4& 3975&    6&    0&  1308& 4009& 2711  \\
    \midrule
         Cluster 1 &    0&    0& 3971&    0&    0&    0&    0&    0  \\
         Cluster 2 &    0&    0&    0&    0&    0&    4& 4005&    0  \\
         Cluster 3 &    8&    4&    4&    6&    0&  1304&    4& 2711  \\
         Cluster 4 & 3979&    0&    0&    0&    0&    0&    0&    0  \\
    \bottomrule
    \addlinespace[3pt]
    \end{tabular}
    \caption{Dispatch details of MoE (nonlinear) with test accuracy $99.95\%$.}
    \label{tab:dispatch_details_nonlinear_toy2}
\end{table}

\begin{table}
    \centering
    \begin{tabular}{c c c c c c c c c}
    \toprule
      Expert number  & 1 & 2 & 3 & 4 & 5 & 6 & 7 & 8 \\
    \midrule
         Initial dispatch & 1969& 2037& 1983& 2007& 1949& 1905& 2053& 2097   \\
         Final dispatch   & 136& 2708& 6969& 5311&   27&   87&    4&  758  \\
    \midrule
         Cluster 1 &  0 &  630& 1629& 1298&   27&   87&    4&  296  \\
         Cluster 2 & 136& 1107& 1884&  651&    0&    0&    0&  231  \\
         Cluster 3 &  0&  594& 1976& 1471&    0&    0&    0&    0  \\
         Cluster 4 & 0 &377& 1480& 1891&    0&    0&    0&  231  \\
    \bottomrule
    \addlinespace[3pt]
    \end{tabular}
    \caption{Dispatch details of MoE (linear) with test accuracy $90.61\%$.}
    \label{tab:dispatch_details_linear_toy}
\end{table}

\noindent\textbf{MoE during training.} We further provide figures that illustrate the growth of the inner products between expert/router weights and feature/center signals during training. Specifically, since each expert has multiple neurons, we plot the max absolute value of the inner product over the neurons of each expert. In Figure~\ref{fig:Nonlinear}, we demonstrate the training process of MoE (nonlinear), and in Figure~\ref{fig:Linear}, we demonstrate the training process of MoE (linear). The data is the same as setting $1$ in Table~\ref{tab:synthetic_exp_results_1}, with $\alpha\in(0.5,2)$, $\beta\in(1,2)$, $\gamma\in(0.5,3)$ and $\sigma_p=1$. We can observe that, in the top left sub-figure of Figure~\ref{fig:Nonlinear} for MoE (nonlinear), the max inner products between expert weight and feature signals exhibit a property that each expert picks up one feature signal quickly. Similarly, as shown in the bottom right sub-figure, the router picks up the corresponding center signal. Meanwhile, the nonlinear experts almost do not learn center signals and the magnitude of the inner products between router weight and feature signals remain small. However, for MoE (linear), as shown in the top two sub-figures of Figure~\ref{fig:Linear}, an expert does not learn a specific feature signal, but instead learns multiple feature and center signals. Moreover, as demonstrated in the bottom sub-figures of Figure~\ref{fig:Linear}, the magnitude of the inner products between router weight and feature signals can be even larger than the inner products between router weight and center signals.   
  
\begin{figure}[!htbp]
    \centering
    \includegraphics[scale=0.6]{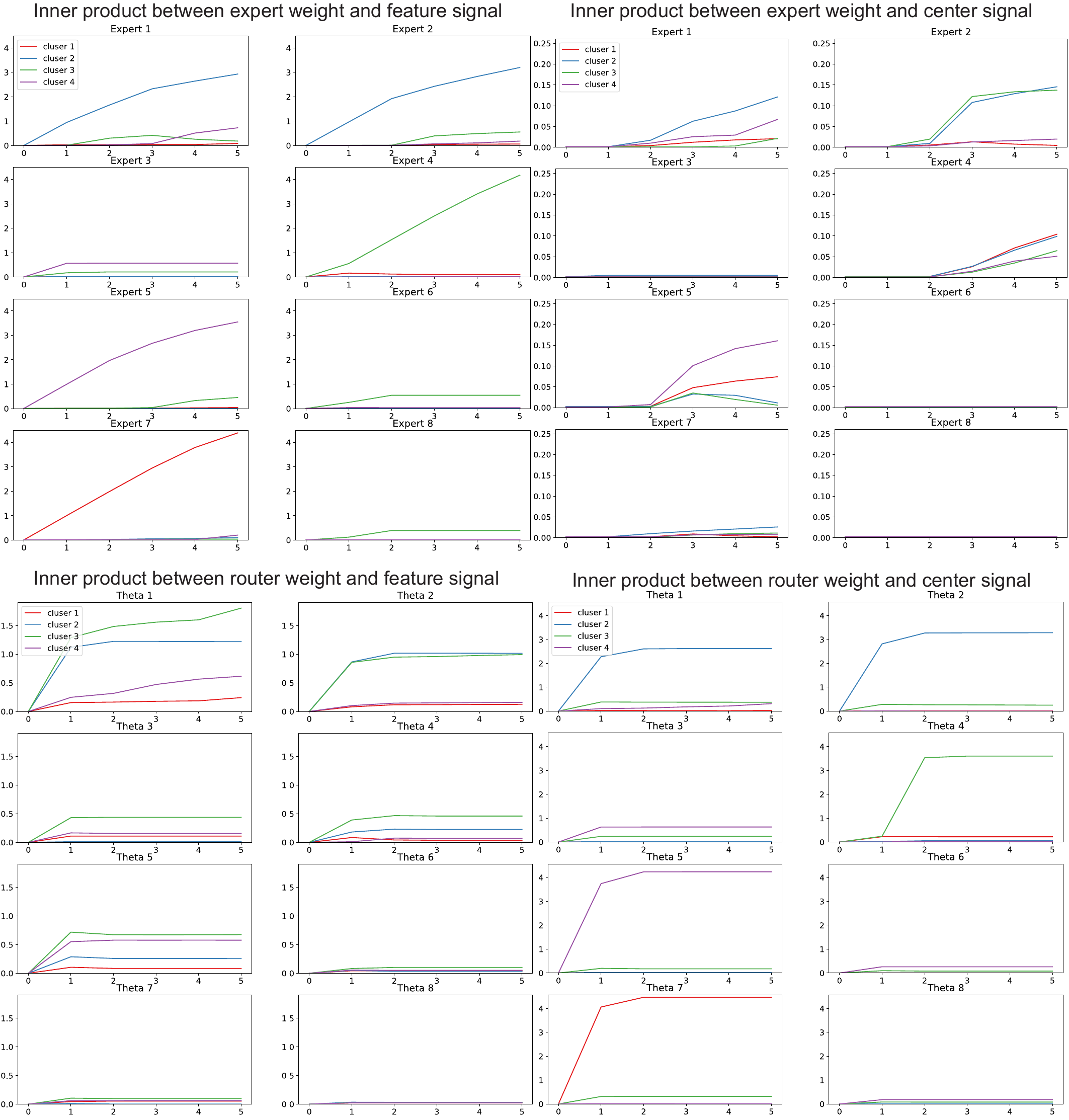}
    \caption{\textbf{Mixture of nonlinear experts.} Growth of inner product between expert/router weight and center/feature vector.}
    \label{fig:Nonlinear}
\end{figure}

\begin{figure}[!htbp]
    \centering
    \includegraphics[scale=0.6]{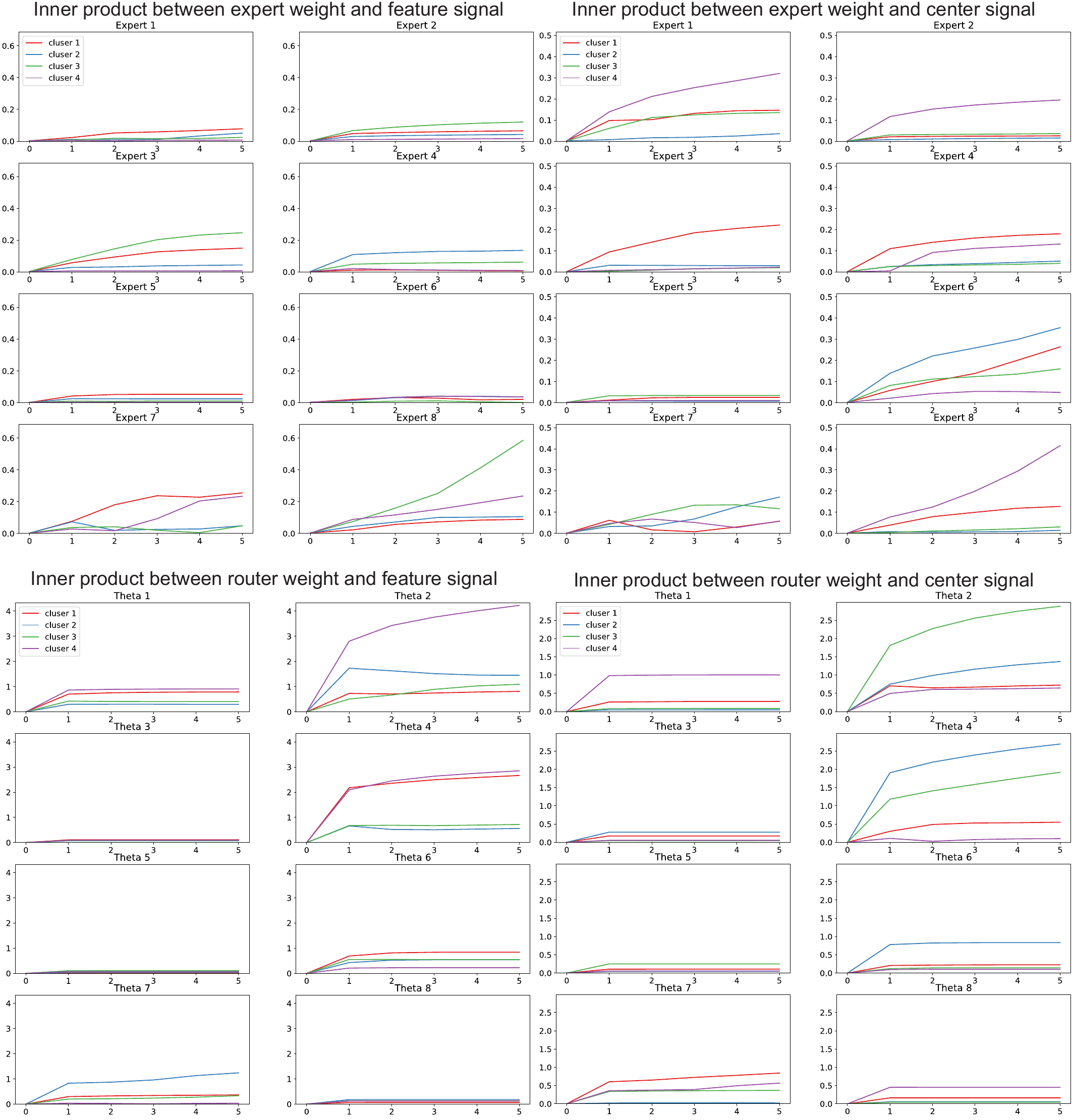}
    \caption{\textbf{Mixture of linear experts.} Growth of inner product between expert/router weight and center/feature vector.}
    \label{fig:Linear}
\end{figure}

\noindent\textbf{Verification of Theorem~\ref{thm:neg}.} In Table~\ref{tab:verification}, we provide the performances of single models with different activation functions under setting $3$, where $\alpha,\gamma \in (1,2)$ follow the same distribution. In Table~\ref{tab:verification_s1}, we further report the performances of single models with different activation functions under setting $1$ and setting $2$. Empirically, even when $\alpha$ and $\gamma$ do not share the same distribution, single models still fail. Note that, for Tables~\ref{tab:verification} and \ref{tab:verification_s1}, the numbers of filters for single models are $128$. 

\begin{table}
    \centering
    \begin{tabular}{c c c}
    \toprule
        Activation  & Optimal Accuracy ($\%$) & Test Accuracy ($\%$) \\
    \midrule
        Linear & $87.50\%$ & $74.81\%$  \\
        Cubic & $87.50\%$ & $72.69\%$  \\
        Relu & $87.50\%$ & $73.45\%$  \\
        Celu & $87.50\%$ & $76.91\%$  \\
        Gelu & $87.50\%$ & $74.01\%$  \\
        Tanh & $87.50\%$ & $74.76\%$  \\
    \bottomrule
    \addlinespace[3pt]
    \end{tabular}
    \caption{\textbf{Verification of Theorem~\ref{thm:neg} (single expert performs poorly)}. Test accuracy of single linear/nonlinear models with different activation functions. Data is generated according to Definition~\ref{def:data_distribution} with $\alpha,\gamma\in(1,2)$, $\beta\in(1,2)$ and $\sigma_p=1$.}
    \label{tab:verification}
\end{table}

\begin{table}
    \centering
    \begin{tabular}{c c c}
    \toprule
        Activation  & Setting $1$ & Setting $2$ \\
    \midrule
        Linear & $68.71 \%$ & $60.59 \%$  \\
        Cubic & $79.48 \%$ & $72.29 \%$  \\
        Relu & $72.28 \%$ & $80.12 \%$  \\
        Celu & $81.75 \%$ & $78.99 \%$  \\
        Gelu & $79.04 \%$ & $82.01 \%$  \\
        Tanh & $81.72 \%$ & $81.03 \%$  \\
    \bottomrule
    \addlinespace[3pt]
    \end{tabular}
    \caption{\textbf{Single expert performs poorly (setting 1\&2).} Test accuracy of single linear/nonlinear models with different activation functions. Data is generated according to Definition~\ref{def:data_distribution} with $\alpha\in (0.5,2)$, $\beta\in (1,2)$, $\gamma\in (0.5,3), \sigma_{p}=1$ for setting $1$. And we have $\alpha\in (0.5,2)$, $\beta\in (1,2)$, $\gamma\in (0.5,3), \sigma_{p}=1$ for setting $2$.}
    \label{tab:verification_s1}
\end{table}

\noindent\textbf{Load balancing loss.} In Table~\ref{tab:linear_load_balancing}, we present the results of linear MoE with load balancing loss and directly compare it with nonlinear MoE without load balancing loss. Load balancing loss guarantees that the experts receive similar amount of data and prevents MoE from activating only one or few experts. However, on the data distribution that we study, load balancing loss is not the key to the success of MoE: the single experts cannot perform well on the entire data distribution and must diverge to learn different labeling functions with respect to each cluster.

\begin{table}
    \centering
    \begin{tabular}{c c c}
    \toprule
        & Linear MoE with Load Balancing & Nonlinear MoE without Load Balancing \\
    \midrule
        Setting $1$ & $ 93.81 \pm 1.02$ & $\mathbf{99.46 \pm 0.55}$  \\
        Setting $2$ & $ 89.20 \pm 2.20$ & $\mathbf{98.09 \pm 1.27}$  \\
        Setting $3$ & $ 95.12\pm 0.58$ & $\mathbf{99.99 \pm 0.02}$  \\
        Setting $4$ & $ 92.50 \pm 1.55$ & $\mathbf{98.92\pm 1.18}$  \\
    \bottomrule
    \addlinespace[3pt]
    \end{tabular}
    \caption{\textbf{Load balancing loss.} We report the results for linear MoE with load balancing loss and compare them with our previous results on nonlinear MoE without load balancing loss. Over ten random experiments, we report the average test accuracy ($\%$) $\pm$ standard deviation. Setting $1$-$4$ follows the data distribution introduced above.}
    \label{tab:linear_load_balancing}
\end{table}

\subsection{Experiments on Image Data}
\begin{figure*}[!htbp]
    \centering
    \includegraphics[scale=0.6]{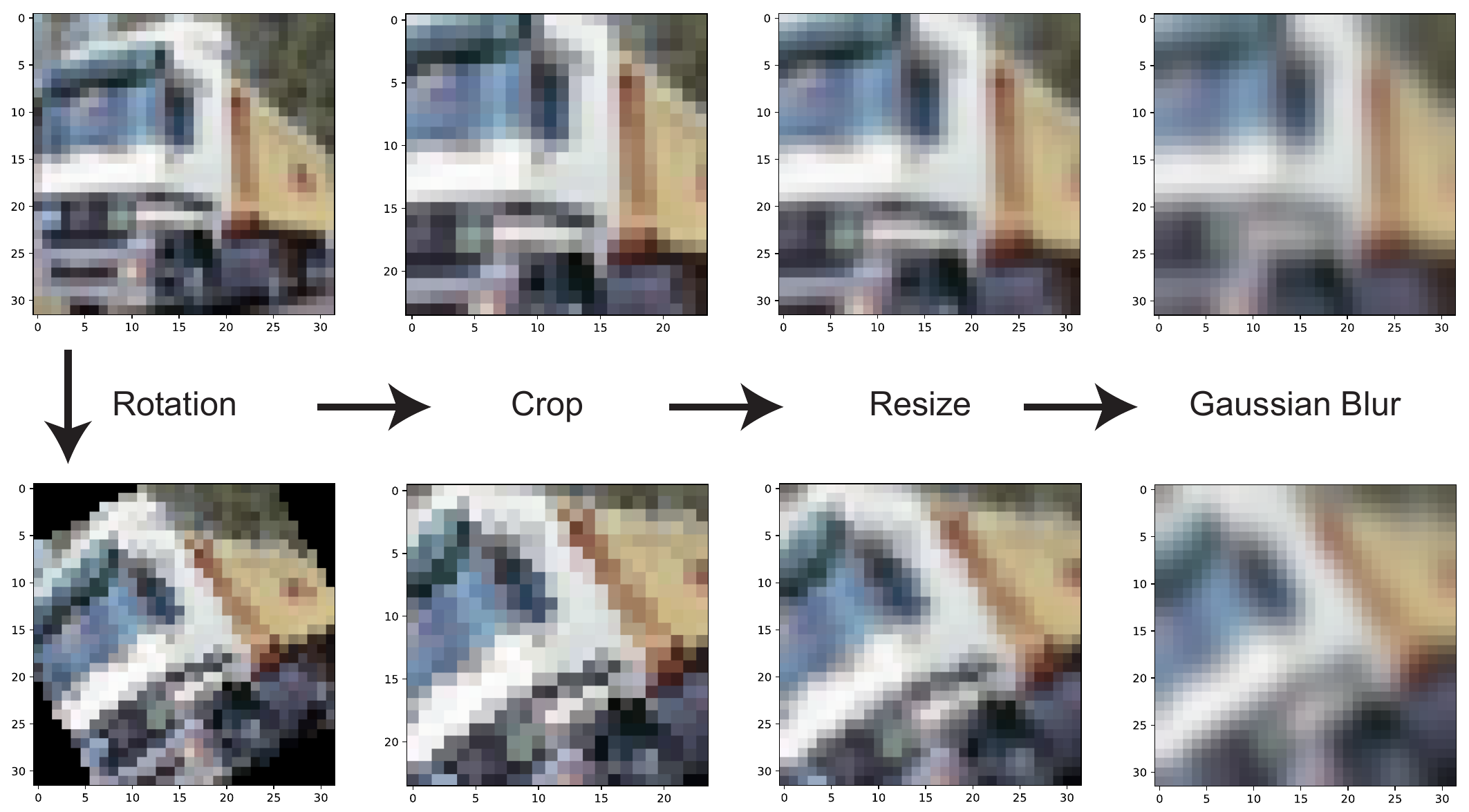}
    \caption{\textbf{Examples of the CIFAR-10-Rotate dataset.} Both the original image and the rotated image are processed in the same way, where we crop the image to $(24,24)$, resize to $(32,32)$ and apply random Gaussian blur.}
    \label{fig:Data}
\end{figure*}

\noindent\textbf{Datasets.} We consider CIFAR-10 \citep{Krizhevsky09learningmultiple} with the $10$-class classification task, which contains $50,000$ training examples and $10,000$ testing examples. For CIFAR-10-Rotate, we design a binary classification task by copying and rotating all images by $30$ degree and let the model predict if an image is rotated. In Figure~\ref{fig:Data}, we demonstrate the positive and negative examples of CIFAR-10-Rotate. Specifically, we crop the rotated images to $(24,24)$, and resize to $(32,32)$ for model architectures that are designed on image size $(32,32)$. And we further apply random Gaussian noise to all images to avoid the models taking advantage of image resolutions.

\noindent\textbf{Models.} For the simple CNN model, we consider CNN with $2$ convolutional layers, both with kernel size $3$ and ReLU activation followed by max pooling with size 2 and a fully connected layer. The number of filters of each convolutional layer is respectively $64$, $128$. 

\noindent\textbf{CIFAR-10 Setup.} For real-data experiments on CIFAR-10, we apply the commonly used transforms on CIFAR-10 before each forward pass: random horizontal flips and random crops (padding the images on all sides with $4$ pixels and randomly cropping to $(32,32)$). And as conventionally, we normalize the data by channel. We train the single CNN model with SGD of learning rate $0.01$, momentum $0.9$ and weight decay 5e-4. And we train single MobileNetV2 and single ResNet18 with SGD of learning rate $0.1$, momentum $0.9$ and weight decay 5e-4 to achieve the best performances. We train MoEs according to Algorithm~\ref{alg:GDrandominit}, with normalized gradient descent on the experts and SGD on the gating networks. Specifically, for MoE (ResNet18) and MoE (MobileNetV2), we use normalized gradient descent of learning rate $0.1$ and SGD of learning rate 1e-4, both with momentum $0.9$ and weight decay of 5e-4. For MoE (CNN), we use normalized gradient descent of learning rate $0.01$ and SGD of learning rate 1e-4, both with momentum $0.9$ and weight decay of 5e-4. We consider top-1 gating with noise and load balancing loss for MoE on both datasets, where the multiplicative coefficient of load balancing loss is set at 1e-3. All models are trained for $200$ epochs to achieve convergence. 

\noindent\textbf{CIFAR-10-Rotate Setup.} For experiments on CIFAR10-Rotate, the data is normalized by channel as the same as in CIFAR-10 before each forward pass. We train the single CNN, single MobileNetV2 and single ResNet18 by SGD with learning rate $0.01$, momentum $0.9$ and weight decay 5e-4 to achieve the best performances. And we train MoEs by Algorithm~\ref{alg:GDrandominit} with normalized gradient descent learning rate $0.01$ on the experts and with SGD of learning rate 1e-4 on the gating networks, both with momentum $0.9$ and weight decay of 5e-4. We consider top-1 gating with noise and load balancing loss for MoE on both datasets, where the multiplicative coefficient for load balancing loss is set at 1e-3. All models are trained for $50$ epochs to achieve convergence.

\noindent\textbf{Visualization.} In Figure~\ref{fig:visual_cifar}, we visualize the latent embedding learned by MoEs (ResNet18) for the 10-class classification task in CIFAR-10 as well as the binary classification task in CIFAR-10-Rotate. We visualize the data with the same label $y$ to see if cluster structures exist within each class. For CIFAR-10, we choose $y=1$ ("car"), and plot the latent embedding of data with $y=1$ using t-SNE on the left subfigure, which does not show an salient cluster structure. For CIFAR-10-Rotate, we choose $y=1$ ("rotated") and visualize the data with $y=1$ in the middle subfigure. Here, we can observe a clear clustering structure even though the class signal is not provided during training. We take a step further to investigate what is in each cluster in the right subfigure. We can observe that most of the examples in the ``frog'' class fall into one cluster, while examples of ``ship'' class mostly fall into the other cluster.

\begin{figure*}[!htbp]
    \begin{center}
    \includegraphics[scale=0.6]{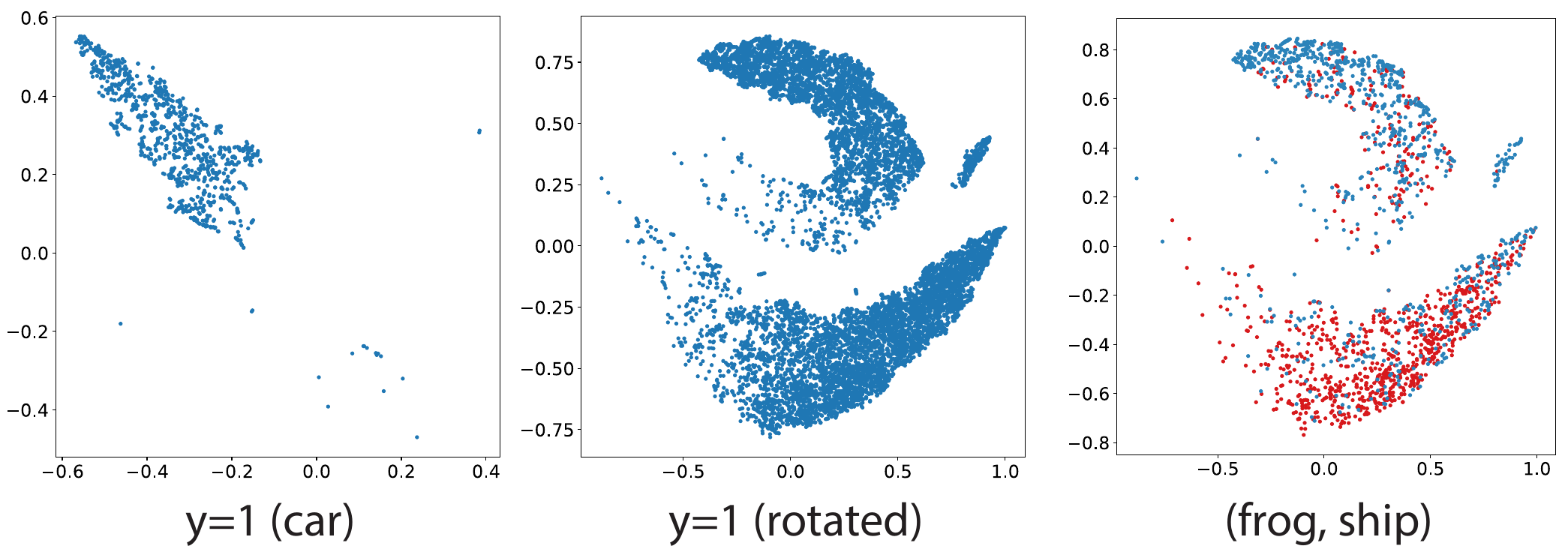}
    \end{center}
    \caption{Visualization of the latent embedding on CIFAR-10 and CIFAR-10-Rotate with fixed label $y$. The left figure denotes the visualization of CIFAR-10 when label $y$ is fixed to be $1$ (car). The central figure represents the visualization of CIFAR-10-Rotate when label $y$ is fixed to be $1$ (rotated). On the right figure, red denotes that the data is from the ship class, and blue denotes that the data is from the frog class.}
    \label{fig:visual_cifar}
\end{figure*}

\begin{figure*}[!htbp]
    \begin{center}
    \includegraphics[scale=0.25]{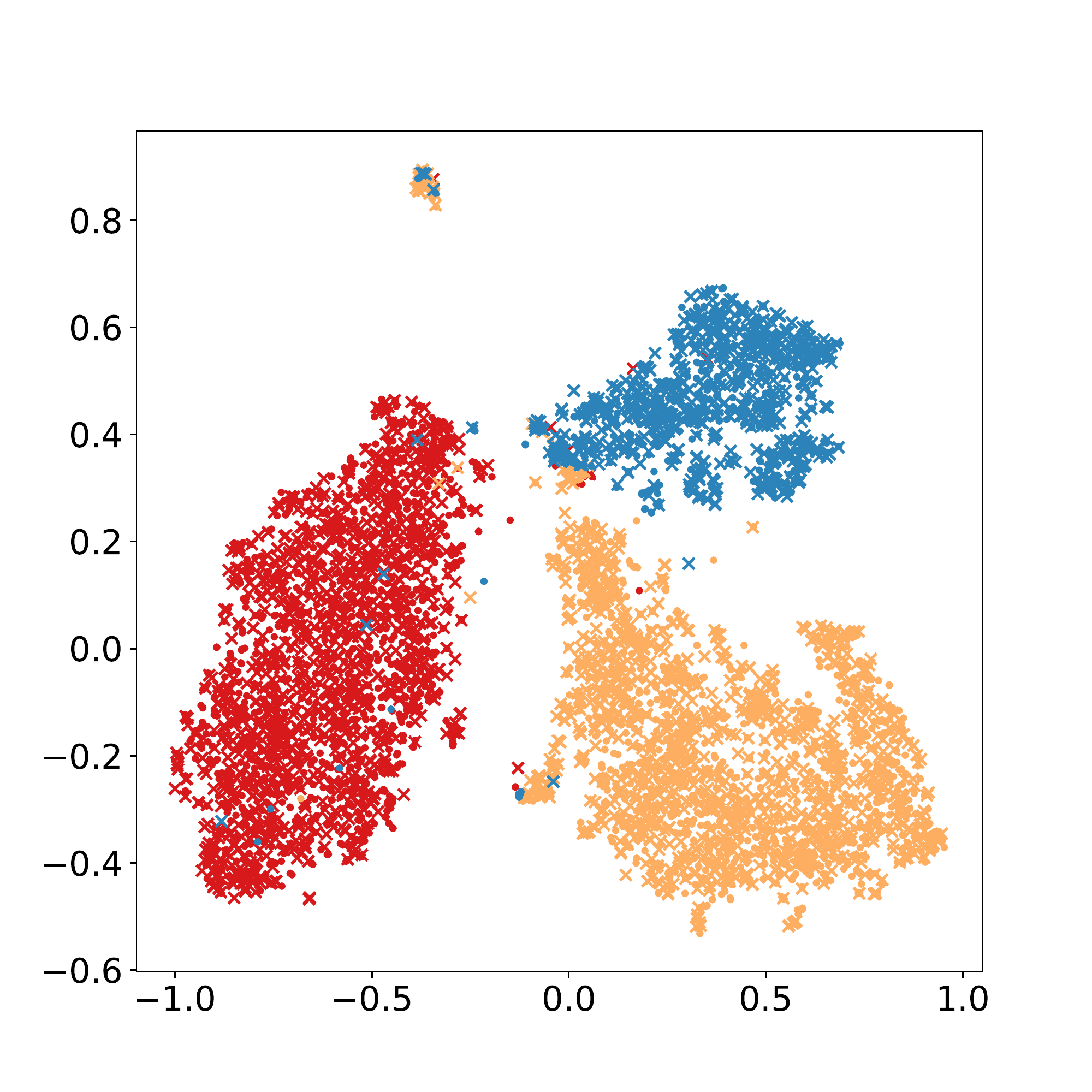}
    \end{center}
    \caption{The distribution of text embedding of the multilingual sentiment analysis dataset. The embedding is generated by the pre-trained BERT multilingual base model and visualized on 2d space using t-SNE. Each color denotes a linguistic source, including English, French, and Russian.}
    \label{fig:multilingual}
\end{figure*}

\subsection{Experiments on Language Data}

\begin{table}[]
    \centering
    \begin{tabular}{c c c}
        \toprule 
            & Single & MoE \\
        \midrule
            Accuracy & $ 74.13\%$ & $76.22\%$ \\
        \bottomrule
        \addlinespace[3pt]
    \end{tabular}
    \caption{The test accuracy of the single classifier vs. MoE classifier.}
    \label{tab:multilingual}
\end{table}

\begin{table}[]
    \centering
    \begin{tabular}{c c c c c}
        \toprule 
            & Expert $1$ & Expert $2$ & Expert $3$ & Expert $4$ \\
        \midrule
            English & $1,374$ & $3,745$ & $2,999$ & $\mathbf{31,882}$ \\
            French & $\mathbf{23,470}$ & $3,335$ & $\mathbf{13,182}$ & $13$ \\
            Russian & $833$ & $\mathbf{9,405}$ & $7,723$ & $39$ \\
        \bottomrule
        \addlinespace[3pt]
    \end{tabular}
    \caption{The final router dispatch details with regard to the linguistic source of the test data.}
    \label{tab:multilingual_dispatch}
\end{table}

Here we provide a simple example of how MoE would work for multilingual tasks. We gather multilingual sentiment analysis data from the source of English (Sentiment140 \citep{go2009twitter}) which is randomly sub-sampled to $200,000$ examples, Russian (RuReviews \citep{smetanin2019sentiment}) which contains $90,000$ examples, and French \citep{blard2020sentiment} which contains $200,000$ examples. We randomly split the dataset into $80\%$ training data and $20\%$ test data. We use a pre-trained BERT multilingual base model \citep{devlin2018bert} to generate text embedding for each text and train $1$-layer neural network with cubic activation as the single model. For MoE, we still let $M=4$ with each expert sharing the same architecture as the single model. In Figure~\ref{fig:multilingual}, we show the visualization of the text embeddings in the $2$d space via t-SNE, where each color denotes a linguistic source, with $\cdot$ representing a positive example and $\times$ representing a negative example. Data from different linguistic sources naturally form different clusters. And within each cluster, positive and negative data exist.

In Table~\ref{tab:multilingual}, we demonstrate the test accuracy of a single classifier and MoE on the multilingual sentiment analysis dataset. And in Table~\ref{tab:multilingual_dispatch}, we show the final router dispatch details of MoE to each expert with regard to the linguistic source of the text. Notably, MoE learned to distribute examples largely according to the original language.

\section{Proof of Theorem~\ref{thm:neg}}



Because we are using CNNs as experts, different ordering of the patches won't affect the value of $F(\xb)$. So for $(\xb ,y)$ drawn from $\cD$ in Definition~\ref{def:data_distribution}, we can assume that the first patch $\xb^{(1)}$ is feature signal, the second patch $\xb^{(2)}$ is cluster-center signal, the third patch $\xb^{(3)}$ is feature noise. The other patches $\xb^{(p)}, p \geq 4$ are random noises. Therefore, we can rewrite $\xb = [\alpha y\vb_{k}, \beta\cb_{k}, \gamma\epsilon\vb_{k'}, \bxi]$, where $\bxi = [\bxi_{4}, \ldots, \bxi_{P}]$ is a Gaussian matrix of size $\RR^{d\times (P-3)}$.

\begin{proof}[Proof of Theorem~\ref{thm:neg}]
Conditioned on the event that $y= -\epsilon$, points $([\alpha y\vb_{k}, \beta\cb_{k}, -\gamma y\vb_{k'}, \bxi], y)$, $\big([-\alpha y\vb_{k}, \beta\cb_{k}, \gamma y\vb_{k'}, \bxi], -y\big)$, $\big([\gamma y\vb_{k'}, \beta\cb_{k'}, -\alpha y\vb_{k}, \bxi], y\big)$, $\big([-\gamma y\vb_{k'}, \beta\cb_{k'}, \alpha y\vb_{k}, \bxi], -y\big)$ 
 follow the same distribution because $\gamma$ and $\alpha$ follow the same distribution, and $y$ and $-y$ follow the same distribution. Therefore, we have
\begin{align*}
&4\mathbb{P}\big(yF(\xb)\leq 0|\epsilon = -y\big)\\
&= \EE\bigg[\underbrace{\ind(yF([\alpha y\vb_{k}, \beta\cb_{k}, -\gamma y\vb_{k'}, \bxi])\leq 0)}_{I_{1}} + \underbrace{\ind(-yF([-\alpha y\vb_{k}, \beta\cb_{k}, \gamma y\vb_{k'}, \bxi])\leq0)}_{I_{2}}\\
&\qquad+ \underbrace{\ind(yF([\gamma y\vb_{k'}, \beta\cb_{k'}, -\alpha y\vb_{k}, \bxi])\leq0)}_{I_{3}} + \underbrace{\ind(-yF([-\gamma y\vb_{k'}, \beta\cb_{k'}, \alpha y\vb_{k}, \bxi])\leq0)\bigg]}_{I_{4}}.
\end{align*}
It is easy to verify the following fact
\begin{align*}
&\Big(yF([\alpha y\vb_{k}, \beta\cb_{k}, -\gamma y\vb_{k'}, \bxi])\Big) + \Big(-yF([-\alpha y\vb_{k}, \beta\cb_{k}, \gamma y\vb_{k'}, \bxi])\Big)\\
&\qquad + \Big(yF([\gamma y\vb_{k'}, \beta\cb_{k'}, -\alpha y\vb_{k}, \bxi])\Big) + \Big(-yF([-\gamma y\vb_{k'}, \beta\cb_{k'}, \alpha y\vb_{k}, \bxi])\Big)\\
&= \bigg(yf(\alpha y\vb_{k}) +  yf(\beta\cb_{k}) + yf(-\gamma y\vb_{k'}) + \sum_{p=4}^{P} yf(\bxi_{p})\bigg)\\
&\qquad + \bigg(-yf(-\alpha y\vb_{k}) -  yf(\beta\cb_{k}) - yf(\gamma y\vb_{k'}) - \sum_{p=4}^{P} yf(\bxi_{p})\bigg)\\
&\qquad + \bigg(yf(\gamma y\vb_{k'}) +  yf(\beta\cb_{k'}) + yf(-\alpha y\vb_{k}) + \sum_{p=4}^{P} yf(\bxi_{p})\bigg)\\
&\qquad + \bigg(-yf(-\gamma y\vb_{k'}) -  yf(\beta\cb_{k'}) - yf(\alpha y\vb_{k}) - \sum_{p=4}^{P} yf(\bxi_{p})\bigg)\\
&=0.
\end{align*}
By pigeonhole principle, at least one of $I_{1}, I_{2}, I_{3}, I_{4}$ is non-zero. This further implies that $4\mathbb{P}\big(yF(\xb)\leq 0|\epsilon = -y\big) \geq 1$. Applying $\mathbb{P}(\epsilon = - y) = 1/2$, we have that
\begin{align*}
 \mathbb{P}\big(yF(\xb)\leq 0\big) \geq \mathbb{P}\big(yF(\xb)\leq 0)|\epsilon = -y\big)\mathbb{P}(\epsilon = -y) \geq 1/8,  
\end{align*}
which completes the proof.
\end{proof}

\section{Smoothed Router}
In this section, we will show that the noise term provides a smooth transition between different routing behavior. All the results in this section is independent from our NN structure and its initialization. We first present a general version of Lemma~\ref{lm:Msmoothly} with its proof.
\begin{lemma}[Extension of Lemma~\ref{lm:Msmoothly}]\label{lm:exMsmooth}
Let $\hb, \hat{\hb} \in \RR^{M}$ to be the output of the gating network and $\{r_{m}\}_{m=1}^{M}$ to be the noise independently drawn from $\cD_{r}$. Denote $\pb, \hat{\pb} \in \RR^{M}$ to be the probability that experts get routed, i.e., $p_{m} = \mathbb{P}(\argmax_{m'\in[M]}\{h_{m'} + r_{m'}\} = m )$, $\hat{p}_{m} = \mathbb{P}(\argmax_{m'\in [M]}\{\hat{h}_{m'} + r_{m'}\} =m )$. Suppose the probability density function of $\cD_{r}$ is bounded by $\kappa$, Then we have that $\|\pb - \hat{\pb}\|_{\infty} \leq (\kappa M^{2})\cdot \|\hb - \hat{\hb}\|_{\infty}$.
\end{lemma}
\begin{proof}
Given random variable $\{r_{m}\}_{m=1}^{M}$, let us first consider the event that $\argmax_{m}\{h_{m} + r_{m}\} \not= \argmax_{m}\{\hat{h}_{m}+ r_{m}\}$. Let $m _{1} = \argmax_{m}\{h_{m} + r_{m}\}$ and $m _{2} = \argmax_{m}\{\hat{h}_{m} + r_{m}\}$, then we have that
\begin{align}
&h_{m_{1}} + r_{m_{1}} \geq h_{m_{2}} + r_{m_{2}}, \hat{h}_{m_{2}} + r_{m_{2}} \geq \hat{h}_{m_{1}} + r_{m_{1}}, \notag
\end{align}
which implies that
\begin{align}
\hat{h}_{m_{2}} -  \hat{h}_{m_{1}} \geq r_{m_{1}} - r_{m_{2}} \geq  h_{m_{2}} -  h_{m_{1}}. \label{eq:gap}
\end{align}
Define $C(m_{1},m_{2}) = (\hat{h}_{m_{2}} - \hat{h}_{m_{1}} + h_{m_{2}} - h_{m_{1}})/2$, then 
\eqref{eq:gap} implies that 
\begin{align}
|r_{m_{1}} - r_{m_{2}} - C(m_{1},m_{2})| \leq |\hat{h}_{m_{2}} - \hat{h}_{m_{1}} - h_{m_{2}}  + h_{m_{1}}|/2 \leq \|\hat{\hb} - \hb\|_{\infty} \label{smooth10}.    
\end{align}
Therefore, we have that,
\begin{align*}
&\mathbb{P}(\argmax_{m}\{h_{m} + r_{m}\} \not= \argmax_{m}\{\hat{h}_{m}+ r_{m}\} )\\
&\leq \mathbb{P}(\exists m_{1} \not= m_{2} \in [M], \text{ s.t. } |r_{m_{1}} - r_{m_{2}} - C(m_{1}, m_{2})|  \leq  \|\hat{\hb} - \hb\|_{\infty})\\
&\leq \sum_{m_{1} < m_{2}}\mathbb{P}\big(|r_{m_{1}} - r_{m_{2}} - C(m_{1}, m_{2})|  \leq  \|\hat{\hb} - \hb\|_{\infty}\big)\\
&= \sum_{m_{1} < m_{2}}\EE\Big[\mathbb{P}\big( r_{m_{2}} + C(m_{1}, m_{2})-\|\hat{\hb} - \hb\|_{\infty} \leq r_{m_{1}}  \leq   r_{m_{2}} + C(m_{1}, m_{2})+\|\hat{\hb} - \hb\|_{\infty}\big)\Big|r_{m_{2}}\Big]\\
&\leq (\kappa M^{2})\cdot \|\hat{\hb} - \hb\|_{\infty},
\end{align*}
where the first inequality is by \eqref{smooth10}, the second inequality is by union bound and the last inequality is due to the fact that the probability density function of $r_{m_{1}}$ is bounded by $\kappa$.
Then we have that for $i \in [M]$,
\begin{align*}
|p_i - \hat{p}_i|  &\leq \bigg|\EE \bigg[\ind\big(\argmax_{m}\{\hat{h}_{m} + r_{m}\} = i\big) - \ind\big(\argmax_{m}\{h_{m} + r_{m}\} = i\big)  \bigg]\bigg|\\ 
&\leq \EE \bigg|\ind\big(\argmax_{m}\{\hat{h}_{m} + r_{m}\} = i\big) - \ind\big(\argmax_{m}\{h_{m} + r_{m}\} = i\big)  \bigg|\\
&\leq \mathbb{P}\big(\argmax_{m}\{\hat{h}_{m} + r_{m}\} \not= \argmax_{m}\{h_{m}+ r_{m}\} \big)\\
&\leq (\kappa M^{2})\cdot \|\hat{\hb} - \hb\|_{\infty},
\end{align*}
which completes the proof.
\end{proof}
\begin{remark}
A widely used choice of $\cD_{r}$ in Lemma~\ref{lm:exMsmooth} is uniform noise Unif[a, b], in which case the density function can be upper bounded by $1/(b-a)$. Another widely used choice of $\cD_{r}$ is Gaussian noise $\cN(0, \sigma_{r}^{2})$, in which case the density function can be upper bounded by $1/(\sigma_{r}\sqrt{2\pi})$. Increase the range of uniform noise or increase the variance of the Gaussian noise will result in a smaller  density function upper bound and a smoother behavior of routing. In our paper, we consider unif[0,1] for simplicity, in which case the the density function can be upper bounded by $1$ ($\kappa = 1$).
\end{remark}

The following Lemma shows that when two gate network outputs are close, the router will distribute the examples to those corresponding experts with nearly the same probability. 

\begin{lemma}\label{lm:Msmooth2}
Let $\hb \in \RR^{M}$ be the output of the gating network and $\{r_{m}\}_{m=1}^{M}$ be the noise independently drawn from Unif[0,1]. Denote the probability that experts get routed  by $\pb$, i.e., $p_{m} = \mathbb{P}(\argmax_{m'}\{h_{m'} + r_{m'}\} = m )$. Then we have that 
\begin{align*}
|p_{m} - p_{m'}| \leq M^{2}|h_{m} - h_{m'}|.
\end{align*}
\end{lemma}
\begin{proof}
Construct $\hat{\hb}$ as copy of $\hb$ and permute its $m, m'$-th element. Denote the corresponding probability vector as $\hat{\pb}$. Then it is obviously that 
$|p_{m} - p_{m'}| = \|\pb - \hat{\pb}\|_{\infty}$ and $|h_{m} - h_{m'}|= \|\hat{\hb} - \hb\|_{\infty}$. Applying Lemma~\ref{lm:Msmoothly} completes the proof.
\end{proof}

The following lemma shows that the router won't route examples to the experts with small gating network outputs, which saves computation and improves the performance.

\begin{lemma}\label{lm:keeporder}
Suppose the noise $\{r_{m}\}_{m=1}^{M}$ are independently drawn from Unif[0,1] and $h_{m}(\xb; \bTheta) \leq \max_{m'}h_{m'}(\xb; \bTheta) - 1 $, example $\xb$ will not get routed to expert $m$.
\end{lemma}
\begin{proof}
Because $h_{m}(\xb; \bTheta) \leq \max_{m'}h_{m'}(\xb; \bTheta) - 1$ implies that for any Uniform noise $\{r_{m'}\}_{m' \in [M]}$ we have that
\begin{align*}
h_{m}(\xb; \bTheta) + r_{m} \leq \max_{m'}h_{m'}(\xb; \bTheta) \leq \max_{m'}\{h_{m'}(\xb; \bTheta) + r_{m'}\},
\end{align*}
where the first inequality is by $r_{m} \leq 1$, the second inequality is by $r_{m'} \geq 0, \forall m' \in [M]$.
\end{proof}

\section{Initialization of the Model} \label{sec: initial}
Before we look into the detailed proof of Theorem~\ref{thm: MoE}, let us first discuss some basic properties of the data distribution and our MoE model. For simplicity of notation, we simplify $(\xb_{i}, y_{i}) \in \Omega_{k}$ as $i \in \Omega_{k}$.

\noindent\textbf{Training Data Set Property.}
Because we are using CNNs as experts, different ordering of the patches won't affect the value of $F(\xb)$. So for $(\xb ,y)$ drawn from $\cD$ in Definition~\ref{def:data_distribution}, we can assume that the first patch $\xb^{(1)}$ is feature signal, the second patch $\xb^{(2)}$ is cluster-center signal, the third patch $\xb^{(3)}$ is feature noise. The other patches $\xb^{(p)}, p \geq 4$ are random noises. Therefore, we can rewrite $\xb = [\alpha y\vb_{k}, \beta\cb_{k}, \gamma\epsilon\vb_{k'}, \bxi]$, where $\bxi = [\bxi_{4}, \ldots, \bxi_{P}]$ is a Gaussian matrix of size $\RR^{d\times (P-3)}$. According to the type of the feature noise, we further divide $\Omega_{k}$ into $\Omega_{k} = \cup \Omega_{k,k'}$ based on the feature noise, i.e. $\xb \in \Omega_{k,k'}$ if $\xb = [\alpha y\vb_{k},\beta \cb_{k}, \gamma \epsilon \vb_{k'}, \bxi]$. To better characterize the router training, we need to break down $\Omega_{k,k'}$ into $\Omega_{k,k'}^{+}$ and $\Omega_{k,k'}^{-}$. Denote by $\Omega_{k,k'}^{+}$ the set that $\{y_{i}=\epsilon_{i} | i \in \Omega_{k,k'}\}$, by $\Omega_{k,k'}^{-}$ the set that $\{y_{i}=-\epsilon_{i} | i \in \Omega_{k,k'}\}$.
\begin{lemma}\label{lm:data balance}
With probability at least $1 - \delta$, the following properties hold for all $k \in [K]$,
\begin{align}
\sum_{i \in \Omega_{k}}y_{i}\beta_{i}^{3} =  \tilde{O}(\sqrt{n}), \sum_{i \in \Omega_{k}}\alpha_{i}^{3} = \EE[\alpha^{3}] \cdot n/K+ \tilde{O}(\sqrt{n}), \sum_{i \in \Omega_{k}}y_{i}\epsilon_{i}\gamma_{i}^{3}  = \tilde{O}(\sqrt{n}), \label{eq:balance1}
\end{align}
\begin{align}
\sum_{i \in \Omega_{k,k'}^{+}}y_{i}\alpha_{i} = \tilde{O}(\sqrt{n}), \sum_{i \in \Omega_{k,k'}^{-}}y_{i}\alpha_{i} = \tilde{O}(\sqrt{n}), \sum_{i \in \Omega_{k,k'}^{+}}\epsilon_{i}\gamma_{i} = \tilde{O}(\sqrt{n}),  \label{eq:balance2} 
\end{align}
\begin{align}
\sum_{i \in \Omega_{k,k'}^{-}}\epsilon_{i}\gamma_{i}= \tilde{O}(\sqrt{n}), \sum_{i \in \Omega_{k}}\beta_{i} = \EE[\beta] \cdot n/K+ \tilde{O}(\sqrt{n}). \label{eq:balance3}    
\end{align}
\end{lemma}
\begin{proof}
Fix $k \in [K]$, by Hoeffding's inequality we have that with probability at least $1 - \delta/8K$,
\begin{align*}
\sum_{i \in \Omega_{k}}y_{i}\beta_{i}^{3} =  \sum_{i=1}^{n}y_{i}\beta_{i}^{3}\ind\big((\xb_i, y_i)  \in \Omega_{k}\big) = \tilde{O}(\sqrt{n}),
\end{align*}
where the last equality is by the fact that the expectation of $y\beta^{3} \ind\big((\xb, y) \in \Omega_{k}\big)$ is zero. Fix $k \in [K]$, by Hoeffding's inequality we have that with probability at least $1 - \delta/8K$,
\begin{align*}
\sum_{i \in \Omega_{k}}\alpha_{i}^{3} =  \sum_{i=1}^{n}\alpha_{i}^{3}\ind\big((\xb_i, y_i)  \in \Omega_{k}\big) = \frac{n\EE[\alpha^{3}]}{K} + \tilde{O}(\sqrt{n}),
\end{align*}
where the last equality is by the fact that the expectation of $\alpha^{3} \ind\big((\xb, y) \in \Omega_{k}\big)$ is $\EE[\alpha^{3}]/K$. Fix $k \in [K]$, by Hoeffding's inequality we have that with probability at least $1 - \delta/8K$,
\begin{align*}
\sum_{i \in \Omega_{k}}y_{i}\epsilon_{i}\gamma_{i}^{3} =  \sum_{i=1}^{n}y_{i}\epsilon_{i}\gamma_{i}^{3}\ind\big((\xb_i, y_i)  \in \Omega_{k}\big) = \tilde{O}(\sqrt{n}),
\end{align*}
where the last equality is by the fact that the expectation of $y\epsilon\gamma^{3}\ind\big((\xb, y) \in \Omega_{k}\big)$ is zero. Now we have proved the bounds in \eqref{eq:balance1}. We can get other bounds in \eqref{eq:balance2} and \eqref{eq:balance3} similarly. Applying union bound over $[K]$ completes the proof.
\end{proof}

\begin{lemma}\label{lm:gaussian bound}
Suppose that $d = \Omega(\log(4nP/\delta))$, with probability at least $1 - \delta$,  the following inequalities hold for all $i \in [n], k \in [K], p \geq 4$,
\begin{itemize}
\item $\|\bxi_{i,p}\|_{2} = O(1)$,
\item $\la\vb_{k}, \bxi_{i,p}\ra \leq \tilde{O}(d^{-1/2})$, $\la\cb_{k}, \bxi_{i,p}\ra \leq \tilde{O}(d^{-1/2})$, $\la\bxi_{i,p},\bxi_{i',p'}\ra \leq \tilde{O}(d^{-1/2})$, $\forall (i', p') \not= (i, p)$.
\end{itemize}
\end{lemma}

\begin{proof}[Proof of Lemma~\ref{lm:gaussian bound}] By Bernstein's inequality, with probability at least $1 - \delta / (2nP)$ we have
\begin{align*}
    \big| \| \bxi_{i,p} \|_2^2 - \sigma_{p}^{2} \big| \leq O(\sigma_{p}^{2}\sqrt{d^{-1} \log(4n P/ \delta)}).
\end{align*}
Therefore, as long as $d = \Omega( \log(4nP / \delta) )$, we have $\| \bxi_{i,p} \|_2^2 \leq 2$. Moreover, clearly $\la \bxi_{i,p}, \bxi_{i',p'} \ra$ has mean zero, $\forall (i,p)\not= (i', p')$. Then by Bernstein's inequality, with probability at least $1 - \delta / (6n^2P^{2})$ we have
\begin{align*}
    | \la \bxi_{i,p}, \bxi_{i',p'} \ra| \leq 2\sigma_{p}^{2} \sqrt{d^{-1} \log(12n^2 P^{2} / \delta)}.
\end{align*}
Similarly, $\la \vb_{k}, \bxi_{i, p}\ra$ and $\la \cb_{k}, \bxi_{i, p}\ra$ have mean zero. Then by Bernstein's inequality, with probability at least $1 - \delta/(3nPK)$ we have 
\begin{align*}
|\la \bxi_{i,p}, \vb_{k}\ra| \leq  2\sigma_{p} \sqrt{d^{-1} \log(6nPK / \delta)}, |\la \bxi_{i,p}, \cb_{k}\ra| \leq  2\sigma_{p} \sqrt{d^{-1} \log(6nPK / \delta)}. 
\end{align*}
Applying a union bound completes the proof.
\end{proof}

\noindent\textbf{MoE Initialization Property.} 

We divide the experts into $K$ sets based on the initialization.
\begin{definition}\label{Def: expert set}
Fix expert $m \in [M]$, denote $(k_{m}^{*}, j_{m}^{*}) = \argmax_{j,k}\la \vb_{k}, \wb_{m,j}^{(0)}
\ra$. Fix cluster $k \in [K]$, denote the profession experts set as $\cM_{k} = \{m| k_{m}^{*} = k\}$. 
\end{definition}

\begin{lemma}\label{lm: Mkset}
For $M \geq \Theta(K\log(K/\delta))$, $J \geq \Theta(\log(M/\delta))$, the following inequalities hold with probability at least $1 - \delta$. 
\begin{itemize}
    \item $\max_{(j,k)\not= (j_{m}^{*},k_{m}^{*})}\la  \wb_{m,j}^{(0)}, \vb_{k}\ra \leq \big(1 -  \delta/\big(3MJ^{2}K^{2})\big)\la  \wb_{m,j_{m}^{*}}^{(0)}, \vb_{k_{m}^{*}}\ra$ for all $m \in [M]$
    \item $\la  \wb_{m,j_{m}^{*}}^{(0)}, \vb_{k_{m}^{*}}\ra \geq 0.01\sigma_{0}$ for all $m \in [M]$.
    \item $|\cM_{k}| \geq 1$ for all $k 
    \in [K]$.
\end{itemize}
\end{lemma}
\begin{proof}
Recall that $\wb_{m,j} \sim \cN(0,  \sigma_{0}^{2}I_{d})$. Notice that signals  $\vb_{1},\ldots, \vb_{K}$ are orthogonal. Given fixed $m \in [M]$, we have that $\{\la  \wb_{m,j}^{(0)}, \vb_{k}\ra| j \in [J], k \in [K]\}$ are independent and individually draw from $\cN(0,\sigma_{0}^{2})$ we have that 
\begin{align*}
\mathbb{P}(\la  \wb_{m,j}^{(0)}, \vb_{k}\ra < 0.01 \sigma_{0}) < 0.9.    
\end{align*}
Therefore, we have that 
\begin{align*}
\mathbb{P}(\max_{j,k}\la  \wb_{m,j}^{(0)}, \vb_{k}\ra < 0.01 \sigma_{0}) < 0.9^{KJ}.     
\end{align*}
Therefore, as long as $J \geq \Theta(K^{-1}\log(M/\delta))$, fix $m \in [M]$ we can guarantee that with probability at least $1 - \delta/(3M)$, 
\begin{align*}
\max_{j,k}\la  \wb_{m,j}^{(0)}, \vb_{k}\ra > 0.01 \sigma_{0}.
\end{align*}
Take $G = \delta/(3MJ^{2}K^{2})$, by Lemma~\ref{lm:GaussianTop1} we have that with probability at least $1-\delta/(3M)$, 
\begin{align*}
\max_{(j,k) \not = (j_{m}^{*}, k_{m}^{*})}\la  \wb_{m,j}^{(0)}, \vb_{k}\ra \leq (1-G)\la  \wb_{m,j_{m}^{*}}^{(0)}, \vb_{k_{m}^{*}}\ra.    
\end{align*}
By the symmetric property, we have that for all $k\in[K], m \in [M]$,
\begin{align*}
\mathbb{P}(k = k_{m}^{*}) = K^{-1}.
\end{align*}
Therefore, the probability that $|\cM_{k}|$ at least include one element is as follows,
\begin{align*}
\mathbb{P}(|\cM_k| \geq 1) \geq 1 - (1 - K^{-1})^{M}.     
\end{align*}
By union bound we get that 
\begin{align*}
\mathbb{P}(|\cM_k| \geq 1, \forall k) \geq 1 - K(1 - K^{-1})^{M} \geq 1 - K\exp(-M/K) \geq 1 - \delta/3,
\end{align*}
where the last inequality is by condition $M \geq K\log(3K/\delta)$. Therefore, with probability at least $1 - \delta/3$, $|\cM_k| \geq 1, \forall k$.

Applying Union bound, we have that with probability at least $1 - \delta$,
\begin{align*}
\max_{(j,k)\not= (j_{m}^{*},k_{m}^{*})}\la  \wb_{m,j}^{(0)}, \vb_{k}\ra &\leq \big(1 -  \delta/\big(3MJ^{2}K^{2})\big)\la  \wb_{m,j_{m}^{*}}^{(0)}, \vb_{k_{m}^{*}}\ra, \\
\la  \wb_{m,j_{m}^{*}}^{(0)}, \vb_{k_{m}^{*}}\ra &\geq 0.01\sigma_{0}, \forall m \in [M],\\
|\cM_{k}| &\geq 1, \forall k \in [K].
\end{align*}
\end{proof}

\begin{lemma}\label{lm: initialvinner}
Suppose the conclusions in Lemma~\ref{lm:gaussian bound} hold, then with probability at least $1-\delta$ we have that $|\la \wb_{m,j}^{(0)}, \vb \ra| \leq \tilde{O}(\sigma_{0})$ for all $\vb \in \{\vb_{k}\}_{k\in[K]}\cup \{\cb_{k}\}_{k\in[K]} \cup \{\bxi_{i,p}\}_{i\in [n], p \in [P-3]}, m \in [M], j \in [J]$.
\end{lemma}
\begin{proof}
Fix $\vb \in \{\vb_{k}\}_{k\in[K]}\cup \{\cb_{k}\}_{k\in[K]} \cup \{\bxi_{i,p}\}_{i\in [n], p \in [P-3]}, m \in [M], j \in [J]$, we have that $\la \wb_{m,j}^{(0)}, \vb \ra \sim \cN(0, \sigma_{0}^{2}\|\vb\|_{2}^{2})$ and $\|\vb\|_{2} = O(1)$. Therefore, with probability at least $1 - \delta/(nPMJ)$ we have that $|\la \wb_{m,j}^{(0)}, \vb \ra| \leq \tilde{O}(\sigma_{0})$. Applying union bound completes the proof.
\end{proof}

\section{Proof of Theorem~\ref{thm: MoE}}\label{appendix:main theory}
In this section we always assume that the conditions in Theorem~\ref{thm: MoE} holds. It is easy to show that all the conclusions in this section~\ref{sec: initial} hold with probability at least $1 - O(1/\log d)$. The results in this section hold when all the conclusions in Section~\ref{sec: initial} hold. For simplicity of notation, we simplify $(\xb_{i}, y_{i}) \in \Omega_{k,k'}$ as $i \in \Omega_{k,k'}$, and $\ell'(y_{i}\pi_{m_{i,t}}(\xb_{i};\bTheta^{(t)})f_{m_{i,t}}(\xb_{i};  \Wb^{(t)}))$ as $\ell'_{i,t}$.

Recall that at iteration $t$, data $\xb_{i}$ is routed to the expert $m_{i,t}$. Here $m_{i,t}$ should be interpreted as a random variable. The gradient of MoE model at iteration $t$ can thus be computed as follows
\begin{align}
\nabla_{\btheta_{m}} \cL^{(t)}
    & =
    \frac{1}{n} \sum_{i,p}\ind(m_{i,t} =m)\ell'_{i,t}
    \pi_{m_{i,t}}(\xb_i; \bTheta^{(t)})(1-\pi_{m_{i,t}}(\xb_i; \bTheta^{(t)}))y_{i}f_{m_{i,t}}(\xb_{i};  \Wb^{(t)}) \xb_{i}^{(p)}\notag\\
    &\qquad - \frac{1}{n} \sum_{i,p}\ind(m_{i,t} \not=m)\ell'_{i,t}\pi_{m_{i,t}}(\xb_i; \bTheta^{(t)})\pi_{m}(\xb_i; \bTheta^{(t)})y_{i}f_{m_{i,t}}(\xb_{i};  \Wb^{(t)}) \xb_{i}^{(p)}\notag\\
    &= \frac{1}{n} \sum_{i,p}\ind(m_{i,t} =m)\ell'_{i,t}
    \pi_{m_{i,t}}(\xb_i; \bTheta^{(t)})y_{i}f_{m_{i,t}}(\xb_{i};  \Wb^{(t)}) \xb_{i}^{(p)}\notag\\
    &\qquad - \frac{1}{n} \sum_{i,p}\ell'_{i,t}
    \pi_{m_{i,t}}(\xb_i; \bTheta^{(t)})\pi_{m}(\xb_i; \bTheta^{(t)})y_{i}f_{m_{i,t}}(\xb_{i};  \Wb^{(t)}) \xb_{i}^{(p)},\label{gradient-router}\\
   \nabla_{\wb_{m,j}} \cL^{(t)}
    & =
    \frac{1}{n} \sum_{i,p}\ind(m_{i,t} =m)\ell'_{i,t}
    \pi_{m}(\xb_i; \bTheta^{(t)})y_{i}\sigma'(\la \wb_{m,j}^{(t)}, \xb_{i}^{(p)}\ra) \xb_{i}^{(p)}\label{gradient-expert}.
\end{align}

Following lemma shows implicit regularity in the gating network training.
\begin{lemma}\label{lm: zero mean}
For all $t \geq 0$, we have that $\sum_{m=1}^{M}\nabla_{\btheta_{m}}\cL^{(t)} = \zero$ and thus $\sum_{m}\btheta_{m}^{(t)} = \sum_{m}\btheta_{m}^{(0)}$. In particular, when $\bTheta$ is zero initialized, then $\sum_{m}\btheta_{m}^{(t)} = 0$
\end{lemma}
\begin{proof}
We first write out the gradient of $\btheta_{m}$ for all $m \in [M]$, 
\begin{align*}
\nabla_{\btheta_{m}} \cL^{(t)}
    & =
    \frac{1}{n} \sum_{i\in [n],p \in [P]}\ind(m_{i,t} =m)\ell'_{i,t}
    \pi_{m_{i,t}}(\xb_i; \bTheta^{(t)})y_{i}f_{m_{i,t}}(\xb_{i};  \Wb^{(t)}) \xb_{i}^{(p)}\\
    &\qquad - \frac{1}{n} \sum_{i \in [n],p \in [P]}\ell'_{i,t}
    \pi_{m_{i,t}}(\xb_i; \bTheta^{(t)})\pi_{m}(\xb_i; \bTheta^{(t)})y_{i}f_{m_{i,t}}(\xb_{i}; \Wb^{(t)}) \xb_{i}^{(p)}.
\end{align*}
Take summation from $m = 1$ to $m = M$, then we have
\begin{align*}
\sum_{m=1}^{M}\nabla_{\btheta_{m}} \cL^{(t)}&=
    \frac{1}{n}\sum_{i\in [n],p \in [P]}\ell'_{i,t}
    \pi_{m_{i,t}}(\xb_i; \bTheta^{(t)})y_{i}f_{m_{i,t}}(\xb_{i};  \Wb^{(t)}) \xb_{i}^{(p)}\\
    &\qquad - \frac{1}{n} \sum_{i\in [n],p \in [P]}\ell'_{i,t}
    \pi_{m_{i,t}}(\xb_i; \bTheta^{(t)})y_{i}f_{m_{i,t}}(\xb_{i}, \Wb^{(t)}) \xb_{i}^{(p)}\\
    &= 0.
\end{align*}
\end{proof}

Notice that the gradient at iteration $t$ in \eqref{gradient-router} and  \eqref{gradient-expert} is depend on the random variable $m_{i,t}$, the following lemma shows that it can be approximated by its expectation.
\begin{lemma}\label{lm:easy}
With probability at least $1 - 1/d$, for all the vector $\vb \in \{\vb_{k}\}_{k \in [K]} \cup \{\cb_{k}\}_{k \in [K]}$, $m \in [M]$, $j \in [J]$, we have the following equations hold
$|\la \nabla_{\btheta_{m}} \cL^{(t)} , \vb\ra - \EE[\la \nabla_{\btheta_{m}} \cL^{(t)} , \vb\ra]| = \tilde{O}(n^{-1/2}(\sigma_{0}+\eta t)^{3})$, $|\la\nabla_{\wb_{m,j}} \cL^{(t)},\vb\ra - \EE[\la\nabla_{\wb_{m,j}} \cL^{(t)},\vb\ra]| = \tilde{O}(n^{-1/2}(\sigma_{0}+\eta t)^{2})$, for all $t \leq d^{100}$. Here $\EE[\la\nabla_{\wb_{m,j}} \cL^{(t)},\vb\ra]$ and $\EE[\la \nabla_{\btheta_{m}} \cL^{(t)} , \vb\ra]$ can be computed as follows,

\begin{align*}
\EE[\la \nabla_{\btheta_{m}} \cL^{(t)} , \vb\ra] &= \frac{1}{n} \sum_{i,p}\mathbb{P}(m_{i,t} = m)\ell'_{i,t}
    \pi_{m}(\xb_i; \bTheta^{(t)})y_{i}f_{m}(\xb_{i};  \Wb^{(t)}) \la\xb_{i}^{(p)}, \vb \ra\\
    &\qquad - \frac{1}{n} \sum_{i,p,m'}\mathbb{P}(m_{i,t} = m')\ell'_{i,t}
    \pi_{m'}(\xb_i; \bTheta^{(t)})\pi_{m}(\xb_i; \bTheta^{(t)})y_{i}f_{m'}(\xb_{i};  \Wb^{(t)}) \la\xb_{i}^{(p)}, \vb\ra\\
   \EE[\la\nabla_{\wb_{m,j}} \cL^{(t)},\vb\ra] &=  \frac{1}{n} \sum_{i,p}\mathbb{P}(m_{i,t} =m)\ell'_{i,t}
    \pi_{m}(\xb_i; \bTheta^{(t)})y_{i}\sigma'(\la \wb_{m,j}^{(t)}, \xb_{i}^{(p)}\ra) \la\xb_{i}^{(p)}, \vb\ra.
\end{align*}
\end{lemma}
\begin{proof}
Because we are using normalized gradient descent, $\|\wb_{m,j}^{(t)} - \wb_{m,j}^{(0)}\|_{2} \leq O(\eta t)$ and thus by Lemma~\ref{lm: initialvinner} we have $|\la \wb_{m,j}^{(t)}, \xb_{i}^{(p)}\ra| \leq \tilde{O}(\sigma_{0} + \eta t)$. Therefore, 
\begin{align*}
\la\nabla_{\wb_{m,j}} \cL^{(t)}, \vb \ra
    & =
    \frac{1}{n} \sum_{i}\underbrace{\sum_{p}\ind(m_{i,t} =m)\ell'_{i,t}
    \pi_{m}(\xb_i; \bTheta^{(t)})y_{i}\sigma'(\la \wb_{m,j}^{(t)}, \xb_{i}^{(p)}\ra) \la\xb_{i}^{(p)}, \vb\ra}_{A_{i}},    
\end{align*}
where $A_{i}$ are independent random variables with $|A_{i}| \leq \tilde{O}\big((\sigma_{0} + \eta t)^{2}\big)$. Applying Hoeffding's inequality gives that with probability at least $1 - 1/(4d^{101}MJK)$ we have that $|\la\nabla_{\wb_{m,j}} \cL^{(t)},\vb\ra - \EE[\la\nabla_{\wb_{m,j}} \cL^{(t)},\vb\ra]| = \tilde{O}(n^{-1/2}(\sigma_{0}+\eta t)^{2})$.  Applying union bound gives that with probability at least $1 - 1/(2d)$,  $|\la\nabla_{\wb_{m,j}} \cL^{(t)},\vb\ra - \EE[\la\nabla_{\wb_{m,j}} \cL^{(t)},\vb\ra]| = \tilde{O}(n^{-1/2}(\sigma_{0}+\eta t)^{2}), \forall m\in[M], j\in[J], t\leq d^{100}$. Similarly, we can prove $|\la \nabla_{\btheta_{m}} \cL^{(t)} , \vb\ra - \EE[\la \nabla_{\btheta_{m}} \cL^{(t)} , \vb\ra]| = \tilde{O}(n^{-1/2}(\sigma_{0}+\eta t)^{3})$.

\end{proof}


\subsection{Exploration Stage}

Denote $T_{1} = \lfloor\eta^{-1}\sigma_{0}^{0.5}\rfloor$.
The first stage ends when $t = T_{1}$. During the first stage training, we can prove that the neural network parameter maintains the following property.
\begin{lemma}\label{lm:stage1fbound}
For all $t\leq T_{1}$, we have the following properties hold,
\begin{itemize}
\item $
\la \wb_{m,j}^{(t)}, \vb_{k}\ra = O(\sigma_{0}^{0.5}),\la \wb_{m,j}^{(t)}, \cb_{k}\ra = O(\sigma_{0}^{0.5}), \la \wb_{m,j}^{(t)}, \bxi_{i,p}\ra = \tilde{O}(\sigma_{0}^{0.5})$,
\item $f_{m}(\xb_{i}; \Wb^{(t)}) = \tilde{O}(\sigma_{0}^{1.5})$,
\item $|\ell'_{i,t} - 1/2 | \leq \tilde{O}(\sigma_{0}^{1.5})$,
\item $\|\btheta^{(t)}_{m}\|_{2} \leq \tilde{O}(\sigma_{0}^{1.5})$, 
\item $\|\hb(\xb_i; \bTheta^{(t)})\|_{\infty} = \tilde{O}(\sigma_{0}^{1.5})$, $\pi_{m}(\xb_i; \bTheta^{(t)}) = M^{-1} + \tilde{O}(\sigma_{0}^{1.5})$,
\end{itemize}
for all $m \in [M], k\in[k], i \in [n], p \geq 4$.
\end{lemma}
\begin{proof}
The first property is obvious since $\|\wb_{m,j}^{(t)} - \wb_{m,j}^{(0)}\|_{2} \leq O(\eta T_{1}) = O(\sigma_{0}^{0.5})$ and thus
\begin{align*}
|f_{m}(\xb_{i}; \Wb^{(t)})| \leq \sum_{p \in [P]}\sum_{j \in [J]}|\sigma(\la \wb_{m,j}^{(t)}, \xb_{i}^{(p)}\ra)| = \tilde{O}(\sigma_{0}^{1.5}).    
\end{align*}
Then we show that the loss derivative is close to $1/2$ during this stage.

Let $s = y_{i}\pi_{m_{i,t}}(\xb_{i};\bTheta^{(t)})f_{m_{i,t}}(\xb_{i},\Wb^{(t)})$, then we have that $|s| = \tilde{O}(\sigma_{0}^{1.5})$ and 
\begin{align*}
\bigg|\ell'_{i,t} - \frac{1}{2}\bigg| = \bigg|\frac{1}{e^{s}+1} - 1/2 \bigg| \overset{(i)}{\leq} |s| = \tilde{O}(\sigma_{0}^{1.5}),
\end{align*}
where $(i)$ can be proved by considering $|s|\leq 1$ and $|s| > 1$.

Now we prove the fourth bullet in Lemma~\ref{lm:stage1fbound}. Because $|f_{m}| = \tilde{O}(\sigma_{0}^{1.5})$, we can upper bound the gradient of the gating network by 

\begin{align*}
\|\nabla_{\btheta_{m}} \cL^{(t)}\|_{2}
    & =
    \bigg\|\frac{1}{n} \sum_{i,p}\ind(m_{i,t} =m)\ell'_{i,t}
    \pi_{m_{i,t}}(\xb_i; \bTheta^{(t)})y_{i}f_{m_{i,t}}(\xb_{i};  \Wb^{(t)}) \xb_{i}^{(p)}\\
    &\qquad - \frac{1}{n} \sum_{i,p}\ell'_{i,t}
    \pi_{m_{i,t}}(\xb_i; \bTheta^{(t)})\pi_{m}(\xb_i; \bTheta^{(t)})y_{i}f_{m_{i,t}}(\xb_{i};  \Wb^{(t)}) \xb_{i}^{(p)}\bigg\|_{2}.\\
    &= \tilde{O}(\sigma_{0}^{1.5}),
\end{align*}
where the last inequality is due to $|\ell_{i,t}'| \leq 1$, $\pi_{m}, \pi_{m_{i,t}} \in [0,1]$ and $\|\xb_{i}^{(p)}\|_{2} = O(1)$.
This further implies that 
\begin{align*}
\|\btheta_{m}^{(t)}\|_{2} = \|\btheta_{m}^{(t)} - \btheta_{m}^{(0)}\|_{2}
\leq  \tilde{O}(\sigma_{0}^{1.5} t \eta_{r}) = \tilde{O}(\sigma_{0}^{1.5}),
\end{align*}
where the last inequality is by $\eta_{r} = \Theta(M^{2})\eta$. The proof of $\|\hb(\xb_i; \bTheta^{(t)})\|_{\infty} \leq O(\sigma_{0}^{1.5})$ and  $\pi_{m}(\xb_i; \bTheta^{(t)}) = M^{-1} + O(\sigma_{0}^{1.5})$ are straight forward given $\|\btheta_{m}^{(t)}\|_{2} = \tilde{O}(\sigma_{0}^{1.5})$.
\end{proof}

We will first investigate the property of the router.
\begin{lemma}\label{eq:concentration}
$\max_{m \in [M]}|\mathbb{P}(m_{i,t} = m) - 1/M| = \tilde{O}(\sigma_{0}^{1.5})$ for all $t \leq T_{1}$, $i \in [n]$ and $m \in [M]$.
\end{lemma}
\begin{proof}
By Lemma~\ref{lm:stage1fbound} we have that $\|\hb(\xb_{i}; \bTheta^{(t)})\|_{\infty} \leq  \tilde{O}(\sigma_{0}^{1.5})$. Lemma~\ref{lm:Msmoothly} further implies that 
\begin{align*}
\max_{m \in [M]}|\mathbb{P}(m_{i,t} = m) - 1/M| = \tilde{O}(\sigma_{0}^{1.5}).   
\end{align*}

\end{proof}

 


\begin{lemma}\label{lm:gradientinner}
We have following gradient update rules hold for the experts, 
\begin{align*}
 \la \nabla_{\wb_{m,j}} \cL^{(t)}, \vb_{k} \ra
    &= -\frac{\EE[\alpha^{3}] + \tilde{O}(d^{-0.005})}{2KM^{2}}
   \sigma'(\la \wb_{m,j}^{(t)}, \vb_{k}\ra) + \tilde{O}(\sigma_{0}^{2.5}), \\
   \la \nabla_{\wb_{m,j}} \cL^{(t)}, \cb_{k} \ra
    &= \tilde{O}(d^{-0.005})\sigma'(\la \wb_{m,j}^{(t)}, \cb_{k}\ra) + \tilde{O}(\sigma_{0}^{2.5}),\\
 \la \nabla_{\wb_{m,j}} \cL^{(t)}, \bxi_{i,p} \ra
    &=  \tilde{O}(d^{-0.005})\sigma'(\la \wb_{m,j}^{(t)}, \bxi_{i,p}\ra) + \tilde{O}(\sigma_{0}^{2.5})    
\end{align*}
for all $t \leq T_{1}, j \in [J], k \in [K], m\in [M], p \geq 4$.
Besides, we have the following gradient norm upper bound holds 
\begin{align*}
\|\nabla_{\wb_{m,j}} \cL^{(t)}\|_{2} &\leq \sum_{k \in [K]}  \frac{\EE[\alpha^{3}] + \tilde{O}(d^{-0.005})}{2KM^{2}}
   \sigma'(\la \wb_{m,j}^{(t)}, \vb_{k}\ra) + \sum_{k \in [K]}\tilde{O}(d^{-0.005})\sigma'(\la \wb_{m,j}^{(t)}, \cb_{k}\ra)\\
   &\qquad + \sum_{i \in [n],p \geq 4}\tilde{O}(d^{-0.005})\sigma'(\la \wb_{m,j}^{(t)}, \bxi_{i,p}\ra) + \tilde{O}(\sigma_{0}^{2.5})
\end{align*}
for all $t \leq T_{1}, j \in [J], m\in [M]$.
\end{lemma}

\begin{proof}
The experts gradient can be computed as follows,

\begin{align*}
   \nabla_{\wb_{m,j}} \cL^{(t)}
    & =
    \frac{1}{n} \sum_{i \in [n],p \in [P]}\ind(m_{i,t} =m)\ell'_{i,t}f_{m}(\xb_{i};  \Wb^{(t)})
    \pi_{m}(\xb_i; \bTheta^{(t)})y_{i}\sigma'(\la \wb_{m,j}^{(t)}, \xb_{i}^{(p)}\ra) \xb_{i}^{(p)}.
\end{align*}
We first compute the inner product $ \la \nabla_{\wb_{m,j}} \cL^{(t)}, \cb_{k} \ra$. By Lemma~\ref{lm:easy}, we have that $|\la \nabla_{\wb_{m,j}} \cL^{(t)}, \cb_{k} \ra - \EE[\la \nabla_{\wb_{m,j}} \cL^{(t)}, \cb_{k}\ra]| = \tilde{O}(n^{-1/2}\sigma_{0}) \leq  \tilde{O}(\sigma_{0}^{2.5})$.
\begin{align*}
\EE[\la \nabla_{\wb_{m,j}} \cL^{(t)}, \cb_{k}\ra] 
    &= -\frac{1}{n}\sum_{i\in \Omega_{k}}\mathbb{P}(m_{i,t}=m)\ell'_{i,t}\pi_{m}(\xb_i; \bTheta^{(t)})\sigma'(\la \wb_{m,j}^{(t)}, \cb_{k}\ra)y_{i}\beta_{i}^{3}\|\cb_{k}\|_{2}^{2}\\
    &\qquad-\frac{1}{n}\sum_{i \in [n], p\geq 4}\mathbb{P}(m_{i,t}=m)\ell'_{i,t}\pi_{m}(\xb_i; \bTheta^{(t)})\sigma'(\la \wb_{m,j}^{(t)}, \bxi_{i,p}\ra)y_{i}\la \cb_{k}, \bxi_{i,p}\ra \\
    &=  \bigg [-\frac{1}{2nM}\sum_{i \in \Omega_{k}}y_{i}\beta_{i}^{3}\mathbb{P}(m_{i,t}=m) + \tilde{O}(\sigma_{0}^{1.5})\bigg] \sigma'(\la \wb_{m,j}^{(t)}, \cb_{k}\ra) +  \tilde{O}(\sigma_{0}^{2.5})\\
    &= \tilde{O}(n^{-1/2} + \sigma_{0}^{1.5}) \sigma'(\la \wb_{m,j}^{(t)}, \cb_{k}\ra) +  \tilde{O}(\sigma_{0}^{2.5})\\
    &= \tilde{O}(d^{-0.005})\sigma'(\la \wb_{m,j}^{(t)}, \cb_{k}\ra) + \tilde{O}(\sigma_{0}^{2.5})
\end{align*}
where the second equality is due to Lemma~\ref{lm:stage1fbound} and \ref{lm:gaussian bound}, the third equality is due to Lemma~\ref{eq:concentration}, the last equality is by the choice of $n$ and $\sigma_{0}$. Next we compute the inner product $\la \nabla_{\wb_{m,j}}\cL, \vb_{k}\ra$. By Lemma~\ref{lm:easy}, we have that $|\la \nabla_{\wb_{m,j}} \cL^{(t)}, \vb_{k} \ra - \EE[\la \nabla_{\wb_{m,j}} \cL^{(t)}, \vb_{k}\ra]| = \tilde{O}(n^{-1/2}\sigma_{0}) \leq  \tilde{O}(\sigma_{0}^{2.5})$.
\begin{align*}
 \EE[\la \nabla_{\wb_{m,j}} \cL^{(t)}, \vb_{k} \ra]
    &= -\frac{1}{n}\sum_{i\in \Omega_{k}}\mathbb{P}(m_{i,t}=m)\ell'_{i,t}\pi_{m}(\xb_i; \bTheta^{(t)})\sigma'(\la \wb_{m,j}^{(t)}, \vb_{k}\ra)\alpha_{i}^{3}\|\vb_{k}\|_{2}^{2}\\
    &\qquad -\frac{1}{n}\sum_{k'\not= k}\sum_{i\in \Omega_{k',k}}\mathbb{P}(m_{i,t}=m)\ell'_{i,t}
    \pi_{m}(\xb_{i}; \btheta^{(t)})\sigma'(\la \wb_{m,j}^{(t)}, \vb_{k}\ra)\gamma_{i}^{3}y_{i}\epsilon_{i}\|\vb_{k}\|_{2}^{2}\\   
    &\qquad -\frac{1}{n}\sum_{i\in[n],p\geq 4}\mathbb{P}(m_{i,t}=m)\ell'_{i,t}\pi_{m}(\xb_i; \bTheta^{(t)})\sigma'(\la \wb_{m,j}^{(t)}, \bxi_{i,p}\ra)y_{i}\la \vb_{k}, \bxi_{i,p}\ra\\
    &= \bigg[-\frac{1}{2nM}\sum_{i\in \Omega_{k}}\mathbb{P}(m_{i,t} = m)\alpha_{i}^{3} - \frac{1}{2nM}\sum_{i\in \Omega_{k',k}}\mathbb{P}(m_{i,t} = m)\gamma_{i}^{3}y_{i}\epsilon_{i} + O(\sigma_{0}^{1.5})\bigg]\cdot\\
    &\qquad \sigma'(\la \wb_{m,j}^{(t)}, \cb_{k}\ra) + \tilde{O}(\sigma_{0}^{2.5})\\
    &= \big(\EE[\alpha^{3}] + \tilde{O}(n^{-1/2} + \sigma_{0}^{1.5})\big)\sigma'(\la \wb_{m,j}^{(t)}, \vb_{k}\ra)+ \tilde{O}(\sigma_{0}^{2.5})\\
    &=  \bigg(\frac{\EE[\alpha^{3}]}{2KM^{2}} + \tilde{O}(d^{-0.005})\bigg)\sigma'(\la \wb_{m,j}^{(t)}, \vb_{k}\ra) + \tilde{O}(\sigma_{0}^{2.5})
\end{align*}
where the second equality is due to Lemma~\ref{lm:stage1fbound} and \ref{lm:gaussian bound}, the third equality is due to Lemma~\ref{eq:concentration}, the last equality is by the choice of $n$ and $\sigma_{0}$. Finally we compute the inner product $\la \nabla_{\wb_{m,j}}\cL, \bxi_{i,p} \ra$ as follows

\begin{align*}
 \la \nabla_{\wb_{m,j}} \cL^{(t)}, \bxi_{i,p} \ra
    &= -\frac{1}{n}\ind(m_{i,t}=m)\ell'_{i,t}
    \pi_{m}(\xb_{i}; \btheta^{(t)})\sigma'(\la\wb_{m,j}^{(t)}, \bxi_{i,p}\ra)\|\bxi_{i,p}\|_{2}^{2} + \tilde{O}(\sigma_{0}d^{-1/2})\\
    &= \tilde{O}\bigg(\frac{\|\bxi_{i,p}\|_{2}^{2}}{n}\bigg)\sigma'(\la\wb_{m,j}^{(t)}, \bxi_{i,p}\ra) + \tilde{O}(\sigma_{0}d^{-1/2})\\
    &= \tilde{O}(d^{-0.005})\sigma'(\la\wb_{m,j}^{(t)}, \bxi_{i,p}\ra) + \tilde{O}(\sigma_{0}^{2.5}),
\end{align*}
where the first equality is due to Lemma~\ref{lm:gaussian bound}, second equality is due to $|\ell_{i,t}'| \leq 1, \pi_{m} \in [0,1]$ and the third equality is due to Lemma~\ref{lm:gaussian bound} and our choice of $n, \sigma_{0}$. Based on previous results, let $B$ be the projection matrix on the linear space spanned by $\{\vb_{k}\}_{k\in[K]}\cup \{\cb_{k}\}_{k\in[K]}$. We can verify that
\begin{align*}
\|\nabla_{\wb_{m,j}} \cL^{(t)}\|_{2} &\leq \|B\nabla_{\wb_{m,j}} \cL^{(t)}\|_{2} + \|(I - B)\nabla_{\wb_{m,j}} \cL^{(t)}\|_{2}\\
    &\leq \sum_{k\in [K]}  \frac{\EE[\alpha^{3}] + \tilde{O}(d^{-0.005})}{2KM^{2}}
   \sigma'(\la \wb_{m,j}^{(t)}, \vb_{k}\ra) + \sum_{k \in [K]}\tilde{O}(d^{-0.005})\sigma'(\la \wb_{m,j}^{(t)}, \cb_{k}\ra)\\
   &\qquad + \sum_{i \in [n],p\geq 4}\tilde{O}(d^{-0.005})\sigma'(\la \wb_{m,j}^{(t)}, \bxi_{i,p}\ra) + \tilde{O}(\sigma_{0}^{2.5}).
\end{align*}
\end{proof}

 Because we use normalized gradient descent, all the experts get trained at the same speed. Following lemma shows that expert $m$ will focus on the signal $\vb_{k_m^{*}}$.

\begin{lemma}\label{lm:GG}
For all $m \in [M]$ and $t \leq T_{1}$, we have following inequalities hold,
\begin{align*}
\la \wb_{m,j_{m}^{*}}^{(t)}, \vb_{k^{*}_{m}} \ra &= O(\sigma_{0}^{0.5}),\\
\la \wb_{m,j}^{(t)}, \vb_{k}\ra &= \tilde{O}(\sigma_{0}), \forall (j, k)\not= (j_{m}^{*}, k^{*}_{m}),\\
\la \wb_{m,j}^{(t)}, \cb_{k}\ra &= \tilde{O}(\sigma_{0}), \forall j \in [J], k \in [K],\\
\la \wb_{m,j}^{(t)}, \bxi_{i,p} \ra &= \tilde{O}(\sigma_{0}), \forall j \in [J], i\in [n], p \geq 4.
\end{align*}
\end{lemma}
\begin{proof}
For $t\leq T_{1}$, the update rule of every expert could be written as,
\begin{align}
\la \wb_{m,j}^{(t+1)}, \vb_{k}\ra  &= \la \wb_{m,j}^{(t)}, \vb_{k}\ra + \frac{\eta}{\| \nabla_{\Wb_{m}} \cL^{(t)}\|_{F}}\bigg[\frac{3\EE[\alpha^{3}] + \tilde{O}(d^{-0.005})}{2KM^{2}}\la \wb_{m,j}^{(t)}, \vb_{k}\ra^{2}+ \tilde{O}(\sigma_{0}^{2.5})\bigg], \notag\\
\la \wb_{m,j}^{(t+1)}, \bxi_{i,p} \ra  &= \la \wb_{m,j}^{(t)}, \bxi_{i,p} \ra + \frac{\eta}{\| \nabla_{\Wb_{m}} \cL^{(t)}\|_{F}}\big[ \tilde{O}(d^{-0.005})\la \wb_{m,j}^{(t)}, \bxi_{i,p} \ra^{2} + \tilde{O}(\sigma_{0}^{2.5})\big], \notag\\ 
\la \wb_{m,j}^{(t+1)}, \cb_{k}\ra  &=  \la \wb_{m,j}^{(t)}, \cb_{k}\ra + \frac{\eta}{\| \nabla_{\Wb_{m}} \cL^{(t)}\|_{F}}\big[\tilde{O}(d^{-0.005})\la \wb_{m,j}^{(t)}, \cb_{k}\ra^{2} + \tilde{O}(\sigma_{0}^{2.5})\big]. \label{eq:Update}
\end{align}
For $t \leq T_{1}$, we have that $\la\wb_{m,j}^{(t)}, \vb_{k^{*}_{m}}\ra \leq O(\sigma_{0}^{0.5})$. By comparing the update rule of $\la\wb_{m,j}^{(t)}, \vb_{k^{*}_{m}}\ra$ and other inner product presented in \eqref{eq:Update} , We can prove that $\la\wb_{m,j}^{(t)}, \vb_{k^{*}_{m}}\ra$ will grow to $\sigma_{0}^{0.5}$ while other inner product still remain nearly unchanged.

\noindent \textbf{Comparison with $\la \wb_{m,j}^{(t)}, \vb_{k}\ra$}. Consider $k \not= k^{*}_{m}$. We want to get an upper bound of $\la \wb_{m,j}^{(t)}, \vb_{k}\ra$, so without loss of generality we can assume $\la \wb_{m,j}^{(t)}, \vb_{k}\ra = \Omega(\sigma_{0})$. 
Since $\sigma_{0} \leq d^{-0.01}$, we have that $\la \wb_{m,j}^{(t)}, \vb_{k}\ra^{2} + \tilde{O}(\sigma_{0}^{2.5}) = (1+ \tilde{O}(d^{-0.005}))\la \wb_{m,j}^{(t)}, \vb_{k}\ra^{2}$. Therefore, we have that
\begin{align}
\la \wb_{m,j}^{(t+1)}, \vb_{k^{*}_{m}}\ra  &= \la \wb_{m,j}^{(t)}, \vb_{k^{*}_{m}}\ra + \frac{\eta}{\| \nabla_{\Wb_{m}} \cL^{(t)}\|_{F}}\frac{3\EE[\alpha^{3}] + \tilde{O}(d^{-0.005})}{2KM^{2}}\la \wb_{m,j}^{(t)}, \vb_{k^{*}_{m}}\ra^{2},\\
\la \wb_{m,j}^{(t+1)}, \vb_{k}\ra  &= \la \wb_{m,j}^{(t)}, \vb_{k}\ra + \frac{\eta}{\| \nabla_{\Wb_{m}} \cL^{(t)}\|_{F}}\frac{3\EE[\alpha^{3}] + \tilde{O}(d^{-0.005})}{2KM^{2}}\la \wb_{m,j}^{(t)}, \vb_{k}\ra^{2}.    
\end{align}
Applying Lemma~\ref{lm: Tensor power update} by choosing $C_{t} = (3\EE[\alpha^{3}] + \tilde{O}(d^{-0.005}))/(2KM^{2}\| \nabla_{\Wb_{m}} \cL^{(t)}\|_{F})$, $S = 1 + \tilde{O}(d^{-0.005})$, $G = 1/(3\log (d)M^{2})$ and verifying $\la \wb^{(0)}_{m}, \vb_{k^{*}_{m}}\ra \geq S(1+G^{-1})\la \wb^{(0)}_{m}, \vb_{k} \ra$ (events in Section~\ref{sec: initial} hold), we have that $\la \wb_{m,j}^{(t)}, \vb_{k}\ra \leq O(G^{-1}\sigma_{0}) = \tilde{O}(\sigma_{0})$.

\noindent \textbf{Comparison with $\la \wb_{m,j}^{(t)}, \cb_{k}\ra$}.We want to get an upper bound of $\la \wb_{m,j}^{(t)}, \cb_{k}\ra$, so without loss of generality we can assume $\la \wb_{m,j}^{(t)}, \vb_{k}\ra = \Omega(\sigma_{0})$. Because $\sigma_{0} \leq d^{-0.01}$, one can easily show that  
\begin{align*}
\la \wb_{m,j}^{(t+1)}, \vb_{k^{*}_{m}}\ra  &= \la \wb_{m,j}^{(t)}, \vb_{k^{*}_{m}}\ra + \frac{\eta}{\| \nabla_{\Wb_{m}} \cL^{(t)}\|_{F}}\frac{3\EE[\alpha^{3}] + \tilde{O}(d^{-0.005})}{2KM^{2}}\la \wb_{m,j}^{(t)}, \vb_{k^{*}_{m}}\ra^{2},\\
\la \wb_{m,j}^{(t+1)}, \cb_{k}\ra  &\leq  \la \wb_{m,j}^{(t)}, \cb_{k}\ra + \frac{\eta}{\| \nabla_{\Wb_{m}} \cL^{(t)}\|_{F}}\tilde{O}(d^{-0.01})\la \wb_{m,j}^{(t)}, \cb_{k}\ra^{2}.
\end{align*}
Again, applying Lemma~\ref{lm: Tensor power update} by choosing $C_{t} = (3\EE[\alpha^{3}] + \tilde{O}(d^{-0.005}))/(2KM^{2}\| \nabla_{\Wb_{m}} \cL^{(t)}\|_{F})$, $S = \tilde{O}(d^{-0.01})$, $G = 2$ and verifying $\la \wb^{(0)}_{m}, \vb_{k^{*}_{m}}\ra \geq S(1+G^{-1})\la \wb^{(0)}_{m}, \cb_{k} \ra$ (events in Section~\ref{sec: initial} hold), we have that $\la \wb^{(t)}, \vb_{k}\ra \leq O(G^{-1}\sigma_{0}) = \tilde{O}(\sigma_{0})$.

\noindent \textbf{Comparison with $\la \wb_{m,j}^{(t)}, \bxi_{i,p} \ra$}. The proof is exact the same as the one with $\cb_{k}$.
\end{proof}

Denote the iteration $T^{(m)}$ as the first time that $\|\nabla_{\Wb_{m}}\cL^{(t)}\|_{F} \geq \sigma_{0}^{1.8}$. Then Following lemma gives an upper bound of $T^{(m)}$ for all $m\in \cM$.

\begin{lemma}\label{lm:shortphase}
For all $m \in [M]$, we have that $T^{(m)} = \tilde{O}(\eta^{-1}\sigma_{0}^{0.8})$ and thus $T^{(m)} < 0.01T_{1}$. Besides, for all $T_m < t \leq T_{1}$ we have that \begin{align*}
\la \nabla_{\wb_{m,j_{m}^{*}}}\cL^{(t)}, \vb_{k^{*}_{m}}\ra \geq (1 - \sigma_{0}^{0.1})\|\nabla_{\Wb_{m}}\cL^{(t)}\|_{F}.   
\end{align*}
\end{lemma}
\begin{proof}
Let projection matrix $B = \vb_{k^{*}_{m}}\vb_{k^{*}_{m}}^{\top} \in \RR^{d\times d}$, then we can divide the gradient into two orthogonal part

\begin{align*}
\|\nabla_{\wb_{m,j_{m}^{*}}}\cL^{(t)}\|_{2} &=  \|B\nabla_{\wb_{m,j_{m}^{*}}}\cL^{(t)} + (I-B)\nabla_{\wb_{m,j_{m}^{*}}}\cL^{(t)}\|_{2} \\
&\leq \|B\nabla_{\wb_{m,j_{m}^{*}}}\cL^{(t)}\|_{2} + \|(I-B)\nabla_{\wb_{m,j_{m}^{*}}}\cL^{(t)}\|_{2}
\end{align*}
Recall that 
\begin{align*}
   \nabla_{\wb_{m,j_{m}^{*}}} \cL^{(t)}
    & =
    \frac{1}{n} \sum_{i,p}\ind(m_{i,t} =m)\ell'_{i,t}
    \pi_{m}(\xb_i; \bTheta^{(t)})y_{i}\sigma'(\la \wb_{m,j_{m}^{*}}^{(t)}, \xb_{i}^{(p)}\ra) \xb_{i}^{(p)},
\end{align*}
So we have that 
\begin{align*}
   \|(I-B)\nabla_{\wb_{m,j_{m}^{*}}} \cL^{(t)}\|_{2} & =
    \bigg\|\frac{1}{n} \sum_{i,p}\ind(m_{i,t} =m)\ell'_{i,t}
    \pi_{m}(\xb_i; \bTheta^{(t)})y_{i}\sigma'(\la \wb_{m,j_{m}^{*}}^{(t)}, \xb_{i}^{(p)}\ra) (I-B)\xb_{i}^{(p)}\bigg\|_{2}\\
    &\leq \frac{1}{n} \sum_{i,p}\bigg\|\sigma'(\la \wb_{m,j_{m}^{*}}^{(t)}, \xb_{i}^{(p)}\ra) (I-B)\xb_{i}^{(p)}\bigg\|_{2}\\
    &\leq \tilde{O}(\sigma_{0}^{2}),
\end{align*}
where the first inequality is by $|\ell_{i,t}'| \leq 1, \pi_{m} \in [0,1]$ and the second equality is because 
\begin{enumerate}
\item when $\xb_{i}^{(p)}$ align with $\vb_{k_{m}^{*}}$, $(I-B)\xb_{i}^{(p)} = \zero$.
\item when $\xb_{i}^{(p)}$ doesn't align with $\vb_{k_{m}^{*}}$, $\la \wb_{m,j_{m}^{*}}^{(t)}, \xb_{i}^{(p)}\ra = \tilde{O}(\sigma_{0})$.
\end{enumerate}
Therefore, we have that 
\begin{align*}
\|\nabla_{\wb_{m,j_{m}^{*}}}\cL^{(t)}\|_{2} \leq \|B\nabla_{\wb_{m,j_{m}^{*}}}\cL^{(t)}\|_{2} +\tilde{O}(\sigma_{0}^{2}) = \la \nabla_{\wb_{m,j_{m}^{*}}}\cL^{(t)}, \vb_{k^{*}_{m}}\ra +\tilde{O}(\sigma_{0}^{2}).
\end{align*}
We next compute the gradient of the neuron $\wb_{m,j}, j\not= j_{m}^{*}$,
\begin{align}
   \|\nabla_{\wb_{m,j}} \cL^{(t)}\|_{2}
    & =
    \bigg\|\frac{1}{n} \sum_{i,p}\ind(m_{i,t} =m)\ell'_{i,t}
    \pi_{m}(\xb_i; \bTheta^{(t)})y_{i}\sigma'(\la \wb_{m,j}^{(t)}, \xb_{i}^{(p)}\ra) \xb_{i}^{(p)}\bigg\|_{2} = \tilde{O}(\sigma_{0}^{2}), \label{eq:g1}
\end{align}
where the inequality is by
 $\la \wb_{m,j}^{(t)}, \xb_{i}^{(p)}\ra = \tilde{O}(\sigma_{0}), \forall j \not= j_{m}^{*}$ which is due to Lemma~\ref{lm:GG}.
Now we can upper bound the gradient norm,
\begin{align}
\|\nabla_{\Wb_{m}}\cL^{(t)}\|_{F} \leq \sum_{j \in [J]}\|\nabla_{\wb_{m,j}}\cL^{(t)}\|_{2} \leq \|\nabla_{\wb_{m,j_{m}^{*}}}\cL^{(t)}\|_{2} + \tilde{O}(\sigma_{0}^{2}). \label{eq:g2}
\end{align}
When $\|\nabla_{\Wb_{m}}\cL^{(t)}\|_{F} \geq \sigma_{0}^{1.8}$, it is obviously that 
\begin{align*}
 \la \nabla_{\wb_{m,j_{m}^{*}}}\cL, \vb_{k^{*}_{m}}\ra \geq \|\nabla_{\wb_{m,j_{m}^{*}}}\cL^{(t)}\|_{2}  - \tilde{O}(\sigma_{0}^{2}) \geq \|\nabla_{\Wb_{m}}\cL^{(t)}\|_{F}  - \tilde{O}(\sigma_{0}^{2})\geq (1-\sigma_{0}^{0.1})\|\nabla_{\Wb_{m}}\cL^{(t)}\|_{F},
\end{align*}
where the first inequality is by \eqref{eq:g1} and the second inequality is by \eqref{eq:g2}. Now let us give an upper bound for $T^{(m)}$. During the period $t \leq T^{(m)}$, $\|\nabla_{\Wb_{m}}\cL^{(t)}\|_{F} < \sigma_{0}^{1.8}$. On the one hand,  by Lemma~\ref{lm:gradientinner} we have that
\begin{align*}
\|\nabla_{\Wb_{m}} \cL^{(t)}\|_{2} \geq -\la  \nabla_{\wb_{m,j}} \cL^{(t)}, \vb_{k_{m}^{*}} \ra
    = \frac{3\EE[\alpha^{3}] - \tilde{O}(d^{-0.005})}{2KM^{2}}
   [\la \wb_{m,j_{m}^{*}}^{(t)}, \vb_{k_{m}^{*}}\ra]^{2} - \tilde{O}(\sigma_{0}^{2.5})    
\end{align*}
which implies that the inner product $\la \wb_{m,j_{m}^{*}}^{(t)}, \vb_{k^{*}_{m}}\ra \leq \tilde{O}(\sigma_{0}^{0.9})$. On the other hand, by Lemma~\ref{lm:GG} we have that
\begin{align*}
\la \wb_{m,j_{m}^{*}}^{(t+1)}, \vb_{k^{*}_{m}}\ra  &\geq \la \wb_{m,j_{m}^{*}}^{(t)}, \vb_{k^{*}_{m}}\ra + \frac{\eta}{\| \nabla_{\Wb_{m}} \cL^{(t)}\|_{F}}\Theta(\frac{1}{KM^{2}})\la \wb_{m,j_{m}^{*}}^{(t)}, \vb_{k^{*}_{m}}\ra^{2} \\
&\geq \la \wb_{m,j_{m}^{*}}^{(t)}, \vb_{k^{*}_{m}}\ra + \Theta\Big(\frac{\eta}{KM^{2}\sigma_{0}^{1.8}}\Big)\la \wb_{m,j_{m}^{*}}^{(t)}, \vb_{k^{*}_{m}}\ra^{2}\\
&\geq \la \wb_{m,j_{m}^{*}}^{(t)}, \vb_{k^{*}_{m}}\ra + \Theta\Big(\frac{\eta}{KM^{2}\sigma_{0}^{0.8}}\Big)\la \wb_{m,j_{m}^{*}}^{(t)}, \vb_{k^{*}_{m}}\ra,
\end{align*}
where last inequality is by $\la \wb_{m,j_{m}^{*}}^{(t)}, \vb_{k^{*}_{m}}\ra \geq 0.1 \sigma_{0} $. Therefore, we have that the inner product $\la \wb_{m,j}^{(t)}, \vb_{k^{*}_{m}}\ra$ grows exponentially and will reach $\tilde{O}(\sigma_{0}^{0.9})$ within $\tilde{O}(\eta^{-1}\sigma_{0}^{0.8})$ iterations.

\end{proof}

Recall that  $T_{1} = \lfloor \eta^{-1}\sigma_{0}^{0.5}\rfloor$, following Lemma shows that the expert $m \in [M]$ only learns one feature during the first stage,
\begin{lemma}\label{lm:shortphase2}
For all $t \leq T_{1}, m \in [M]$, we have that
\begin{align*}
\la \wb_{m,j_{m}^{*}}^{(t)}, \vb_{k_{m}^{*}}\ra &= O(\sigma_{0}^{0.5}),\\
\la \wb_{m,j}^{(t)}, \vb_{k}\ra &= \tilde{O}(\sigma_{0}), \forall (j, k) \not= (j_{m}^{*}, k_{m}^{*}),\\
\la \wb_{m,j}^{(t)}, \cb_{k}\ra &= \tilde{O}(\sigma_{0}), \forall j \in [J], k \in [K],\\
\la \wb_{m,j}^{(t)}, \bxi_{i,p} \ra &= \tilde{O}(\sigma_{0}), \forall j \in [J], i \in [n], p \geq 4.
\end{align*}
Besides $\la \wb_{m,j_{m}^{*}}^{(t)}, \vb_{k_{m}^{*}}\ra \geq  (1 - \sigma_{0}^{0.1}) \eta t $, for all $t \geq T_{1}/2$.
\end{lemma}
\begin{proof}
By Lemma~\ref{lm:shortphase}, we have $T^{(m)} = \tilde{O}(\eta^{-1}\sigma_{0}^{0.8}) < \sigma_{0}^{0.2}\cdot T_1$.  Notice that $\la \nabla_{\wb_{m,j_{m}^{*}}}\cL^{(t)}, \vb_{k^{*}}\ra \geq ( 1- \sigma_{0}^{0.1})\|\nabla_{\Wb_{m}}\cL^{(t)}\|_{F}$, for all $T_{m}\leq t \leq T_{1}$. Therefore, we have that 
\begin{align*}
\la \wb_{m,j_{m}^{*}}^{(t+1)}, \vb_{k^{*}_{m}}\ra \geq   \la \wb_{m,j_{m}^{*}}^{(t)}, \vb_{k^{*}_{m}}\ra + (1 - \sigma_{0}^{0.1})\eta, \forall  T_{m}\leq t \leq T_{1}, 
\end{align*}
which implies $\la \wb_{m,j_{m}^{*}}^{(t)}, \vb_{k^{*}_{m}}\ra \geq (1 - O(\sigma_{0}^{0.1}))\eta t, \forall t \geq T_{1}/2$. Finally, applying Lemma~\ref{lm:GG} completes the proof.

\end{proof}

\subsection{Router Learning Stage}
\label{subsection: routerlearning}

Denote $T_{2} = \lfloor \eta^{-1}M^{-2}\rfloor$, 
The second stage ends when $t = T_{2}$. Given $\xb = [\alpha y\vb_{k}, \beta\cb_{k},  \gamma \epsilon\vb_{k'}, \bxi]$, we denote by $\bar{x} = [\zero, \beta\cb_{k},  \zero, \ldots, \zero]$ the one only keeps cluster-center signal and denote by $\hat{x} = [\alpha y \vb_{k}, \zero , \gamma \epsilon \vb_{k'}, \zero]$ the one that only keeps feature signal and feature noise.

For all $T_{1} \leq t \leq T_{2}$, we will show that the router only focuses on the cluster-center signals and the experts only focus on the feature signals, i.e., we will prove that $|f_{m}(\xb_{i}; \Wb^{(t)}) - f_{m}(\hat{\xb}_{i}; \Wb^{(t)})|$ and $\|\hb(\xb_i; \bTheta^{(t)}) - \hb(\bar{\xb}_{i},\bTheta^{(t)})\|_{\infty}$ are small. In particular, We claim that for all $T_{1} \leq t \leq T_{2}$, following proposition holds.

\begin{proposition}\label{claim:main}
For all $T_{1} \leq t \leq T_{2}$, following inequalities hold, 
\begin{align}
&|f_{m}(\xb_{i}; \Wb^{(t)}) - f_{m}(\hat{\xb}_{i}; \Wb^{(t)})| \leq O(d^{-0.001}), \forall m \in [M], i \in [n] \label{eq: induction1},\\ 
&\|\hb(\xb_i; \bTheta^{(t)}) - \hb(\bar{\xb}_{i};\bTheta^{(t)})\|_{\infty} \leq O(d^{-0.001}),
\forall  i \in [n],\label{eq:induction2}\\
&\mathbb{P}(m_{i,t} = m ), \pi_{m}(\xb_{i}; \bTheta^{(t)}) = \Omega(1/M), \forall  m \in [M], i \in \Omega_{k_{m}^{*}}\label{eq:induction3}.
\end{align}
\end{proposition}

Proposition~\ref{claim:main} implies that expert will only focus on the label signal and router will only focus on the cluster-center signal.  We will prove Proposition~\ref{claim:main} by induction. Before we move into the detailed proof of Proposition~\ref{claim:main}, we will first prove some important lemmas. 

\begin{lemma}\label{lm:fpibound}
For all $T_{1} \leq t \leq T_{2}$, the neural network parameter maintains following property.

\begin{itemize}
\item $|f_{m}(\xb_{i}; \Wb^{(t)})| = O(1), \forall m \in [M]$,
\item $\pi_{m_{i,t}}(\xb_{i};\bTheta^{(t)})= \Omega(1/M)$, $\forall i\in [n]$.
\end{itemize}
\end{lemma}
\begin{proof}
Because we use normalized gradient descent, the first bullet would be quite straight forward.
\begin{align*}
|f_{m}(\xb_{i}, \Wb^{(t)})| = \sum_{j \in [J]}\sum_{p \in [P]}\sigma(\la\wb_{m,j}^{(t)}, \xb_{i}^{(p)}\ra) \overset{(i)}{=}  O(1),    
\end{align*}
where (i) is by  $\|\wb_{m,j}^{(t)} - \wb_{m,j}^{(0)}\|_{2} = O(\eta T_{2}) = O(M^{-2})$ and $\xb_{i}^{(p)} = O(1)$.

Now we prove the second bullet. By Lemma~\ref{lm:keeporder}, we have that $h_{m_{i,t}}(\xb; \bTheta) \geq \max_{m}h_{m}(\xb; \bTheta) - 1$, which implies that 
\begin{align*}
\pi_{m_{i,t}}(\xb_{i};\bTheta^{(t)}) = \frac{\exp(h_{m_{i,t}}(\xb_i; \bTheta^{(t)}))}{\sum_{m}\exp(h_{m}(\xb; \bTheta^{(t)}))} \geq    \frac{\exp(h_{m_{i,t}}(\xb_i; \bTheta^{(t)}))}{M\max_{m}\exp(h_{m}(\xb; \bTheta^{(t)}))} \geq \frac{1}{eM}.
\end{align*}
\end{proof}

\begin{lemma}\label{lemma:pi}
Denote $\delta_{\bTheta} = \max_{i}\|\hb(\bar{\xb}_{i};\bTheta) - \hb(\xb_{i}; \bTheta)\|_{\infty}$ and let the random variable $\bar{m}_{i,t}$ be expert that get routed if we use the gating network output $\hb(\bar{\xb}_{i}; \bTheta^{(t)})$ instead. Then we have following inequalities, 
\begin{align}
&|\pi_{m}(\xb_{i};  \bTheta) - \pi_{m}(\bar{\xb}_{i}; \bTheta)| = O(\delta_{\bTheta}), \forall m \in [M], i \in [n],\label{eq:pi1}. \\
&|\mathbb{P}(m_{i,t} = m) - \mathbb{P}(\bar{m}_{i,t} = m)| = O(M^{2}\delta_{\bTheta}),  \forall m \in [M], i \in [n] \label{eq:pi2}. 
\end{align}
\end{lemma}
\begin{proof}
By definition of $\delta_{\bTheta}$, we have that $\|\hb(\xb_{i}; \bTheta^{(t)}) - \hb(\bar{\xb}_{i}; \bTheta^{(t)})\|_{\infty} \leq \delta_{\bTheta}$. Then applying Lemma~\ref{lm:Msmoothly} gives $|\mathbb{P}(m_{i,t} = m) - \mathbb{P}(\bar{m}_{k,t} = m)| = \tilde{O}(\delta_{\bTheta}),  \forall m \in [M], i \in [n]$, which completes the proof for \eqref{eq:pi2}.

Next we prove \eqref{eq:pi1}, which needs more effort. For all $i \in [n]$, we have
\begin{align*}
\pi_{m}(\xb_{i}; \bTheta) &= \frac{\pi_{m}(\bar{\xb}_{i}; \bTheta)\exp(h_{m}(\xb_{i}; \bTheta) - h_{m}(\bar{\xb}_{i}; \bTheta))}{\sum_{m'}\pi_{m'}(\bar{\xb}_{i}; \bTheta)\exp(h_{m'}(\xb_{i}; \bTheta) - h_{m'}(\bar{\xb}_{i}; \bTheta))}.
\end{align*}
Let $\delta_{m'} = \exp(h_{m'}(\xb_{i}; \bTheta) - h_{m'}(\bar{\xb}_{i}; \bTheta)) = 1 + O(\delta_{\bTheta})$. Then for sufficiently small $\delta_{\bTheta}$, we have that $\delta_{ m'} \geq 0.5$ . Then we can further compute
\begin{align*}
|\pi_{m}(\xb_i; \bTheta^{(t)}) - \pi_{m}(\bar{\xb}_{i}; \bTheta)| &= \pi_{m}(\bar{\xb}_{i}; \bTheta)\bigg|\frac{\delta_{m}}{\sum_{m'}\pi_{m'}(\bar{\xb}_{i}; \bTheta)\delta_{m'}} - 1\bigg| \\
&= \pi_{m}(\bar{\xb}_{i}; \bTheta)\frac{|\sum_{m'}\pi_{m'}(\bar{\xb}_{i}; \bTheta)(\delta_{m'}-\delta_{m})|}{\sum_{m'}\pi_{m'}(\bar{\xb}_{i}; \bTheta)\delta_{m'}}\\
&\leq \pi_{m}(\bar{\xb}_{i}; \bTheta)\frac{\sum_{m'}\pi_{m'}(\bar{\xb}_{i}; \bTheta)|\delta_{m'}-\delta_{m}|}{\sum_{m'}\pi_{m'}(\bar{\xb}_{i}; \bTheta)\delta_{m'}}\\
&\leq O(\delta_{\bTheta}),
\end{align*}
where the last inequality is by $|\delta_{m'}-\delta_{m}| \leq O(\delta_{\bTheta})$, $\pi_{m}(\bar{\xb}_{i}; \bTheta) \leq 1$ and $\sum_{m'}\pi_{m'}(\bar{\xb}_{i}; \bTheta)\delta_{m'}\geq [\sum_{m'}\pi_{m'}(\bar{\xb}_{i}; \bTheta)]/2 = 0.5$.
\end{proof}

Following Lemma implies that the pattern learned by experts during the first stage won't change in the second stage.

\begin{lemma}\label{lm:keepstage1} 
Suppose \eqref{eq: induction1}, \eqref{eq:induction2}, \eqref{eq:induction3} hold for all $t \in [T_{1}, T] \subseteq [T_{1}, T_{2}-1]$, then we have following inequalities hold for all $t \in [T_{1}, T+1]$,
\begin{align*}
&\la \wb_{m,j_{m}^{*}}^{(t)}, \vb_{k_{m}^{*}} \ra \geq (1 - O(\sigma_{0}^{0.1}))\eta t, \\
&\la \wb_{m,j}^{(t)}, \vb_{k} \ra = \tilde{O}(\sigma_{0}),\forall (j,k)\not= (j_{m}^{*}, k_{m}^{*}),\\
&\la \wb_{m,j}^{(t)}, \cb_{k} \ra= \tilde{O}(\sigma_{0}),  \forall j \in [J], k \in [K],\\
&\la \wb_{m,j}^{(t)}, \bxi_{i,p}\ra = \tilde{O}(\sigma_{0}), \forall j \in [J], k \in [K], i \in [n], p \geq 4.
\end{align*}
\end{lemma}
\begin{proof}
Most of the proof exactly follows the proof in the first stage, so we only list some key steps here. Recall that 
\begin{align*}
   \nabla_{\wb_{m,j}} \cL^{(t)}
    & =
    \frac{1}{n} \sum_{i,p}\ind(m_{i,t} =m)\ell'_{i,t}
    \pi_{m}(\xb_i; \bTheta^{(t)})y_{i}\sigma'(\la \wb_{m,j}^{(t)}, \xb_{i}^{(p)}\ra) \xb_{i}^{(p)}.
\end{align*}
In the proof of Lemma~\ref{lm:gradientinner}, we do Taylor expansion at the zero point. Now we will do Taylor expansion at $f_{m}(\hat{\xb}_{i}; \Wb)$ and $\pi(\bar{\xb}_{i}; \bTheta)$ as follows,
\begin{align*}
&|\pi_{m}(\xb_i; \bTheta^{(t)})f_{m}(\xb_{i};\Wb^{(t)}) -  \pi_{m}(\bar{\xb}_{i};\bTheta^{(t)})f_{m}(\hat{\xb}_{i};\Wb^{(t)})|\\
&\leq |\pi_{m}(\bar{\xb}_i; \bTheta^{(t)})[f_{m}(\xb_{i};\Wb^{(t)}) - f_{m}(\hat{\xb}_{i};\Wb^{(t)})]| + |[\pi_{m}(\xb_i; \bTheta^{(t)}) - \pi_{m}(\bar{\xb}_{i};\bTheta^{(t)}) ]f_{m}(\xb_{i};\Wb^{(t)})| \\
&\leq |f_{m}(\xb_{i};\Wb^{(t)}) - f_{m}(\hat{\xb}_{i};\Wb^{(t)})| + O(|\pi_{m}(\xb_i; \bTheta^{(t)}) - \pi_{m}(\bar{\xb}_{i};\bTheta^{(t)})|) \\
&\leq O(d^{-0.001}),   
\end{align*}
where the first inequality is by triangle inequality, the second inequality is by $\pi_{m}(\bar{\xb}_i; \bTheta^{(t)}) \leq 1$ and $|f_{m}(\xb_{i};\Wb^{(t)})| = O(1)$ in Lemma~\ref{lm:fpibound}, the third inequality is by \eqref{eq: induction1}, \eqref{eq:induction2} and \eqref{eq:pi1}. 

Then follow the proof of Lemma~\ref{lm:gradientinner}, we have that
\begin{align*}
\EE[\la \nabla_{\wb_{m,j}} \cL^{(t)}, \vb_{k_{m}^{*}} \ra]  &= -\frac{1}{n}\sum_{i\in \Omega_{k_{m}^{*}}}\mathbb{P}(m_{i,t}=m)\ell'_{i,t}\pi_{m}(\xb_i; \bTheta^{(t)})\sigma'(\la \wb_{m,j}^{(t)}, \vb_{k_{m}^{*}}\ra)\alpha_{i}^{3}\|\vb_{k_{m}^{*}}\|_{2}^{2}\\
    &\qquad -\frac{1}{n}\sum_{i\in \Omega_{k',k_{m}^{*}}}\mathbb{P}(m_{i,t}=m)\ell'_{i,t}
    \pi_{m}(\xb_{i}; \bTheta^{(t)})\sigma'(\la \wb_{m,j}^{(t)}, \vb_{k_{m}^{*}}\ra)\gamma_{i}^{3}y_{i}\epsilon_{i}\|\vb_{k_{m}^{*}}\|_{2}^{2}\\   
    &\qquad -\frac{1}{n}\sum_{i,p}\mathbb{P}(m_{i,t}=m)\ell'_{i,t}\pi_{m}(\xb_i; \bTheta^{(t)})\sigma'(\la \wb_{m,j}^{(t)}, \bxi_{i,p}\ra)y_{i}\la \vb_{k_{m}^{*}}, \bxi_{i,p}\ra\\
    &= \bigg[-\tilde{\Theta}\Big(\frac{1}{n}\Big)\sum_{i\in \Omega_{k_{m}^{*}}}\mathbb{P}(m_{i,t} = m)\alpha_{i}^{3} - \tilde{\Theta}\Big(\frac{1}{n}\Big)\sum_{i\in \Omega_{k',k_{m}^{*}}}\mathbb{P}(m_{i,t} = m)\gamma_{i}^{3}y_{i}\epsilon_{i}\\
    &\qquad + O(d^{-0.001})\bigg]\cdot \sigma'(\la \wb_{m,j}^{(t)}, \vb_{k_{m}^{*}}\ra) + \tilde{O}(d^{-1/2})\\
    &\overset{(i)}{=}  -\tilde{\Theta}(1)\sigma'(\la \wb_{m,j}^{(t)}, \vb_{k_{m}^{*}}\ra),
\end{align*}
where (i) is due to \eqref{eq:induction3}: $\mathbb{P}(m_{i,t} = m ) \geq \Theta(1/M)$, $\forall i \in \Omega_{k_{m}^{*}}, m \in [M]$. Again follow Lemma~\ref{lm:gradientinner} and Lemma~\ref{lm:GG}, we further have that 
\begin{align*}
    \la \nabla_{\wb_{m,j}} \cL^{(t)}, \vb_{k} \ra
    &=   -\tilde{\Theta}(1)
   [\la \wb_{m,j}^{(t)}, \vb_{k}\ra]^{2} , \\
   \la \nabla_{\wb_{m,j}} \cL^{(t)}, \cb_{k} \ra
    &=  \tilde{O}(1)
   [\la \wb_{m,j}^{(t)}, \cb_{k}\ra]^{2},\\
 \la \nabla_{\wb_{m,j}} \cL^{(t)}, \bxi_{i,p} \ra
    &= \tilde{O}(1)[\la\wb_{m,j}^{(t)}, \bxi_{i,p}\ra]^{2}.    
\end{align*}
Thus for all $T_{1} \leq t \leq T$, the update rule of every expert could be written as,
\begin{align*}
\la \wb_{m,j}^{(t+1)}, \vb_{k^{*}_{m}}\ra  &= \la \wb_{m,j}^{(t)}, \vb_{k^{*}_{m}}\ra + \tilde{\Theta}(1)\frac{\eta}{\| \nabla_{\Wb_{m}} \cL^{(t)}\|_{F}}\la \wb_{m,j}^{(t)}, \vb_{k^{*}_{m}}\ra^{2}\\
\la \wb_{m,j}^{(t+1)}, \vb_{k}\ra  &= \la \wb_{m,j}^{(t)}, \vb_{k}\ra + \tilde{O}(1)\frac{\eta}{\| \nabla_{\Wb_{m}} \cL^{(t)}\|_{F}}\la \wb_{m,j}^{(t)}, \vb_{k}\ra^{2}\\
\la \wb^{(t+1)}, \bxi_{i,p} \ra  &= \la \wb^{(t)}, \bxi_{i,p} \ra + \tilde{O}(1)\frac{\eta}{\| \nabla_{\Wb_{m}} \cL^{(t)}\|_{F}}\la \wb^{(t)}, \bxi_{i,p} \ra^{2}\\ 
\la \wb_{m,j}^{(t+1)}, \cb_{k}\ra  &=  \la \wb_{m,j}^{(t)}, \cb_{k}\ra +\tilde{O}(1)\frac{\eta}{\| \nabla_{\Wb_{m}} \cL^{(t)}\|_{F}}\la \wb_{m,j}^{(t)}, \cb_{k}\ra^{2}.
\end{align*}
By the first stage of training we have that 
$\la \wb_{m,j}^{(T_{1})}, \vb_{k^{*}_{m}}\ra  = \Theta(\sigma_{0}^{0.5})$, while others remains $\tilde{O}(\sigma_{0})$. Then we can use Lemma~\ref{lm: Tensor power update}, by choosing $S = \tilde{\Theta}(1)$ and $G = 2$, then we have that 
\begin{align*}
\la \wb_{m,j}^{(t)}, \vb_{k^{*}_{m}}\ra &= O(1).\\
\la \wb_{m,j}^{(t)}, \vb_{k}\ra &= \tilde{O}(\sigma_{0}), \forall k\not= k^{*}_{m}.\\
\la \wb_{m,j}^{(t)}, \cb_{k}\ra &= \tilde{O}(\sigma_{0}).\\
\la \wb^{(t)}, \bxi_{i,p} \ra &= \tilde{O}(\sigma_{0}).
\end{align*}
Then following Lemma~\ref{lm:shortphase} and \ref{lm:shortphase2}, we can prove that for all $T_{1} \leq t \leq T+1$, $m \in [M]$,
\begin{align*}
\la \wb_{m,j_{m}^{*}}^{(t)}, \vb_{k^{*}_{m}}\ra &\geq (1 - O(\sigma_{0}^{0.1}))\eta t,\\
\la \wb_{m,j}^{(t)}, \vb_{k}\ra &= \tilde{O}(\sigma_{0}), \forall (j, k) \not= (j_{m}^{*}, k_{m}^{*}),\\
\la \wb_{m,j}^{(t)}, \cb_{k}\ra &= \tilde{O}(\sigma_{0}), \forall j \in [J], k \in [K],\\
\la \wb_{m,j}^{(t)}, \bxi_{i,p} \ra &= \tilde{O}(\sigma_{0}), \forall j \in [J], i \in [n], p \geq 4.
\end{align*}
\end{proof}


 By the result of expert training we have following results

\begin{lemma}\label{lm:app1}
Suppose \eqref{eq: induction1}, \eqref{eq:induction2}, \eqref{eq:induction3} hold for all $t \in [T_{1}, T] \subseteq [T_{1}, T_{2}-1]$, then we have that $|f_{m}(\xb_{i}; \Wb^{(t)}) - f_{m}(\hat{\xb}_{i}; \Wb^{(t)}) | = \tilde{O}(\sigma_{0}^{3})$ for all $m \in [M]$ and $i \in [n]$, $t \in [T_{1}, T+1]$. Besides, 
\begin{align*}
y_{i}f_{m}(\hat{\xb}_{i}; \Wb^{(t)}) &= \sum_{j \in [J]}\Big[\alpha_{i}^{3}\sigma(\la \wb_{m,j}^{(t)}, \vb_{k}\ra) + \gamma_{i}^{3}\sigma(\la \wb_{m,j}^{(t)}, \vb_{k'} \ra)\Big],   \forall i \in \Omega_{k,k'}^{+}, m \in [M], \\
y_{i}f_{m}(\hat{\xb}_{i}; \Wb^{(t)}) &= \sum_{j \in [J]} \Big[\alpha_{i}^{3}\sigma(\la \wb_{m,j}^{(t)}, \vb_{k}\ra) - \gamma_{i}^{3}\sigma(\la \wb_{m,j}^{(t)}, \vb_{k'} \ra)\Big],  \forall   i \in \Omega_{k,k'}^{-}, m \in [M].
\end{align*}
\end{lemma}
\begin{proof}
For all $i \in \Omega_{k}$, we have that 
\begin{align*}
\big|f_{m}(\xb_{i}; \Wb^{(t)}) - f_{m}(\hat{\xb}_{i}; \Wb^{(t)}) \big| &\leq  \big|\sum_{j \in [J]} \sigma(\la \wb_{m,j}^{(t)}, \cb_{k}\ra)\big| + \big|\sum_{j \in [J], p \geq 4}\sigma(\la\wb_{m,j}^{(t)}, \bxi_{i,p}\ra)\big|\\
&\leq O(J)\cdot\max_{k,j}\sigma(\la \wb_{m,j}^{(t)}, \cb_{k}\ra) + O(J)\cdot\max_{i, j, p}|\sigma(\la\wb_{m,j}^{(t)}, \bxi_{i,p}\ra)|\\
&= \tilde{O}(\sigma_{0}^{3}),
\end{align*}
where the first inequality is by triangle inequality and the last equality is by Lemma~\ref{lm:keepstage1}.
\end{proof}

Next we will show that router only focus on the cluster-center signal rather than the label signal during the router training.

\begin{lemma}\label{lm:noiseinner}
Suppose \eqref{eq: induction1}, \eqref{eq:induction2}, \eqref{eq:induction3} hold for all $t \in [T_{1}, T] \subseteq [T_{1}, T_{2}-1]$, then we have that 
$\|\hb(\bar{\xb}_{i},\bTheta^{(t)}) - \hb(\xb_i; \bTheta^{(t)})\|_{\infty} = \tilde{O}(d^{-0.005})$ hold for all $i \in [n]$ and  $t \in [T_{1}, T+1]$. Besides, we have that $\max_{m,k}|\la \btheta_{m}^{(t)}, \vb_{k} \ra|, \max_{m,i,p}|\la \btheta_{m}^{(t)}, \bxi_{i,p} \ra| = \tilde{O}(d^{-0.005})$ for all $t \in [T_{1}, T+1]$.
\end{lemma}

\begin{proof}

Recall the definition of $\delta_{\bTheta}$ in Lemma~\ref{lemma:pi}, we need to show that $\delta_{\bTheta^{(t)}} = \tilde{O}(d^{-0.005})$ for all $t \in [T_{1}, T+1]$. We first prove following router parameter update rules,
\begin{align}
      \langle\nabla_{\btheta_{m}} \cL^{(t)}, 
    {\vb_k}
     \rangle
    = O(\delta_{\bTheta^{(t)}}K^{2}) + \tilde{O}(d^{-0.005}),
     \langle\nabla_{\btheta_{m}} \cL^{(t)}, 
    \bxi_{i,p}
     \rangle = \tilde{O}(d^{-0.005}), \label{lemma:gradient}
\end{align}
for all $T_{1} \leq t \leq T$, $m \in [M]$, $k \in [K]$, $i \in [n]$ and $p \geq 4$.

Consider the inner product of the router gradient and the feature vector and we have
\begin{align}
&\EE[\big \langle\nabla_{\btheta_{m}} \cL^{(t)}, 
    {\vb_k}
    \big \rangle]\notag\\ &=\underbrace{\frac{1}{n} \sum_{i \in \Omega_{k}}\mathbb{P}(m_{i,t} =m)\ell'_{i,t}
    y_{i}\pi_{m}(\xb_i; \bTheta^{(t)})f_{m}(\xb_{i};  \Wb^{(t)})y_{i}\alpha_{i}}_{I_{1}}\notag\\
    &\qquad + \underbrace{\frac{1}{n} \sum_{i \in \Omega_{k',k}}\mathbb{P}(m_{i,t} =m)\ell'_{i,t}
   y_{i} \pi_{m}(\xb_i; \bTheta^{(t)})f_{m}(\xb_{i};  \Wb^{(t)})\epsilon_{i}\gamma_{i}}_{I_{2}}\notag\\
    &\qquad - \underbrace{\frac{1}{n} \sum_{i\in \Omega_{k},m'\in[M]}\mathbb{P}(m_{i,t}=m')\ell'_{i,t}y_{i}
    \pi_{m'}(\xb_i; \bTheta^{(t)})\pi_{m}(\xb_i; \bTheta^{(t)})f_{m'}(\xb_{i}, \Wb^{(t)})y_{i}\alpha_{i}}_{I_{3}}\notag\\
    &\qquad - \underbrace{\frac{1}{n} \sum_{i\in \Omega_{k',k}, m'\in [M]}\mathbb{P}(m_{i,t} = m')\ell'_{i,t}
    \pi_{m'}(\xb_i; \bTheta^{(t)})y_{i}\pi_{m}(\xb_i; \bTheta^{(t)})f_{m'}(\xb_{i}, \Wb^{(t)})\epsilon_{i}\gamma_{i}}_{I_{4}} \notag\\
    &\qquad+\underbrace{\frac{1}{n} \sum_{i\in [n],p\geq 4}\mathbb{P}(m_{i,t} =m)\ell'_{i,t}
    y_{i}\pi_{m}(\xb_i; \bTheta^{(t)})f_{m}(\xb_{i};  \Wb^{(t)}) \la\xb_{i}^{(p)}, \vb_{k}\ra}_{I_{5}} \notag\\
    &\qquad - \underbrace{\frac{1}{n} \sum_{i \in [n],p \geq 4, m'\in [M]}\mathbb{P}(m_{i,t} = m')\ell'_{i,t}y_{i}\pi_{m'}(\xb_i; \bTheta^{(t)})\pi_{m}(\xb_i; \bTheta^{(t)})f_{m'}(\xb_{i};  \Wb^{(t)}) \la\xb_{i}^{(p)}, \vb_{k}\ra}_{I_{6}}.
    \label{eq:I1234}
\end{align}
Denote $y_{i}\pi_{m}(\bar{\xb}_{i}; \bTheta^{(t)})f_{m}(\hat{\xb}_{i}; \Wb^{(t)}), \forall i \in \Omega_{k,k'}^{+}$ by $\bar{F}_{k,k'}^{+}$. We next show that the output of the MoE multiplied by label:  $y_{i}\pi_{m}(\xb_{i}; \bTheta^{(t)})f_{m}(\xb_{i}; \Wb), \forall i \in \Omega_{k,k'}^{+}$ can be approximated by $\bar{F}_{k,k'}^{+}$.
\begin{align*}
&|\pi_{m}(\xb_i; \bTheta^{(t)})f_{m}(\xb_{i};\Wb^{(t)}) -  \pi_{m}(\bar{\xb}_{i};\bTheta^{(t)})f_{m}(\hat{\xb}_{i};\Wb^{(t)})|\\
&\leq |[\pi_{m}(\xb_i; \bTheta^{(t)}) - \pi_{m}(\bar{\xb}_{i};\bTheta^{(t)}) ]f_{m}(\xb_{i};\Wb^{(t)})| + |\pi_{m}(\bar{\xb}_i; \bTheta^{(t)})[f_{m}(\xb_{i};\Wb^{(t)}) - f_{m}(\hat{\xb}_{i};\Wb^{(t)})]| \\
&\leq O(|\pi_{m}(\xb_i; \bTheta^{(t)}) - \pi_{m}(\bar{\xb}_{i};\bTheta^{(t)})|) + |f_{m}(\xb_{i};\Wb^{(t)}) - f_{m}(\hat{\xb}_{i};\Wb^{(t)})|\\
&\leq O(\delta_{\bTheta^{(t)}})  + \tilde{O}(\sigma_{0}^{3}),
\end{align*}
where the first inequality is by triangle inequality, the second inequality is by $\pi_{m}(\bar{\xb}_i; \bTheta^{(t)}) \leq 1$ and $|f_{m}(\xb_{i};\Wb^{(t)})| = O(1)$ in Lemma~\ref{lm:fpibound}, the third inequality is by \eqref{eq:pi1} and  Lemma~\ref{lm:app1}. 

Similarly, denote $y_{i}\pi_{m}(\bar{\xb}_{i}; \bTheta^{(t)})f_{m}(\hat{\xb}_{i}; \Wb^{(t)}), i \in \Omega_{k,k'}^{-}$ by $\bar{F}_{k,k'}^{-}$ and we can show that value $y_{i}\pi_{m}(\xb_{i}; \bTheta^{(t)})f_{m}(\xb_{i}; \Wb^{(t)}), \forall i \in \Omega_{k,k'}^{-}$ can be approximated by $\bar{F}_{k,k'}^{-}$. Now we can bound $I_{1}$ as follows, 
\begin{align*}
I_{1} &= \sum_{k'\not = k}\frac{\ell'(\bar{F}_{k,k'^{+}})\bar{F}_{k,k'}^{+}}{n} \sum_{i\in \Omega_{k,k'}^{+}}\big[\mathbb{P}(m_{i,t} = m)y_{i}\alpha_{i} + O(\delta_{\bTheta^{(t)}})\big] + \tilde{O}(\sigma_{0}^{3})\\ 
&\qquad+ \sum_{k'\not = k}\frac{\ell'(\bar{F}_{k,k'^{-}})\bar{F}_{k,k'}^{-}}{n} \sum_{i\in \Omega_{k,k'}^{-}}\big[\mathbb{P}(m_{i,t} = m)y_{i}\alpha_{i} + O(\delta_{\bTheta^{(t)}})\big] + \tilde{O}(\sigma_{0}^{3})\\
&\overset{(i)}{=}\sum_{k'\not = k}\frac{\ell'(\bar{F}_{k,k'^{+}})\bar{F}_{k,k'}^{+}}{n} \sum_{i\in \Omega_{k,k'}^{+}}\big[\mathbb{P}(\bar{m}_{i,t} = m)y_{i}\alpha_{i} + O(M^{2}\delta_{\bTheta^{(t)}})\big] + \tilde{O}(\sigma_{0}^{3})\\ 
&\qquad+ \sum_{k'\not = k}\frac{\ell'(\bar{F}_{k,k'^{-}})\bar{F}_{k,k'}^{-}}{n} \sum_{i\in \Omega_{k,k'}^{-}}\big[\mathbb{P}(\bar{m}_{i,t} = m)y_{i}\alpha_{i} + O(M^{2}\delta_{\bTheta^{(t)}})\big] + \tilde{O}(\sigma_{0}^{3})\\
&\overset{(ii)}{=} O(M^{2}\delta_{\bTheta^{(t)}})+ \tilde{O}(n^{-1/2} + \sigma_{0}^{3})\\
&= O(M^{2}\delta_{\bTheta^{(t)}})+ \tilde{O}(d^{-0.005})
\end{align*}
where (i) is due to \eqref{eq:pi2} and (ii) is by $\sum_{i\in \Omega_{k,k'}^{+}}y_{i}\alpha = \tilde{O}(\sqrt{n})$ and $\sum_{i\in \Omega_{k,k'}^{-}}y_{i}\alpha = \tilde{O}(\sqrt{n})$ in  Lemma~\ref{lm:data balance}. Similarly we can prove that $I_{2}, I_{3}, I_{4} = O(M^{2}\delta_{\bTheta^{(t)}})+ \tilde{O}(d^{-0.005})$. Since $\la \xb_{i}^{(p)},\vb_{i} \ra = \tilde{O}(d^{-1/2}), \forall p \geq 4$, $\pi_{m},\pi_{m_{i,t}} \leq 1$ and $ f_{m_{i,t}} = O(1)$, we can upper bound $I_{5}, I_{6}$ by $\tilde{O}(d^{-1/2})$. Plugging those bounds into the gradient computation \eqref{eq:I1234} gives
\begin{align*}
\EE[\big \langle\nabla_{\btheta_{m}} \cL^{(t)}, {\vb_k} \big \rangle]  = O(M^{2}\delta_{\bTheta^{(t)}})+ \tilde{O}(d^{-0.005}).
\end{align*}
We finally consider the alignment between router gradient and noise
\begin{align*}
\big \langle\nabla_{\btheta_{m}} \cL^{(t)}, 
    \bxi_{i',p'}
    \big \rangle
    & = \frac{1}{n} \sum_{i\in [n],p\geq 4}\ind(m_{i,t} =m)\ell'_{i,t}
    y_{i}\pi_{m_{i,t}}(\xb_i; \bTheta^{(t)})f_{m_{i,t}}(\xb_{i};  \Wb^{(t)}) \la\xb_{i}^{(p)}, \bxi_{i',p'}\ra\\
    &\qquad - \frac{1}{n} \sum_{i \in [n],p \geq 4}\ell'_{i,t}y_{i}\pi_{m_{i,t}}(\xb_i; \bTheta^{(t)})\pi_{m}(\xb_i; \bTheta^{(t)})f_{m_{i,t}}(\xb_{i};  \Wb^{(t)}) \la\xb_{i}^{(p)}, \bxi_{i',p'}\ra.\\
     &\overset{(i)}{=}\tilde{O}\bigg(\frac{1}{n}\bigg) + \tilde{O}(d^{-1/2})\\
     &\overset{(ii)}{=} \tilde{O}(d^{-1/2}),
\end{align*}
where the (i) is by considering the cases $(i',p') = \bxi_{i,p}$ and $\bxi_{i',p'} \not= \bxi_{i,p}$ respectively and (ii) is due to our choice of $n$. Now, we have completed the proof of \eqref{lemma:gradient}. 

Plugging the gradient estimation \eqref{lemma:gradient} in to the gradient update rule for the gating network \eqref{eq:theta-update} gives  
\begin{align}
\max_{m, k}|\la \btheta_{m}^{(t+1)}, \vb_{k}\ra| &\leq  \max_{m, k}|\la \btheta_{m}^{(t)}, \vb_{k}\ra| + O(\eta_{r}M^{2}\delta_{\bTheta^{(t)}}) + \tilde{O}(\eta_{r}d^{-0.005})\label{eq:lb1}\\
\max_{m, i , p}|\la \btheta_{m}^{(t+1)}, \bxi_{i,p} \ra| &\leq \max_{m,i,p}|\la \btheta_{m}^{(t)}, \bxi_{i,p} \ra| + \tilde{O}(\eta_{r}d^{-0.005})\label{eq:lb2}
\end{align}
Combining \eqref{eq:lb1} and \eqref{eq:lb2},  we have that there exist $C_{1} = O(M^{2})$ and $C_{2} = \tilde{O}(d^{-0.005})$ such that $\delta_{\bTheta^{(t+1)}}\leq \delta_{\bTheta^{(t)}} + C_{1}\eta_{r} \delta_{\bTheta^{(t)}}  + C_{2}\eta_{r}$. Therefore, we have that
\begin{align*}
\delta_{\bTheta^{(t+1)}} + C_{1}^{-1}C_{2}&\leq (1+ C_{1}\eta_{r})[\delta_{\bTheta^{(t)}}+C_{1}^{-1}C_{2}]\\
&\leq \exp(C_{1}\eta_{r})[\delta_{\bTheta^{(t)}} + C_{1}^{-1}C_{2}],  
\end{align*}
where the last inequality is due to $\exp(z) \geq 1 + z$ for all $z \in \RR$. Then we further have that 
\begin{align*}
\delta_{\bTheta^{(t)}} 
&\leq \exp(C_{1}\eta_{r}t)[\delta_{\bTheta^{(0)}} + C_{1}^{-1}C_{2}]  \leq \exp(C_{1}\eta_{r}\eta^{-1}M^{-2})[\delta_{\bTheta^{(0)}} + C_{1}^{-1}C_{2}] = \tilde{O}(d^{-0.005}),
\end{align*}
where the last equality is by $\eta_{r} = \Theta(M^{2})\eta$.
\end{proof}

Define $\Delta_{\bTheta} := \max_{k \in [K]}\max_{m, m' \in \cM_{k}}\max_{(\xb_{i}, y_{i}) \in \Omega_{k}} |h_{m}(\xb_{i}; \bTheta) - h_{m'}(\xb_{i}; \bTheta)|$, which measures the bias of the router towards different experts in the same $\cM_{k}$. Following Lemma shows that the router will treats professional experts equally when $\Delta_{\btheta}$ is small.

\begin{lemma}\label{lm:D15}
 For all $t \geq 0$, we have that following inequality holds,
\begin{align*}
&\max_{k \in [K]}\max_{m, m' \in \cM_{k}}\max_{(\xb_{i}, y_{i}) \in \Omega_{k}}|\pi_{m'}(\xb_{i}; \bTheta^{(t)}) - \pi_{m}(\xb_{i}; \bTheta^{(t)})|  \leq 2\Delta_{\bTheta^{(t)}}, \\
&\max_{k \in [K]}\max_{m, m' \in \cM_{k}}\max_{(\xb_{i}, y_{i}) \in \Omega_{k}}|\mathbb{P}(m_{i,t} = m) - \mathbb{P}(m_{i,t} = m')| = O(M^{2})\Delta_{\bTheta^{(t)}}.
\end{align*}
\end{lemma}
\begin{proof}
By Lemma~\ref{lm:Msmooth2}, we directly have that
\begin{align*}
 |\mathbb{P}(m_{i,t} = m) - \mathbb{P}(m_{i,t} = m')| \leq O(M^{2})|h_{m}(\xb_i; \bTheta^{(t)}) - h_{m'}(\xb_i; \bTheta^{(t)})|.
\end{align*}

Then, we prove that 
\begin{align}
|\pi_{m'}(\xb_{i}; \bTheta) - \pi_{m}(\xb_{i}; \bTheta)| \leq  2|h_{m}(\xb_{i}; \bTheta^{(t)}) - h_{m'}(\xb_{i}; \bTheta^{(t)})|. \label{eq: pibound}    
\end{align}
When $|h_{m}(\xb_{i}; \bTheta^{(t)})- h_{m'}(\xb_{i}; \bTheta^{(t)})| \geq 1$, it is obvious that \eqref{eq: pibound} is true. When $|h_{m}(\xb_{i}; \bTheta^{(t)}) - h_{m'}(\xb_{i}; \bTheta^{(t)})| \leq 1$ we have that
\begin{align*}
|\pi_{m'}(\xb_{i}; \bTheta) - \pi_{m}(\xb_{i}; \bTheta)| &= \bigg|\frac{\exp(h_{m}(\xb_{i}; \bTheta^{(t)})) - \exp(h_{m'}(\xb_{i}; \bTheta^{(t)}))}{\sum_{m''}\exp\big(h_{m''}(\xb_{i}; \bTheta^{(t)})\big)}\bigg|\\
&= \bigg|\frac{\exp(h_{m'}(\xb_{i}; \bTheta^{(t)}))}{\sum_{m''}\exp\big(h_{m''}(\xb_{i}; \bTheta^{(t)})\big)}\bigg|\cdot |\exp(h_{m}(\xb_{i}; \bTheta^{(t)}) - h_{m'}(\xb_{i}; \bTheta^{(t)})) - 1|\\
&\leq  2|h_{m}(\xb_{i}; \bTheta^{(t)}) - h_{m'}(\xb_{i}; \bTheta^{(t)})|,
\end{align*}
which completes the proof of \eqref{eq: pibound}.
\end{proof}

Notice that the gating network is initialized to be zero, so we have $\Delta_{\bTheta}=0$ at initialization. We can further show that $\Delta_{\bTheta} = O\big(1/\text{poly}(d)\big)$ during the training up to time $T = \tilde{O}(\eta^{-1})$. 

\begin{lemma}\label{lm:expclose}
Suppose \eqref{eq: induction1}, \eqref{eq:induction2}, \eqref{eq:induction3} hold for all $t \in [T_{1}, T] \subseteq [T_{1}, T_{2}-1]$, then we have that 
$\Delta_{\bTheta^{(t)}} \leq \tilde{O}(d^{-0.001})$ holds for all $t \in [T_{1}, T+1]$.
\end{lemma}
\begin{proof}
One of the key observation is the similarity of the m-th and the $m'$-th expert in the same expert class $\cM_{k}$. Lemma~\ref{lm:keepstage1} implies that $\max_{i\in \Omega_{k}}|f_{m}(\xb_{i},\Wb^{(t)}) - f_{m'}(\xb_{i},\Wb^{(t)})| = \tilde{O}(\sigma_{0}^{0.1}) \leq \tilde{O}(d^{-0.001})$.

Another key observe is that, we only need to focus on the $k-th$ cluster-center signal. Lemma~\ref{lm:noiseinner} implies that,
\begin{align*}
\Delta_{\bTheta^{(t)}} &= \max_{k \in [K]}\max_{m, m' \in \cM_{k}}\max_{(\xb_{i}, y_{i}) \in \Omega_{k}} |h_{m}(\xb_{i}; \bTheta) - h_{m'}(\xb_{i}; \bTheta^{(t)})|\\
&\leq \max_{k \in [K]}\max_{m, m' \in \cM_{k}}\max_{(\xb_{i}, y_{i}) \in \Omega_{k}} |h_{m}(\bar{\xb}_{i}; \bTheta^{(t)}) - h_{m'}(\bar{\xb}_{i}; \bTheta^{(t)})| + 2\delta_{\bTheta^{(t)}}\\
&= \max_{k \in [K]}\max_{m, m' \in \cM_{k}} |\la \btheta_{m} - \btheta_{m'}, \beta_{i}\cb_{k} \ra| + 2\delta_{\bTheta^{(t)}}\\
&\leq C_{2}\max_{k \in [K]}\max_{m, m' \in \cM_{k}} |\la \btheta_{m} - \btheta_{m'}, \cb_{k} \ra| + 2\delta_{\bTheta^{(t)}},
\end{align*}
where the first inequality is by Lemma~\ref{lm:noiseinner} and the second inequality is by $\beta_{i} \leq C_{2}$. We now prove that following gradient difference is small
\begin{align*}
&\big\la\nabla_{\btheta_{m}}\cL^{(t)} - \nabla_{\btheta_{m'}}\cL^{(t)}, \cb_{k}\big\ra\\
&\overset{(i)}{=}  \frac{1}{n} \sum_{i \in [n]}\sum_{p \in [P]}\mathbb{P}(m_{i,t} = m)\ell'_{i,t}
    \pi_{m}(\xb_i; \bTheta^{(t)})y_{i}f_{m}(\xb_{i};  \Wb^{(t)})\la\xb_{i}^{(p)}, \cb_{k}\ra\\
    &\qquad - \frac{1}{n} \sum_{i \in [n]}\sum_{p \in [P]}\mathbb{P}(m_{i,t} = m')\ell'_{i,t}\pi_{m'}(\xb_i;\bTheta^{(t)})y_{i}f_{m'}(\xb_{i};  \Wb^{(t)})\la\xb_{i}^{(p)}, \cb_{k}\ra\\
    &\qquad  + \frac{1}{n} \sum_{i\in \Omega_{k}}\sum_{p \in [P]}\sum_{m''\in [M]}[\pi_{m'}(\xb_i; \bTheta^{(t)}) - \pi_{m}(\xb_i; \bTheta^{(t)})]\mathbb{P}(m_{i,t} = m'')\ell'_{i,t}
    \pi_{m''}(\xb_i; \bTheta^{(t)})\cdot \\
    &\qquad y_{i}f_{m''}(\xb_{i}, \Wb)\la\xb_{i}^{(p)}, \cb_{k}\ra + \tilde{O}(d^{-0.001})\\
    &= O\Big(\frac{1}{n}\Big) \sum_{i \in \Omega_{k}}[ \mathbb{P}(m_{i,t} = m') - \mathbb{P}(m_{i,t} = m) ]|\ell'_{i,t}
    \pi_{m}(\xb_i; \bTheta)\beta_{i}y_{i}f_{m}(\xb_{i};  \Wb^{(t)})| + \tilde{O}(d^{-0.001})\\
    &\qquad + O(1)\max_{i\in \Omega_{k}}|\pi_{m'}(\xb_{i}; \bTheta^{(t)}) - \pi_{m}(\xb_{i}; \bTheta^{(t)})| + O(1)\max_{i\in \Omega_{k}}|f_{m}(\xb_{i},\Wb^{(t)}) - f_{m'}(\xb_{i},\Wb^{(t)}) |\\
    &= O(1) |\mathbb{P}(m_{i,t} = m') - \mathbb{P}(m_{i,t} = m) ]| + O(1)\max_{i\in \Omega_{k}}|\pi_{m'}(\xb_{i}; \bTheta^{(t)}) - \pi_{m}(\xb_{i}; \bTheta^{(t)})|\\
    &\qquad + O(1)\max_{i\in \Omega_{k}}|f_{m}(\xb_{i},\Wb^{(t)}) - f_{m'}(\xb_{i},\Wb^{(t)}) | + \tilde{O}(d^{-0.001})\\
    &\overset{(ii)}{=}  O(M^{2}\Delta_{\bTheta^{(t)}}) + \tilde{O}(d^{-0.001}),
\end{align*}
where the (i) is by Lemma~\ref{lm:easy} and (ii) is by Lemma~\ref{lm:D15}. It further implies that $\Delta_{\bTheta^{(t+1)}} \leq O(\eta_{r}M^{2})\Delta_{\bTheta^{(t)}} + \tilde{O}(\eta_{r}d^{-0.001})$. Following  previous proof of $\delta_{\bTheta}$, we have that $\Delta_{\bTheta^{(T+1)}}= \tilde{O}(d^{-0.001})$.
\end{proof}
Together with the key technique 1, we can infer that each expert $m \in \cM_{k}$ will get nearly the same load as other experts in $\cM_{k}$. Since $\Delta_{\bTheta}$ keeps increasing during the training, it cannot be bounded if we allow the total number of iterations goes to infinity in Algorithm~\ref{alg:GDrandominit}. This is the reason that we require early stopping in Theorem~\ref{thm: MoE}, which we believe can be waived by adding load balancing loss \citep{eigen2013learning, shazeer2017outrageously, fedus2021switch}, or advanced MoE layer structure such as BASE Layers \citep{lewis2021base, dua2021tricks} and Hash Layers \citep{roller2021hash}.

\begin{lemma}\label{lm: ckinnerdecrease}
Suppose \eqref{eq: induction1}, \eqref{eq:induction2}, \eqref{eq:induction3} hold for all $t \in [T_{1}, T] \subseteq [T_{1}, T_{2}-1]$, then for $m \notin \cM_{k}$ and $t \in [T_{1}, T]$ , if $\la \btheta_{m}^{(t)}, \cb_{k} \ra \geq \max_{m'}\la \btheta_{m'}^{(t)}, \cb_{k} \ra - 1$ we have that
\begin{align*}
\langle\nabla_{\btheta_{m}} \cL^{(t)},{\cb_k}\rangle \geq \Omega\bigg(\frac{\eta^{3}t^{3}}{KM^{3}}\bigg) + \tilde{O}(d^{-0.005}).
\end{align*}
\end{lemma}

\begin{proof}
The expectation of the inner product $\la\nabla_{\btheta_{m}} \cL^{(t)},\cb_{k}\ra$ can be computed as follows,
\begin{align}
\EE[\la\nabla_{\btheta_{m}} \cL^{(t)},\cb_{k}\ra]
    & =
    \frac{1}{n} \sum_{i,p}\mathbb{P}(m_{i,t} =m)\ell'_{i,t}
    \pi_{m}(\xb_i; \bTheta^{(t)})y_{i}f_{m}(\xb_{i};  \Wb^{(t)}) \la\xb_{i}^{(p)}, \cb_{k}\ra \notag\\
    &\qquad - \frac{1}{n} \sum_{i,p,m'}\mathbb{P}(m_{i,t} = m')\ell'_{i,t}
    \pi_{m'}(\xb_i; \bTheta^{(t)})\pi_{m}(\xb_i; \bTheta^{(t)})y_{i}f_{m'}(\xb_{i}, \Wb^{(t)}) \la\xb_{i}^{(p)},\cb_{k}\ra \notag\\
    &\overset{(i)}{=} \frac{1}{n} \sum_{i \in \Omega_{k}}\mathbb{P}(m_{i,t} = m)\ell'_{i,t}
    \pi_{m}(\xb_i; \bTheta^{(t)})\beta_{i}y_{i}f_{m}(\xb_{i};  \Wb^{(t)}) + \tilde{O}(d^{-0.005}) \notag\\
    &\qquad - \frac{1}{n} \sum_{i\in \Omega_{k}}\sum_{m'\in [M]}\mathbb{P}(m_{i,t} = m')\ell'_{i,t}
    \pi_{m'}(\xb_i; \bTheta^{(t)})\pi_{m}(\xb_i; \bTheta^{(t)})\beta_{i}y_{i}f_{m'}(\xb_{i}, \Wb)\label{eq:decrease1}.
\end{align}
where (i) is due to $|\la \bxi_{i,p}, \cb_{k} \ra| = \tilde{O}(d^{-0.5})$.

We can rewrite the inner product \eqref{eq:decrease1} as follows,

\begin{align}
\EE[\big \langle\nabla_{\btheta_{m}} \cL^{(t)}, 
    {\cb_k}
    \big \rangle]
    & =
   \frac{1}{n} \sum_{i \in \Omega_{k}}\mathbb{P}(m_{i,t} = m)\ell'_{i,t}
    \pi_{m}(\xb_i; \bTheta^{(t)})\beta_{i}y_{i}f_{m}(\xb_{i};  \Wb^{(t)}) + \tilde{O}(d^{-0.005}) \notag\\
    &\qquad - \frac{1}{n} \sum_{i\in \Omega_{k}}\sum_{m'\in [M]}\mathbb{P}(m_{i,t} = m')\ell'_{i,t}
    \pi_{m'}(\xb_i; \bTheta^{(t)})\pi_{m}(\xb_i; \bTheta^{(t)})\beta_{i}y_{i}f_{m'}(\xb_{i}, \Wb) \notag\\
    &=   \underbrace{\frac{1}{n} \sum_{i \in \Omega_{k}}\mathbb{P}(m_{i,t} = m)\ell'_{i,t}
    \pi_{m}(\xb_i; \bTheta^{(t)})y_{i}\beta_{i}f_{m}(\xb_{i};  \Wb^{(t)})}_{I_{1}} + \tilde{O}(d^{-0.005}) \notag\\
    &\qquad \underbrace{- \frac{1}{n} \sum_{i\in \Omega_{k}, m'\in \cM_{k}}\mathbb{P}(m_{i,t}= m')\ell'_{i,t}
    \pi_{m'}(\xb_i; \bTheta^{(t)})\pi_{m}(\xb_i; \bTheta^{(t)})\beta_{i}y_{i}f_{m'}(\xb_{i}, \Wb^{(t)})}_{I_{2}}\\
    &\qquad \underbrace{- \frac{1}{n} \sum_{i\in \Omega_{k}, m'\notin \cM_{k}}\mathbb{P}(m_{i,t}= m')\ell'_{i,t}
    \pi_{m'}(\xb_i; \bTheta^{(t)})\pi_{m}(\xb_i; \bTheta^{(t)})\beta_{i}y_{i}f_{m'}(\xb_{i}, \Wb^{(t)})}_{I_{3}}\label{eq:decrease3}.
\end{align}
To calculate $I_{1}, I_{2}, I_{3}$, let's first lower bound $I_{2}$. 
We now consider the case that $m \not\in \cM_{k}, m'\in \cM_{k}$. Because $\la \btheta^{(t)}_{m}, \cb_{k} \ra \geq \max_{m'}\la \btheta^{(t)}_{m}, \cb_{k} \ra -1$, we can easily prove that $\pi_{m}(\xb_{i};\bTheta^{(t)})= \Omega(1/M), \forall i \in \Omega_{k} $. Then we have that 
\begin{align*}
I_{2} &= - \frac{1}{n} \sum_{i\in \Omega_{k}, m'\in \cM_{k}}\mathbb{P}(m_{i,t}= m')\ell'_{i,t}
    \pi_{m'}(\xb_i; \bTheta^{(t)})\pi_{m}(\xb_i; \bTheta^{(t)})\beta_{i}y_{i}f_{m'}(\xb_{i}, \Wb^{(t)})\\
    &\geq  \Omega\Big(\frac{\eta^{3}t^{3}}{nM^{3}}\Big) \sum_{i\in \Omega_{k}, m'\in \cM_{k}}\beta_{i}\\
    &\geq \Omega\Big(\frac{\eta^{3}t^{3}}{KM^{3}}\Big),
\end{align*}
where the first inequality is by  $\pi_{m'}(\xb_{i};\bTheta^{(t)})= \Omega(1/M)$, $\mathbb{P}(m_{i,t} = m' ) \geq \Theta(1/M)$, $\forall i \in \Omega_{k_{m}^{*}}, m \in [M]$, $y_{i}f_{m'}(\xb_i;\Wb^{(t)}) = \eta^{3}t^{3}(1 - O(\sigma_{0}^{0.1}))$ and $\ell' = -\Theta(1)$ for all $i \in \Omega_{k}, m' \in \cM_{k}$ due to Proposition~\ref{claim:main} and Lemma~\ref{lm:keepstage1}, and the last inequality is by $|\cM_{k}|\geq 1$ in Lemma~\ref{lm: Mkset} and $\sum_{i \in \Omega_{k}}\beta_{i} = \Omega(n/K)$ in Lemma~\ref{lm:data balance}. 

Then we consider the case that $m, m' \not \in \cM_{k}$. Applying Taylor expansion of $\ell'_{i,t} = 1/2 + O(J\eta^{3}t^{3})$ gives
\begin{align}
&\frac{1}{n} \sum_{i \in \Omega_{k}}\mathbb{P}(m_{i,t} = m)\ell'_{i,t}
    \pi_{m}(\xb_i; \bTheta^{(t)})y_{i}\beta_{i}f_{m}(\xb_{i};  \Wb^{(t)})\notag\\
    &= \frac{1}{2n} \sum_{i \in \Omega_{k}}\mathbb{P}(m_{i,t} = m)\pi_{m}(\xb_i; \bTheta^{(t)})y_{i}\beta_{i}f_{m}(\xb_{i};  \Wb^{(t)}) + O\big(J^{2}\eta^{6}t^{6}\big)\notag\\
    &= \frac{1}{2n} \sum_{k'}\sum_{i \in \Omega_{k,k'}^{+}}\mathbb{P}(m_{i,t} = m)\pi_{m}(\xb_i; \bTheta^{(t)})y_{i}\beta_{i}f_{m}(\xb_{i};  \Wb^{(t)})  + O\big(J^{2}\eta^{6}t^{6}\big)\notag\\
    &\qquad + \frac{1}{2n} \sum_{k'}\sum_{i \in \Omega_{k,k'}^{-}}\mathbb{P}(m_{i,t} = m)\pi_{m}(\xb_i; \bTheta^{(t)})y_{i}\beta_{i}f_{m}(\xb_{i};  \Wb^{(t)})\notag\\
    &= O(J^{2}\eta^{6}t^{6})+ \tilde{O}(d^{-0.005}). \label{eq:for-i1i3}
\end{align}
where the last inequality is by the technique we have used before in Lemma~\ref{lm:expclose}. By \eqref{eq:for-i1i3}, we can get upper bound $|I_{1}|, |I_{3}|$ by $O(J^{2}\eta^{6}t^{6})+ \tilde{O}(d^{-0.005})$.

Plugging the bound of $I_{1}, I_{2}, I_{3}$ into \eqref{eq:decrease3} gives,
\begin{align*}
\big \langle\nabla_{\btheta_{m}} \cL^{(t)}, 
    {\cb_k}
    \big \rangle &\geq \Omega\bigg(\frac{\eta^{3}t^{3}}{KM^{3}}\bigg) + O(J^{2}\eta^{6}t^{6}) + \tilde{O}(d^{-0.005})\\
&\leq \Omega\bigg(\frac{\eta^{3}t^{3}}{KM^{3}}\bigg) + \tilde{O}(d^{-0.005}),    
\end{align*}
where the last inequality is by $t \leq T_{2} = \lfloor \eta^{-1}M^{-2}\rfloor$.
\end{proof}




Now we can claim that Proposition~\ref{claim:main} is true and we summarize the results as follow lemma.
\begin{lemma}\label{lm:stage2end}
For all $T_{1} \leq t \leq T_{2}$, we have Proposition~\ref{claim:main} holds. Besides, we have that $\la \btheta_{m}^{(T_{2})}, \cb_{k} \ra \leq \max_{m' \in [M]}\la \btheta_{m'}^{(T_{2})}, \cb_{k} \ra - \Omega(K^{-1}M^{-9})$ for all $m \notin \cM_{k}$..
\end{lemma}
\begin{proof}
    We will first use induction to prove Proposition~\ref{claim:main}. It is worth noting that proposition~\ref{claim:main} is true at the beginning of the second stage $t = T_{1}$. Suppose \eqref{eq: induction1}, \eqref{eq:induction2}, \eqref{eq:induction3} hold for all $t \in [T_{1}, T] \subseteq [T_{1}, T_{2}-1]$, we next verify that they also hold for $t \in [T_{1}, T + 1]$. Lemma~\ref{lm:app1} shows that \eqref{eq: induction1} holds for  $t \in [T_{1}, T + 1]$. Lemma~\ref{lm:noiseinner} further shows that \eqref{eq: induction1} holds for  $t \in [T_{1}, T + 1]$. Therefore, we only need to verify whether \eqref{eq:induction3} holds for  $t \in [T_{1}, T + 1]$. Therefore, for each pair $i \in \Omega_{k}$, $m \in \cM_{k}$, we need to estimate the gap between expert $m$ and the expert with best performance $h_{m}(\xb_{i};\bTheta^{(t)}) - \max_{m'}h_{m'}(\xb_{i};\bTheta^{(t)})$. By Lemma~\ref{lm: ckinnerdecrease} and Lemma~\ref{lm:noiseinner}, we can induce that $h_{m}(\xb_{i};\bTheta^{(t)})$ is small therefore cannot be the largest one. Thus $h_{m}(\xb_{i};\bTheta^{(t)}) - \max_{m'}h_{m'}(\xb_{i};\bTheta^{(t)}) = h_{m}(\xb_{i};\bTheta^{(t)}) - \max_{m'}h_{m'}(\xb_{i};\bTheta^{(t)})\leq \Delta_{\bTheta^{(t)}} \leq \tilde{O}(d^{-0.001}) $. Therefore, by Lemma~\ref{lm:Msmooth2} we have \eqref{eq:induction3} holds. 
Now we have verified that \eqref{eq:induction3} also holds for  $t \in [T_{1}, T + 1]$, which completes the induction for Lemma~\ref{claim:main}.



Finally, we carefully characterize the value of $\langle \btheta_m^{(t)}, \cb_k \rangle$, for $\eta_{r}\eta^{-1} = \Theta(M^{2})$ and $m \notin \cM_{k}$. If $\la \btheta_{m}^{(t)}, \cb_{k} \ra \geq \max_{m'}\la \btheta_{m'}^{(t)}, \cb_{k} \ra - 1$, by Lemma~\ref{lm: ckinnerdecrease} we have that 
\begin{align}
\la \btheta_{m}^{(t+1)}, c_{k}\ra  &\leq \la \btheta_{m}^{(t)}, c_{k}\ra -\Theta\bigg(\frac{\eta_{r}\eta^{3}t^{3}}{KM^{3}}\bigg) + \tilde{O}(\eta_{r}d^{-0.005}) \leq 0.\label{eq:tsum}
\end{align}
If there exists $t \leq T_{2} - 1$ such that $\la \btheta_{m}^{(t+1)}, c_{k}\ra \leq \max_{m'}\la \btheta_{m'}^{(t)}, \cb_{k} \ra - 1$, clearly we have that $\la \btheta_{m}^{(T_{2})}, c_{k}\ra \leq - \Omega(K^{-1}M^{-9})$ since $\la \btheta_{m}^{(t)}, c_{k}\ra$ will keep decreasing as long as $\la \btheta_{m}^{(t+1)}, c_{k}\ra \geq -1$ and our step size $\eta_{r} = \Theta(M^{2})\eta$ is small enough. If $\la \btheta_{m}^{(t+1)}, c_{k}\ra \geq \max_{m'}\la \btheta_{m'}^{(t)}, \cb_{k} \ra - 1$ holds for all $t \leq T_{2} -1$, take telescope sum of \eqref{eq:tsum} from $t = 0$ to $t = T_{2}-1$ gives that 
\begin{align*}
\la \btheta_{m}^{(T_{2})}, c_{k}\ra &\leq \la \btheta_{m}^{(0)}, c_{k}\ra -\sum_{s=0}^{T_{2}-1}\Theta\bigg(\frac{\eta_{r}\eta^{3}s^{3}}{KM^{3}}\bigg) + \tilde{O}(d^{-0.005})\\
&\overset{(i)}{=} -\sum_{s=0}^{T_{2}-1}\Theta\bigg(\frac{\eta_{r}\eta^{3}s^{3}}{KM^{3}}\bigg) + \tilde{O}(d^{-0.005})\\
&\overset{(ii)}{=}  -\Theta\bigg(\frac{\eta_{r}\eta^{3}T_{2}^{4}}{KM^{3}}\bigg) +  \tilde{O}(d^{-0.005})\\
&\leq -\Omega(K^{-1}M^{-9}),   
\end{align*}
where the (i) is by $\btheta_{m}^{(0)} = 0$ and (ii) is by $\sum_{i=0}^{n-1}i^{3} = n^{2}(n-1)^{2}/4$ and the last inequality is due to $T_{2} = \lfloor \eta^{-1}M^{-2}\rfloor$ and $\eta_{r} = \Theta(M^{2})\eta$. Now we have proved that  $\la \btheta_{m}^{(T_{2})}, c_{k}\ra \leq -\Omega(K^{-1}M^{-9})$ for all $m \notin \cM_{k}$. Finally, by Lemma~\ref{lm: zero mean} we have that
\begin{align*}
\max_{m' \in [M]}\la \btheta_{m'}^{(T_{2})}, \cb_{k} \ra \geq \frac{1}{m}\sum_{m' \in [M]}\la \btheta_{m'}^{(T_{2})}, \cb_{k} \ra = 0.
\end{align*}
Therefore, we have that $\la \btheta_{m}^{(T_{2})}, \cb_{k} \ra \leq - \Omega(K^{-1}M^{-9}) \leq \max_{m' \in [M]}\la \btheta_{m'}^{(T_{2})}, \cb_{k} \ra - \Omega(K^{-1}M^{-9})$, which completes the proof.



\end{proof}


\subsection{Generalization Results}
In this section, we will present the detailed proof of Lemma~\ref{lm:stage1} and Theorem~\ref{thm: MoE} based on analysis in the previous stages.

\begin{proof}[Proof of Lemma~\ref{lm:stage1}]
We consider the $m$-th expert in the MoE layer, suppose that $m \in \cM_{k}$. Then if we draw a new sample $(\xb, y) \in \Omega_{k}$. Without loss of generality, we assume $\xb= [\alpha y\vb_{k}, \beta \cb_{k}, \gamma\epsilon\vb_{k'}, \bxi]$. By Lemma~\ref{lm:shortphase2}, we have already get the bound for inner product between weights and feature signal, cluster-center signal and feature noise. However, we need to recalculate the bound of the inner product between weights and random noises because we have fresh random noises i.i.d drawn from $\cN(0, (\sigma_{p}^{2}/d) \cdot I_{d})$. Notice that we use normalized gradient descent for expert with step size $\eta$, so we have that 
\begin{align*}
\|\wb_{m,j}^{(T_{1})} - \wb_{m,j}^{(0)}\|_{2} \leq \eta T_{1} = O(\sigma_{0}^{0.5}).     
\end{align*}
Therefore, by triangle inequality we have that $\|\wb_{m,j}^{(T_{1})}\|_{2} \leq \|\wb_{m,j}^{(0)}\|_{2}  + O(\sigma_{0}^{0.5}) \leq \tilde{O}(\sigma_{0}\sqrt{d})$. Because the inner product $\la \wb_{m,j}^{(t)}, \bxi_{p}\ra$ follows the distribution $\cN(0, (\sigma_{p}^{2}/d)\cdot \|\wb_{m,j}^{(T_{1})}\|_{2}^{2})$, we have that with probability at least $1 - 1/(dPMJ)$, 
\begin{align*}
|\la \wb_{m,j}^{(T_{1})}, \bxi_{p}\ra| = O(\sigma_{p}d^{-1/2}\|\wb_{m,j}^{(t)}\|_{2}\log(dPMJ)) \leq \tilde{O}(\sigma_{0}). 
\end{align*}
Applying Union bound for $m \in [M], j \in [J], p \geq 4$ gives that, with probability at least $1 - 1/d$, 
\begin{align}
|\la \wb_{m,j}^{(T_{1})}, \bxi_{p}\ra| = \tilde{O}(\sigma_{0}), \forall m \in [M], j \in [J], p \geq 4.\label{eq:event}
\end{align}

Now under the event that \eqref{eq:event} holds, we have that 
\begin{align*}
yf_{m}(\xb, \Wb^{(t)}) &= y\sum_{j\in [J]}\sum_{p\in[P]}\sigma(\la\wb_{m,j}, \xb^{(p)}\ra)\\
&= y\sigma(\la\wb_{m,j_{m}^{*}}, \alpha y\vb_{k}\ra) + y\sum_{(j,p)\not = (j_{m}^{*}, 1)}\sigma(\la\wb_{m,j}, \xb^{(p)}\ra)\\
&\geq C_{1}^{3}(1-\sigma_{0}^{0.1})^{3}\sigma_{0}^{1.5} -  \tilde{O}(\sigma_{0}^{3})\\
&\geq \Omega(\sigma_{0}^{1.5}),
\end{align*}
where the first inequality is due to \eqref{lm:stage1fbound}.
Because \eqref{eq:event} holds holds with probability at least $1 - 1/d$, so we have prove that 
\begin{align*}
\mathbb{P}_{(\xb, y)\sim \cD}\big(yf_{m}(\xb; \Wb^{(T_{1})}\big) \leq 0 \big|(\xb, y) \in \Omega_{k}\big) &\leq 1/d.  
\end{align*}

On the other hand, if we draw a new sample $(\xb,y) \in \Omega_{k'}, k' \not = k$. Then we consider the special set $\Omega_{k',k}^{-} \subseteq \Omega_{k'}$ where feature noise is $\vb_{k}$ and the sign of the feature noise $\epsilon$ is not equal to the label $y$. Without loss of generality, we assume it as $\xb = [\alpha y \vb_{k'}, \beta \cb_{k'}, -\gamma y  \vb_{k}, \bxi]$. Then under the event that \eqref{eq:event} holds, we have that 
\begin{align*}
yf_{m}(\xb, \Wb^{(t)}) &= y\sum_{j\in [J]}\sum_{p\in[P]}\sigma(\la\wb_{m,j}, \xb^{(p)}\ra)\\
&= y\sigma(\la\wb_{m,j_{m}^{*}}, -\gamma y\vb_{k}\ra) + y\sum_{(j,p)\not = (j_{m}^{*}, 3)}\sigma(\la\wb_{m,j}, \xb^{(p)}\ra)\\
&\leq - C_{1}^{3}(1-\sigma_{0}^{0.1})^{3}\sigma_{0}^{1.5} +  \tilde{O}(\sigma_{0}^{3})\\
&\leq -\Omega(\sigma_{0}^{1.5}),
\end{align*}
where the first inequality is due to \eqref{lm:stage1fbound}.
Because \eqref{eq:event} holds holds with probability at least $1 - 1/d$, so we have prove that 
\begin{align*}
\mathbb{P}_{(\xb, y)\sim \cD}\big(yf_{m}(\xb; \Wb^{(T_{1})}\big) \leq 0 \big|(\xb, y) \in \Omega_{k',k}^{-}\big) \geq 1 - 1/d.  
\end{align*}
Then we further have that 
\begin{align*}
&\mathbb{P}_{(\xb, y)\sim \cD}\big(yf_{m}(\xb; \Wb^{(T_{1})}\big) \leq 0 \big|(\xb, y) \in \Omega_{k'}\big) \\
&\geq \mathbb{P}_{(\xb, y)\sim \cD}\big(yf_{m}(\xb; \Wb^{(T_{1})}\big) \leq 0 \big|(\xb, y) \in \Omega_{k',k}^{-}\big)\cdot \mathbb{P}_{(\xb, y)\sim \cD}\big((\xb, y) \in \Omega_{k',k}^{-} \big|(\xb, y) \in \Omega_{k'}\big) \\
&\geq \Omega(1/K),
\end{align*}
which completes the proof.

\end{proof}

\begin{proof}[Proof of Theorem~\ref{thm: MoE}]
We will give the prove for $T = T_{2}$, i.e., at the end of the second stage.

\noindent\textbf{Test Error is small.} We first prove the following result for the experts. For all expert $m \in \cM_{k}$, we have that
\begin{align}
\mathbb{P}_{(\xb, y)\sim \cD}\big(yf_{m}(\xb; \Wb^{(T)}\big) \leq 0 \big|(\xb, y) \in \Omega_{k}\big) &= o(1). \label{eq:mainmain11}  
\end{align}
The proof of \label{eq:mainmain11} is similar to the proof of Lemma~\ref{lm:stage1}. We consider the $m$-th expert in the MoE layer, suppose that $m \in \cM_{k}$. Then if we draw a new sample $(\xb, y) \in \Omega_{k}$. Without loss of generality, we assume $\xb= [\alpha y\vb_{k}, \beta \cb_{k}, \gamma\epsilon\vb_{k'}, \bxi]$. By Lemma~\ref{lm:shortphase2}, we have already get the bound for inner product between weights and feature signal, cluster-center signal and feature noise. However, we need to recalculate the bound of the inner product between weights and random noises because we have fresh random noises i.i.d drawn from $\cN(0, (\sigma_{p}^{2}/d) \cdot I_{d})$. Notice that we use normalized gradient descent with step size $\eta$, so we have that 
\begin{align*}
\|\wb_{m,j}^{(T)} - \wb_{m,j}^{(0)}\|_{2} \leq \eta T = \tilde{O}(1).     
\end{align*}
Therefore, by triangle inequality we have that $\|\wb_{m,j}^{(T)}\|_{2} \leq \|\wb_{m,j}^{(0)}\|_{2}  + \tilde{O}(1) \leq \tilde{O}(\sigma_{0}\sqrt{d})$. Because the inner product $\la \wb_{m,j}^{(t)}, \bxi_{p}\ra$ follows the distribution $\cN(0, (\sigma_{p}^{2}/d)\cdot \|\wb_{m,j}^{(T)}\|_{2}^{2})$, with probability at least $1 - 1/(dPMJ)$ we have that , 
\begin{align*}
|\la \wb_{m,j}^{(T)}, \bxi_{p}\ra| = O(\sigma_{p}d^{-1/2}\|\wb_{m,j}^{(t)}\|_{2}\log(dPMJ)) \leq \tilde{O}(\sigma_{0}). 
\end{align*}
Applying Union bound for $m \in [M], j \in [J], p \geq 4$ gives that, with probability at least $1 - 1/d$, 
\begin{align}
|\la \wb_{m,j}^{(T)}, \bxi_{p}\ra| = \tilde{O}(\sigma_{0}), \forall m \in [M], j \in [J], p \geq 4.\label{eq:event2}
\end{align}

Now, under the event that \eqref{eq:event2} holds, we have that 
\begin{align*}
yf_{m}(\xb, \Wb^{(T)}) &= y\sum_{j\in [J]}\sum_{p\in[P]}\sigma(\la\wb_{m,j}^{(T)}, \xb^{(p)}\ra)\\
&= y\sigma(\la\wb^{(T)}_{m,j_{m}^{*}}, \alpha y\vb_{k}\ra) + y\sum_{(j,p)\not = (j_{m}^{*}, 1)}\sigma(\la\wb^{(T)}_{m,j}, \xb^{(p)}\ra)\\
&\geq C_{1}^{3}(1-\sigma_{0}^{0.1})^{3}M^{-4} -  \tilde{O}(\sigma_{0}^{3})\\
&= \tilde{\Omega}(1),
\end{align*}
where the first inequality is by Lemma~\ref{lm:keepstage1}.
Because \eqref{eq:event2} holds with probability at least $1 - 1/d$, so we have prove that 
\begin{align*}
\mathbb{P}_{(\xb, y)\sim \cD}\big(yf_{m}(\xb; \Wb^{(T)}\big) \leq 0 \big|(\xb, y) \in \Omega_{k}\big) &\leq 1/d.  
\end{align*}

We then prove that, with probability at least $1 - o(1)$, an example $\xb \in \Omega_{k}$ will be routed to one of the experts in $\cM_{k}$. For $\xb= [\alpha y\vb_{k}, \beta \cb_{k}, \gamma\epsilon\vb_{k'}, \bxi]$, we need to check that $h_{m}(\xb; \bTheta^{(T)}) < \max_{m'}h_{m'}(\xb; \bTheta^{(T)}), \forall m \not \in \cM_{k}$. By Lemma~\ref{lm:stage2end}, we know that $\la \btheta_{m}^{(T)}, \cb_{k} \ra \leq \max_{m'}\la \btheta_{m'}^{(T)}, \cb_{k} \ra - - \Omega(K^{-1}M^{-9})$. Further by Lemma~\ref{lm:noiseinner}, we have that $\max_{m,k}|\la \btheta_{m}^{(T)}, \vb_{k} \ra| = O(d^{-0.001})$. Again to calculate test error, we need to give an upper bound $\la \btheta_{m}^{(T)}, \bxi_{p} \ra$, where $\bxi_{p}$ is a fresh noise drawn from $\cN(0, (\sigma_{p}^{2}/d) \cdot I_{d})$. We can upper bound the gradient of the gating network by 

\begin{align*}
\|\nabla_{\btheta_{m}} \cL^{(t)}\|_{2}
    & =
    \bigg\|\frac{1}{n} \sum_{i,p}\ind(m_{i,t} =m)\ell'_{i,t}
    \pi_{m_{i,t}}(\xb_i; \bTheta^{(t)})y_{i}f_{m_{i,t}}(\xb_{i};  \Wb^{(t)}) \xb_{i}^{(p)}\\
    &\qquad - \frac{1}{n} \sum_{i,p}\ell'_{i,t}
    \pi_{m_{i,t}}(\xb_i; \bTheta^{(t)})\pi_{m}(\xb_i; \bTheta^{(t)})y_{i}f_{m_{i,t}}(\xb_{i};  \Wb^{(t)}) \xb_{i}^{(p)}\bigg\|_{2}.\\
    &= \tilde{O}(1),
\end{align*}
where the last inequality is due to $|\ell_{i,t}'| \leq 1$, $\pi_{m}, \pi_{m_{i,t}} \in [0,1]$ and $\|\xb_{i}^{(p)}\|_{2} = O(1)$.
This further implies that 
\begin{align*}
\|\btheta_{m}^{(T)}\|_{2} = \|\btheta_{m}^{(T)} - \btheta_{m}^{(0)}\|_{2}
\leq  \tilde{O}(t \eta_{r}) \leq  \tilde{O}(\eta^{-1} \eta_{r}) = \tilde{O}(1),
\end{align*}
where the last inequality is by $\eta_{r} = \Theta(M^{2})\eta$. Because the inner product $\la \btheta_{m}^{(T)}, \bxi_{p}\ra$ follows the distribution $\cN(0, (\sigma_{p}^{2}/d)\cdot \|\btheta_{m}^{(T)}\|_{2}^{2})$, we have that with probability at least $1 - 1/(dPM)$, 
\begin{align*}
|\la\btheta_{m}^{(T)}, \bxi_{p}\ra| = O(\sigma_{p}d^{-1/2}\|\btheta_{m}^{(T)}\|_{2}\log(dPM)) \leq \tilde{O}(d^{-1/2}). 
\end{align*}
Applying Union bound for $m \in [M], p \geq 4$ gives that, with probability at least $1 - 1/d$, 
\begin{align}
|\la \btheta_{m}^{(T)}, \bxi_{p}\ra| = \tilde{O}(d^{-1/2}), \forall m \in [M], p \geq 4.\label{eq:event3}
\end{align}

Now, under the event that \eqref{eq:event3} holds, we have that 
\begin{align*}
&h_{m}(\xb; \bTheta^{(T)}) - \max_{m'}h_{m'}(\xb; \bTheta^{(T)})\\
&\leq \la \btheta_{m}^{(T)}, \cb_{k}\ra -  \max_{m'}\la \btheta_{m'}^{(T)}, \cb_{k} \ra + 4\max_{m,k}|\la \btheta_{m}^{(T)}, \vb_{k} \ra| + 4P\max_{m,p}|\la \btheta_{m}^{(T)}, \bxi_{p}\ra|\\
&\leq - \Omega(K^{-1}M^{-9}) + \tilde{O}(d^{-0.001})\\
&< 0.
\end{align*}
Because \eqref{eq:event3} holds holds with probability at least $1 - 1/d$, so we have prove that with probability at least $1 - 1/d$, an example $\xb \in \Omega_{k}$ will be routed to one of the experts in $\cM_{k}$.

\noindent\textbf{Training Error is zero.}
The prove for training error is much easier, because we no longer need to deal with the fresh noises and we no longer need to use high probability bound for those inner products with fresh noises. That's the reason we can get exactly zero training error. We first prove the following result for the experts. For all expert $m \in \cM_{k}$, we have that
\begin{align*}
y_{i}f_{m}(\xb_{i}; \Wb^{(T)}\big) \leq 0, \forall i \in \Omega_{k}.
\end{align*}

Without loss of generality, we assume that the feature patch appears in $\xb_{i}^{(1)}$. By Lemma~\ref{lm:keepstage1}, we have that for all $i \in \Omega_{k}$ 
\begin{align*}
y_{i}f_{m}(\xb_{i}, \Wb^{(T)}) &= y_{i}\sum_{j\in [J]}\sum_{p\in[P]}\sigma(\la\wb_{m,j}^{(T)}, \xb_{i}^{(p)}\ra)\\
&= y_{i}\sigma(\la\wb^{(T)}_{m,j_{m}^{*}}, \alpha y_{i}\vb_{k}\ra) + y_{i}\sum_{(j,p)\not = (j_{m}^{*}, 1)}\sigma(\la\wb^{(T)}_{m,j}, \xb^{(p)}\ra)\\
&\geq C_{1}^{3}(1-\sigma_{0}^{0.1})^{3}M^{-4} -  \tilde{O}(\sigma_{0}^{3})\\
&>0,
\end{align*}
where the first inequality is Lemma~\ref{lm:keepstage1}.
We then prove that, and example $(\xb_{i}, y_{i}) \in \Omega$ will be routed to one of the experts in $\cM_{k}$. Suppose the $m$-th expert is not in $\cM_{k}$.  We only need to check the value of $h_{m}(\xb_{i}; \bTheta^{(T)}) < \max_{m'}h_{m'}(\xb_{i}; \bTheta^{(T)}) $, which is straight forward by Lemma~\ref{lm:stage2end} and Lemma~\ref{lm:noiseinner}.

\end{proof}

\section{Auxiliary Lemmas}
\begin{lemma}\label{lm:GaussianTop1}
Let $\{a_{m}\}_{m=1}^{M}$ are the random variable i.i.d. drawn from $\cN(0,1)$. Define the non-increasing sequence of $\{a_{m}\}_{m=1}^{M}$ as $a^{(1)}\geq \ldots \geq a^{(M)}$. Then we have that 
\begin{align*}
\mathbb{P}(a^{(2)} \geq (1 - G)a^{(1)}) \leq GM^{2}    
\end{align*}

\end{lemma}
\begin{proof}

Let $\Psi$ be the CDF of $\cN(0, 1)$ and let $\rho$ be the PDF of $\cN(0, \sigma_{0}^{2})$. Then we have that,

\begin{align*}
&\mathbb{P}(a^{(2)} \geq (1 - G)a^{(1)})\\
&= \int_{a^{(1)} \geq \ldots \geq a^{(M)} } \ind(a^{(2)} \geq (1 - G)a^{(1)})M! \Pi_{m}\rho(a^{(m)})d\ab\\
&=\int_{a^{(1)} \geq a^{(2)} } \ind(a^{(2)} \geq (1 - G)a^{(1)})M(M-1) \rho(a^{(1)})\rho(a^{(2)})\Psi(a^{(2)})^{M-2}da^{(1)}da^{(2)}\\
&\leq \int_{a^{(1)} \geq a^{(2)} } \ind(a^{(2)} \geq (1 - G)a^{(1)})M(M-1) \rho(a^{(1)})\frac{1}{\sqrt{2\pi}}da^{(1)}da^{(2)}\\
&= \int_{a^{(1)}\geq 0} \frac{GM(M-1)}{\sqrt{2\pi}} a^{(1)}\rho(a^{(1)}) da^{(1)}\\
&\leq  GM^{2}.
\end{align*}
\end{proof}

For normalized gradient descent we have following lemma,
\begin{lemma}[Lemma C.19 \citealt{allen2020towards}]\label{lm: Tensor power update}
Let $\{x_{t},y_{t}\}_{t=1,..}$ be two positive sequences that satisfy
\begin{align*}
    x_{t+1}&\geq x_{t}+ \eta \cdot C_{t} x_{t}^{2}\\
    y_{t+1}&\leq y_{t}+ S\eta\cdot C_{t} y_{t}^{2},
\end{align*}
and $|x_{t+1} - x_{t}|^{2} + |y_{t+1} - y_{t}|^{2} \leq \eta^{2}$. Suppose $x_{0}, y_{0} = o(1), x_{0} \geq y_{0}S(1+G)$, 
$$\eta \leq \min\{\frac{G^{2}x_{0}}{\log(A/x_{0})}, \frac{G^{2}y_{0}}{\log(1/G)}\}.$$
 Then we have for all $A > x_{0}$, let $T_{x}$ be the first iteration such that $x_{t}\geq A$, then we have $y_{T_{x}}\leq O(y_{0}G^{-1})$.

\end{lemma}
\begin{proof}
We only need to replace $O(\eta A^{q-1})$ in the proof of Lemma C.19 by $O(\eta)$, because we use normalized gradient descent, i.e, $C_{t}\xb_{t}^{2} \leq 1$. For completeness, we present the whole poof here.

for all $g = 0,1,2, \ldots, $, let $\cT_{g}$ be the first iteration such that $x_{t} \geq (1+\delta)^{g}x_{0}$, let $b$ be the smallest integer such that $(1+\delta)^{b}x_{0}\geq A$. For simplicity of notation, we replace $x_{t}$ with $A$ whenever $x_{t} \geq A$. Then by the definition of $\cT_{g}$, we have that
\begin{align*}
\sum_{t \in [\cT_{g}, \cT_{g+1})} \eta C_{t} [(1+\delta)^{g}x_{0}]^{2} \leq x_{\cT_{g+1}} - x_{\cT_{g}} \leq \delta(1+\delta)^{g}x_{0} + O(\eta),    
\end{align*}
where the last inequality holds because we are using normalized gradient descent, i.e., $\max_{t}|x_{t+1} - x_{t}| \leq \eta$.
This implies that 
\begin{align*}
\sum_{t \in [\cT_{g}, \cT_{g+1})} \eta C_{t} \leq \frac{\delta}{(1+\delta)^{g}}\frac{1}{x_{0}} + \frac{O(\eta)}{x_{0}^{2}}.
\end{align*}
Recall that $b$ is the smallest integer such that $(1+\delta)^{b}x_{0}\geq A$, so we can calculate
\begin{align*}
\sum_{t \geq 0, x_{t} \leq A}\eta C_{t} \leq \bigg[\sum_{g=0}^{b-1}\frac{\delta}{(1+\delta)^{g}}\frac{1}{x_{0}}\bigg] + \frac{O(\eta)}{x_{0}^{2}}b
= \frac{1+\delta}{x_{0}} + \frac{O(\eta)b}{x_{0}^{2}} \leq \frac{1+\delta}{x_{0}} + \frac{O(\eta)\log(A/x_{0})}{x_{0}^{2}\log(1+\delta)}
\end{align*}
Let $T_{x}$ be the first iteration $t$ in which $x_{t} \geq A$.  Then we have that 
\begin{align}
\sum_{t=0}^{T_{x}}\eta C_{t} \leq \frac{1+\delta}{x_{0}} + \frac{O(\eta)\log(A/x_{0})}{\delta x_{0}^{2}}. \label{eq:xseq}  
\end{align}
On the other hand, let $A' = G^{-1}y_{0}$ and b' be the smallest integer such that $(1+\delta)b' x_{0} \geq A'$. For simplicity of notation, we replace $y_{t}$ with $A'$ when $y_{t}\geq A'$. Then let $\cT_{g}'$ be the first iteration such that $y_{t} \geq (1+\delta)^{g}y_{0}$, then we have that 
\begin{align*}
\sum_{t \in [\cT'_{g}, \cT'_{g+1})} \eta S C_{t}[(1+\delta)^{g+1}x_{0}]^{(q-1)} \geq y_{\cT'_{g+1}} - y_{\cT'_{g}} \geq \delta(1+\delta)^{g}y_{0} - O(\eta).  
\end{align*}
Therefore, we have that
\begin{align*}
\sum_{t\in[\cT_{g}',\cT_{g+1}')} S\eta C_{t} \geq \frac{\delta}{(1+\delta)^{g}(1+\delta)^{2}}\frac{1}{y_{0}} - \frac{O(\eta)}{y_{0}^{2}}.  
\end{align*}
Recall that $b'$ is the smallest integer such that $(1+\delta)^{b'}y_{0} \geq A'$. wo we have that 
\begin{align*}
\sum_{t\geq 0, x_{t}\leq A}\eta S C_{t} \geq \sum_{g=0}^{b'-2} \frac{\delta}{(1+\delta)^{g}(1+\delta)^{2}}\frac{1}{y_{0}} - \frac{O(\eta)b'}{y_{0}^{2}}
\end{align*}
Let $T_{y}$ be the first iteration $t$ in which $y_{t} \geq A'$, so we can calculate
\begin{align}
\sum_{t=0}^{T_{y}}\eta S C_{t} \geq \frac{1 - O(\delta + G)}{y_{0}} - \frac{O(\eta)\log(A'/y_{0})}{y_{0}^{2}\delta}. \label{eq:yseq}
\end{align}
Compare \eqref{eq:xseq} and \eqref{eq:yseq}. Choosing $\delta = G$ and $\eta \leq \min\{\frac{G^{2}x_{0}}{\log(A/x_{0})}, \frac{G^{2}y_{0}}{\log(1/G)}\}$, together with $x_{0}\geq y_{0}S(1+G)$

\end{proof}

\bibliography{deeplearningreference}
\bibliographystyle{ims}

\end{document}